\documentclass[review]{elsarticle}

\usepackage{amsfonts} 
\usepackage{fullpage} 
\usepackage{algorithm}
\usepackage{algorithmic}
\usepackage{amsmath}
\usepackage{graphicx}
\usepackage{mathtools}

\usepackage{mathrsfs}
\usepackage{multirow}
\usepackage{enumitem}

\usepackage{bm}
\usepackage{bbm}
\usepackage{amssymb}
\usepackage{caption}
\usepackage{subcaption}
\usepackage{amsthm}
\usepackage{lineno,hyperref}
\usepackage{color}

\newcommand{\ip}[1]{\left\langle #1 \right\rangle}
\newcommand{\transp}{\mathsf{T}}

\newcommand{\unif}{\operatorname{unif}}

\newcommand{\classlabel}{\operatorname{class\_label}}

\newcommand{\spn}{\operatorname{span}}

\newcommand{\supp}{\operatorname{supp}}
\newcommand{\err}{\operatorname{err}}
\newcommand{\mean}{\operatorname{mean}}

\newcommand{\corr}{\operatorname{corr}}
\newcommand{\dd}[1]{\, {\mathrm d}{#1}}

\newcommand{\widehatbv}[1]{\,\widehat{\!{#1}}}

\DeclarePairedDelimiter\floor{\lfloor}{\rfloor}

\makeatletter
\def\@author#1{\g@addto@macro\elsauthors{\normalsize%
    \def\baselinestretch{1}%
    \upshape\authorsep#1\unskip\textsuperscript{%
      \ifx\@fnmark\@empty\else\unskip\sep\@fnmark\let\sep=,\fi
      \ifx\@corref\@empty\else\unskip\sep\@corref\let\sep=,\fi
      }%
    \def\authorsep{\unskip,\space}%
    \global\let\@fnmark\@empty
    \global\let\@corref\@empty  
    \global\let\sep\@empty}%
    \@eadauthor={#1}
}
\makeatother

\journal{Applied and Computational Harmonic Analysis}

\begin{document}

\newtheorem{mythm}{Theorem}[section]
\newtheorem{myprop}{Proposition}[section]
\newtheorem{mycor}{Corollary}[section]
\newtheorem{remark}{Remark}[section]
\newtheorem{mydef}{Definition}[section]
\newtheorem{mylem}{Lemma}[section]
\newtheorem{myex}{Example}[section]
\newtheorem{myassump}{Assumption}

\begin{frontmatter}

\title{The Use of Mutual Coherence to Prove $\ell^1/\ell^0$-Equivalence\\ in Classification Problems}

\author{Chelsea Weaver\corref{cor1}\fnref{fn1}}
\ead{chelseaannweaver@gmail.com}
\cortext[cor1]{Corresponding author}
\fntext[fn1]{Current address: Amazon Web Services, Seattle, WA}

\author{Naoki Saito}
\ead{saito@math.ucdavis.edu}

\address{Department of Mathematics\\
University of California, Davis \\
One Shields Avenue\\
Davis, California, 95616, United States}

\begin{abstract}

We consider the decomposition of a signal over an overcomplete set of vectors. Minimization of
the $\ell^1$-norm of the coefficient vector can often retrieve the sparsest solution (so-called ``$\ell^1/\ell^0$-equivalence''), a generally NP-hard task, and this fact has powered the field of compressed sensing. Wright et al.'s sparse representation-based classification (SRC) applies this relationship to machine learning, wherein the signal to be decomposed represents the test sample and columns of the dictionary are training samples. We investigate the relationships between $\ell^1$-minimization, sparsity, and classification accuracy in SRC. After proving that the tractable, deterministic approach to verifying $\ell^1/\ell^0$-equivalence fundamentally conflicts with the high coherence between same-class training samples, we demonstrate that $\ell^1$-minimization can still recover the sparsest solution when the classes are well-separated. Further, using a nonlinear transform so that sparse recovery conditions may be satisfied, we demonstrate that approximate (not strict) equivalence is key to the success of SRC.

\end{abstract}


\end{frontmatter}


\section{Introduction}

The decomposition of a given signal or sample over a pre-determined set of vectors is a technique often used in signal processing and pattern recognition. We can store a signal by decomposing it over a fixed basis and keeping only the largest coefficients; in linear regression, predictions are made by estimating parameters via least-squared error using the training data. In the case that the system is underdetermined, so that an infinite number of representations of the signal or sample exist, regularization is often used to make the problem well-posed. The question, naturally, is how to choose the type of regularization used, so that the representation is well-suited to the task at hand and can be found efficiently. 

In compressed sensing, a fairly recent advancement in signal processing, it is assumed that a vector of signal measurements is represented using an overcomplete set of vectors (often called a \emph{dictionary}) and that the (unknown) coefficient vector is sparse. Obtaining this sparse solution vector is the key to recovering the complete signal in a way that requires fewer measurements than traditional methods \cite{can:cs}. Thus, to determine the unknown coefficients, an appropriate regularization term should enforce sparsity, i.e., seek the solution requiring the fewest nonzero coefficients. Determining tractable methods for solving such optimization problems are the core of compressed sensing techniques, as minimizing the $\ell^0$-``norm'' (which counts the number of nonzero coefficients) is NP-hard in general. However, in addition to successful greedy methods such as \emph{orthogonal matching pursuit} \cite{tro:omp}, it was found that sparse regularization can, in many circumstances, be replaced with minimization of the $\ell^1$-norm (which sums the coefficient magnitudes) to the same effect. That is, under certain conditions, minimization of the $\ell^1$-norm is \emph{equivalent} to sparse regularization, hence the term ``\emph{$\ell^1/\ell^0$-equivalence}''. Though requiring an iterative algorithm to solve, this relaxation to $\ell^1$-minimization reduces the optimization problem to a linear program and can be solved efficiently. There has been a lot of work done (see, for example, the seminal papers by Candes and Tao \cite{can:decode} and Donoho \cite{don:cs}) showing that, under certain conditions, $\ell^1$-minimization \emph{exactly} recovers the sparsest solution, and analogous results hold in the case of noisy data. We review some of these results in Section \ref{sec:equiv_guar_2}.

A similar technique used in compressed sensing has been successfully applied to tasks in pattern recognition. The popular classification method \emph{sparse representation-based classification} (SRC) \cite{wri:src}, proposed by Wright et al.\ in 2009, classifies a given test sample by decomposing it over an overcomplete set of training samples so that the $\ell^1$-norm of the coefficient vector is minimized. The test sample is assigned to the class with the most contributing coefficients (in terms of reconstruction). By minimizing the $\ell^1$-norm, the goal is that the sparsest such representation will be found (as in compressed sensing), and that this will automatically produce nontrivial nonzero coefficients at training samples in the same class as the test sample, rendering correct classification. Similar approaches have been used in dimensionality reduction \cite{qiao:spp}, semi-supervised learning \cite{chen:l1graph}, and clustering \cite{chen:l1graph}. 

In this paper, we investigate the role of sparsity in SRC, specifically, the two-fold question of: (i) whether or not $\ell^1/\ell^0$-equivalence can be achieved in practice, i.e., whether $\ell^1$-minimization reliably produces the sparsest solution in the classification context; and (ii) whether this equivalence is necessary for good classification performance. The inherent problem with (i) is that practically-implementable recovery conditions under which $\ell^1$-minimization is guaranteed to find the sparsest solution require that the vectors in the dictionary be \emph{incoherent}, or in some way ``spread out'' in space. These guarantees hold with high probability, for example, on dictionaries of vectors that are randomly-generated from certain probability distributions and dictionaries consisting of randomly-selected rows of the discrete Fourier transform matrix \cite{can:rob,don:und,can:decode}. Obviously, unlike these examples, data samples in the same class are often \emph{highly}-correlated. In fact, strong inner-class similarity generally makes the data \emph{easier} to classify. 

Our contributions in this paper are the following:
\begin{enumerate}[noitemsep,nolistsep]
\item We show that the fundamental assumptions of SRC are in direct contradiction with applicable and tractable sparse recovery guarantees. It follows that the experimental success of SRC should not automatically imply the usefulness of sparsity in this framework. 

\item Using a randomly-generated database designed to model facial images, we show that $\ell^1$-minimization can still recover the sparsest solution on highly-correlated data, provided that the classes are sufficiently well-separated. Thus the lack of implementable equivalence \emph{guarantee} does not automatically imply lack of \emph{equivalence} in SRC, at least on certain databases.

\item We investigate the feasibility and implementation of a nonlinear transform that maximally spreads out the training samples in each class while maintaining the dataset's class structure. Though there are strict limitations on the design of such a transform, which we describe in detail in Section \ref{MCD_Project_3}, we demonstrate that the higher-dimensional space can allow for the application of equivalence guarantees while still allowing us to classify the dataset. This renders a method for examining the relationship between classification accuracy and the sparsity of the coefficient vector in SRC, and how close this is to the (provably) sparsest solution. We demonstrate that \emph{approximate} (and not strict) equivalence between the $\ell^1$-minimized solution and the sparsest solution is the key to the success of SRC. 

\end{enumerate}

The paper is organized as follows: We begin by motivating and reviewing the basics of compressed sensing and sparsity recovery guarantees in Section \ref{Equivalence_Guarantees}, and we give an overview of SRC in Section \ref{sec:src}. In Section \ref{sec:conflict}, we formerly describe the conflict between $\ell^1/\ell^0$-recovery guarantees and classification data, and in Section \ref{MCD_Project_1}, we rigorously assess the applicability of these recovery guarantees in the classification context. Section \ref{MCD_Project_2} presents empirical findings relating sparse recovery and highly-correlated data. In Section \ref{MCD_Project_3}, we investigate the feasibility of a nonlinear data transform to force the aforementioned recovery guarantees to hold and insights that can be gained from this procedure. We conclude this paper in Section \ref{sec:conclusion}.

\section{Compressed Sensing and Recovery Guarantees} \label{Equivalence_Guarantees}

In this section, we detail the motivation behind $\ell^1/\ell^0$-equivalence and state practically-implementable equivalence theorems.

\subsection{Motivation from Compressed Sensing}\label{sec:cs}

Suppose that we wish to collect information about (i.e., sample or take measurements of) a continuous signal $f(t)$ and then send or store this information in an efficient manner. For example, $f(t)$ could be a sound wave or an image. Also suppose that a good approximation of the original signal must later be recovered. According to the Nyquist/Shannon sampling theorem, we must sample $f(t)$ at a rate of at least twice its maximum frequency in order to be able to reconstruct $f(t)$ exactly \cite{shan:samp_thm}. But in some applications, doing so may be expensive or even impossible. 

In the circumstances that we are able to take many measurements of $f(t)$ to obtain its discrete analog $\bm{f} \in \mathbb{R}^N$, one efficient method of compressing it is the following procedure: Let the columns of $\Psi := [\bm{\psi}_1,\ldots,\bm{\psi}_N]$ form an orthonormal basis for $\mathbb{R}^N$, and suppose that $\bm{f}$ has a sparse representation in this basis, i.e., that we can write $\bm{f} = \sum_{j=1}^N \alpha_j \bm{\psi}_j$, where $\alpha_j := \ip{\bm{f},\bm{\psi}_j}$, $1\leq j \leq N$, and $\bm{\alpha}:=[\alpha_1,\ldots,\alpha_N]^\transp$ is sparse. Setting all but the $k$ largest (in absolute value) entries of $\bm{\alpha}$ to 0 in order to obtain $\bm{\alpha}_k$, it can be shown that $\Psi \bm{\alpha}_k$ gives the best $k$-term least squares approximation of $\bm{f}$ in this basis. Clearly, the sparser $\bm{\alpha}$ is, the better approximation we will obtain of $\bm{f}$, and in the case that $\bm{\alpha}$ has no more than $k$ nonzero coefficients, we recover the exact solution. This is the basic idea behind the so-called transform coding, and the most popular one is the JPEG image compression
standard \cite{pen:jpeg}, which uses the discrete cosine transform as the sparsifying basis $\Psi$. 

The problem with this procedure is that it is inefficient to collect all $N$ samples if we are only going to throw most (all but $k$) of them away when the signal is compressed. This is the motivation behind \emph{compressed sensing}, originally proposed by Cand{\`e}s and Tao \cite{can:decode} and Donoho \cite{don:cs} (see also Cand{\`e}s and Tao's work \cite{can:near_opt} and the paper by Cand{\`e}s et al.\ \cite{can:sta}). Let $\Phi \in \mathbb{R}^{m\times N}$ be a \emph{sensing} or \emph{measurement} matrix with $m< N$ and consider the underdetermined system
\begin{align*}
\bm{y}_0 := \Phi \bm{f} = \Phi \Psi \bm{\alpha} = X \bm{\alpha}
\end{align*}
for sparse $\bm{\alpha}$, where we have set $X := \Phi \Psi$. Using $\|\bm{\alpha}\|_0$ to denote the number of nonzero coordinates of $\bm{\alpha}$ (hence the terminology ``$\ell^0$-`norm'\thinspace''---observe that $\|\cdot\|_0$ is only a pseudonorm because it does not satisfy homogeneity), we would ideally recover $\bm{f}$ by solving the optimization problem
\begin{equation}\label{eq:cs_l0}
\bm{\alpha}_0 := \arg \min_{\bm{\alpha}\in \mathbb{R}^N} \|\bm{\alpha}\|_0 \text{ subject to } X\bm{\alpha} = \bm{y}_0
\end{equation}
and setting $\widehatbv{\bm{f}} := \Psi \bm{\alpha}_0$ with $\widehatbv{\bm{f}}\approx \bm{f}$. Unfortunately, solving Eq.~\eqref{eq:cs_l0} is NP-hard. When $X$ satisfies certain conditions and when $\bm{\alpha}_0$ is sufficiently sparse, however, the solution to Eq.~\eqref{eq:cs_l0} can be found by solving the $\ell^1$-minimization problem
\begin{equation}\label{eq:cs_l1}
\bm{\alpha}_1 := \arg \min_{\bm{\alpha}\in \mathbb{R}^N} \|\bm{\alpha}\|_1 \text{ subject to } X\bm{\alpha} = \bm{y}_0.
\end{equation}
This was a riveting finding, as the optimization problem in Eq.~\eqref{eq:cs_l1} is convex and can be solved efficiently. It has been shown that, under certain conditions (e.g., when the columns of $\Phi$ are uniformly random on the sphere $S^{m-1}$), this procedure produces an approximation of $\bm{f}$ that is as good as that of its best $k$-term approximation \cite{don:cs}. Further, theoretical and experimental results demonstrate that in many situations, the number of measurements $m$ needed to recover $\bm{f}$ is significantly less than $N$ and can be much lower than the number required by the Nyquist/Shannon theorem. For example, when the measurement matrix $\Phi\in \mathbb{R}^{m\times N}$ contains i.i.d.\ Gaussian entries, then exact recovery of $\bm{\alpha}_0$ via $\ell^1$-minimization can be achieved (with high probability) in only $m = O(k\log(N/k))$ measurements, where $\|\bm{\alpha}_0\|_0 = k$ \cite{can:decode}. 

Even more astoundingly, similar results hold in the presence of noise. Suppose that the noiseless vector $\bm{y}_0$ is replaced with $\bm{y} = \bm{y}_0 + \bm{z}$, for $\bm{z}\in\mathbb{R}^m$ a vector of errors satisfying $\|\bm{z}\|_2\leq \zeta$. It follows that under certain conditions (see Section \ref{sec:mc}), the $\ell^1$-minimization problem
\begin{equation}\label{eq:cs_l1_error}
\bm{\alpha}_{1,\epsilon} := \arg \min_{\bm{\alpha}\in \mathbb{R}^N} \|\bm{\alpha}\|_1 \text{ subject to } \|X\bm{\alpha} - \bm{y}\|_2 \leq \epsilon
\end{equation}
is guaranteed to recover a coefficient vector approximating the ground truth sparse vector $\bm{\alpha}_0$ (the solution to Eq.~\eqref{eq:cs_l0}) with $\|\bm{\alpha}_{1,\epsilon} - \bm{\alpha}_0\|_2 \leq C_k (\epsilon+\zeta)$ \cite{don:sta}. The constant $C_k$ depends on properties of the matrix $X$ and the sparsity level $\|\bm{\alpha}_0\|_0 = k$. 

A popular application of compressed sensing is \emph{magnetic resonance imaging} (MRI), in which the measurement matrix $\Phi$ consists of $m$ randomly-selected rows of the discrete Fourier transform in $\mathbb{R}^{N\times N}$ \cite{don:mri}. Other applications abound in the areas of data acquisition and compression, including sensor networks \cite{xia:sens}, seismology \cite{her:seis}, and single pixel cameras \cite{dua:sin_pix}.

\subsection{Recovery Guarantees} \label{sec:equiv_guar_2}

The conditions under which $\ell^1$-minimization can guarantee exact or approximate recovery of the sparsest solution (e.g., conditions under which the solutions to Eq.~\eqref{eq:cs_l0} and Eq.~\eqref{eq:cs_l1} are equal, i.e, $\ell^1/\ell^0$-equivalence holds) are called \emph{recovery guarantees}.
These conditions concern the \emph{incoherence} (or \emph{spread}) of the vectors in the dictionary. Essentially, recovery guarantees cannot be applied when the vectors are too correlated. A prototypical example is that if the dataset contains two copies of the same vector (i.e., a pair of maximally-correlated vectors), then the minimum $\ell^1$-norm solution may contain a nonzero coefficient at either one of the copies or at a combination of the two. Contrast this with the sparsest solution, which would never contain nonzero coefficients at both copies.

There are various ways of measuring the incoherence in a dictionary, each leading its own theory relating the solutions of Eq.~\eqref{eq:cs_l0} and Eq.~\eqref{eq:cs_l1} (or its noise version Eq.~\eqref{eq:cs_l1_error}). In this paper, we focus primarily on recovery guarantees stated in terms of \emph{mutual coherence}, and we review mutual coherence-based recovery guarantees below. Unlike other approaches, the mutual coherence method is both tractable and deterministic, as we subsequently discuss.

To make the problem more general, we no longer explicitly assume the use of a sparsifying transform matrix $\Psi$ and consider the general system $X\bm{\alpha} = \bm{y}_0$, for $X \in \mathbb{R}^{m\times N}$ with $m<N$.

\subsubsection{Recovery Guarantees in Terms of Mutual Coherence} \label{sec:mc}

\begin{mydef} \label{def:mc}
Given a matrix $X = [\bm{x}_1,\ldots,\bm{x}_N] \in \mathbb{R}^{m\times N}$ with normalized columns (so that $\|\bm{x}_i\|_2 = 1$ for $1\leq i \leq N$), the \emph{mutual coherence} of $X$, denoted $\mu(X)$, is given by
\begin{align}
\label{eqn:mu}
\mu(X) := \max_{1\leq i\neq j\leq N} |\ip{\bm{x}_i,\bm{x}_j}|.
\end{align}
\end{mydef}
Note that mutual coherence costs $O(N^2m)$ to compute. 

\begin{mythm}[Donoho and Elad \cite{don:osr} ; Gribonval and Nielsen \cite{grib:union}] \label{thm:mc}

Let $X \in \mathbb{R}^{m\times N}$, $m<N$, have normalized columns and mutual coherence $\mu(X)$. If $\boldsymbol{\alpha}$ satisfies $X\boldsymbol{\alpha} = \bm{y}_0$ with
\begin{equation} \label{eq:k_max}
\|\boldsymbol{\alpha}\|_0 < \frac{1}{2}\Big( 1+ \frac{1}{\mu(X)}\Big),
\end{equation}
then $\boldsymbol{\alpha}$ is the unique solution to the $\ell^1$-minimization problem in Eq.~\eqref{eq:cs_l1}.
\end{mythm}
This means that if $\ell^1$-minimization finds a solution with less than $(1/2)(1+\mu(X)^{-1})$ nonzeros, then it is necessarily the sparsest solution and so $\ell^1/\ell^0$-equivalence holds. 

Given noise tolerance $\zeta$ and approximation error bound $\epsilon$, the following theorem by Donoho et al.\ gives conditions for $\ell^1/\ell^0$-equivalence in the noisy setting:

\begin{mythm} [Donoho, Elad, and Temlyakov \cite{don:sta}] \label{thm:mc_noise}

Let $X \in \mathbb{R}^{m\times N}$, $m<N$, have normalized columns and mutual coherence $\mu(X)$. Suppose there exists an ideal noiseless signal $\bm{y}_0$ such that $\bm{y}_0 = X\boldsymbol{\alpha}$ and 
\begin{equation} \label{eq:k_max_noise}
\|\boldsymbol{\alpha}\|_0 = k \leq \frac{1}{4} \Big(1+\frac{1}{\mu(X)}\Big).
\end{equation}
Then $\boldsymbol{\alpha} = \boldsymbol{\alpha}_0$ is the unique sparsest representation of $\bm{y}_0$ over $X$.
Further, suppose that we only observe $\bm{y} = \bm{y}_0 + \bm{z}$ with $\|\bm{z}\|_2 \leq \zeta$.
Then we have 
\begin{equation} \label{eq:mc_noise}
\| \boldsymbol{\alpha}_{1,\epsilon} - \boldsymbol{\alpha}_0\|_2^2 \leq \frac{(\epsilon+\zeta)^2}{1-\mu(X)(4k-1)},
\end{equation}
where $\boldsymbol{\alpha}_{1,\epsilon}$ is the solution to Eq.~\eqref{eq:cs_l1_error}.
\end{mythm}

That is, if the ideal sparse vector $\boldsymbol{\alpha}_0$ is sparse enough and the mutual coherence of $X$ is small enough, $\ell^1$-minimization will give us a solution close to $\boldsymbol{\alpha}_0$, with ``how close'' depending on the sparsity level $k$, mutual coherence $\mu(X)$, noise tolerance $\zeta$, and approximation error bound $\epsilon$. 

Something can also be said regarding the support of $\bm{\alpha}_{1,\epsilon}$ in the noisy setting:

\begin{mythm} [Donoho, Elad, Temlyakov \cite{don:sta}] \label{thm:mc_stab}

Suppose that $\bm{y} = \bm{y}_0 + \bm{z}$, where $\bm{y}_0 = X\bm{\alpha}_0$, $\|\bm{\alpha}_0\|_0 \leq k$ and $\|\bm{z}\|_2 \leq \zeta$. Suppose that $\beta := \mu(X) k < \frac{1}{2}$ (so $k < \frac{1}{2\mu(X)}$). Set
\begin{equation} \label{eq:mc_stab}
\gamma := \frac{\sqrt{1-\beta}}{1-2\beta}.
\end{equation}
Then given $\bm{\alpha}_{1,\epsilon}$ the solution to Eq.~\eqref{eq:cs_l1_error}
with exaggerated error tolerance $\epsilon := C \zeta$ where $C = C(\mu(X),k) := \gamma\sqrt{k}$, we have that $\supp(\bm{\alpha}_{1,\epsilon}) \subset \supp(\bm{\alpha}_0)$. 
\end{mythm}

This says that when the mutual coherence is very small relative to the sparsity level, the solution $\bm{\alpha}_{1,\epsilon}$ to Eq.~\eqref{eq:cs_l1_error} has the same support as the sparsest solution $\bm{\alpha}_0$. (Observe that $\bm{\alpha}_0$ is indeed the sparsest solution by Theorem \ref{thm:mc}, since $\|\bm{\alpha}_0\|_0 < (1/2)\mu(X)^{-1} < (1/2)(1+\mu(X)^{-1})$.) Since $\epsilon = \gamma\sqrt{k} \,\zeta$ and $\gamma \geq 1$, $\epsilon \geq \zeta$ is required in Theorem \ref{thm:mc_stab}.

\subsubsection{Other Recovery Guarantees}

There are methods of proving $\ell^1/\ell^0$-equivalence that do not involve mutual coherence. For example, those using the \emph{restricted isometry constant} involve a quantification of how close any set of $k$ columns of $X$ is to being an orthonormal basis \cite{can:rip,cai:rip}, and other guarantees use the smallest number of linearly dependent columns of $X$, defined as the \emph{spark} of $X$ \cite{don:osr}. However, these approaches are generally not tractable in deterministic settings; their usefulness is largely limited to applications in which $X$ is a random matrix with known (with high probability) restricted isometry constant or spark.

Alternatively, if we desire stochastic results, there are other recovery guarantees involving versions of mutual incoherence. When applied to random matrices, these guarantees are generally stronger than those in Theorem \ref{thm:mc} and \ref{thm:mc_noise} (in terms of requiring less measurements and/or less sparsity of the solution vector). For example, Cand{\`e}s and Plan \cite{can:ripless} provide conditions that guarantee recovery (with high probability) of sparse and approximately sparse solutions in the case that the rows of the dictionary are sampled independently from certain probability distributions. These conditions are in terms of incoherence defined as an upper bound on the squared norms of the rows of $X$ (either deterministically or stochastically), and require an \emph{isotropy} property \cite{can:ripless}. In the case that the probability distribution has mean $0$, this property states that the covariance matrix of the probability distribution is equal to the identity matrix. In another paper \cite{can:ripless2}, Cand{\`e}s and Plan guarantee probabilistic recovery in terms of a condition on mutual coherence (as defined in Definition \ref{def:mc}) that is satisfied with high probability on certain random matrices. These recovery guarantees allow for the sparsity level $k$ in the case of these random matrices to be notably larger than in Eq.~\eqref{eq:k_max} in Theorem \ref{thm:mc}. We also mention the results by Tropp \cite{tro:ran_sub} concerning recovery in terms of mutual coherence and the extreme singular values of randomly-chosen subsets of dictionary columns.

If we do not assume that classification data are drawn from a particular probability distribution, then these stochastic results either do not apply or are intractable to compute. Thus Donoho et al.'s theorems discussed in Section \ref{sec:mc} are the best tool we have to prove $\ell^1/\ell^0$-equivalence given an arbitrary (possibly large) matrix of training data. That said, it is important to note that these mutual coherence theorems produce what are generally considered to be fairly loose bounds on the sparsity level $\|\bm{\alpha}_0\|_0$, given experimental results and cases for which restricted isometry constants are known \cite[Chap.~10]{has:sta}. 

\section{Sparse Representation-Based Classification} \label{sec:src}


We next review Wright et al.'s application of the $\ell^1$-norm/sparsity relationship to classification. In reviewing the compressed sensing framework, we referred to our underdetermined system using the notation $X\bm{\alpha} = \bm{y}_0$ (or $X\bm{\alpha} =\bm{y}$, if the represented signal was expected to be noisy), for $X \in \mathbb{R}^{m \times N}$. To differentiate the classification context, let $X_\mathrm{tr} \in \mathbb{R}^{m\times N_\mathrm{tr}}$ be the matrix of training samples, and let $\bm{y} \in \mathbb{R}^m$ be an arbitrary test sample.


SRC solves
\begin{equation}\label{eq:src_opt}
\bm{\alpha}^* := \arg \min_{\bm{\alpha}\in \mathbb{R}^{N_\mathrm{tr}}} \|\bm{\alpha}\|_1, \text{ subject to } \bm{y} = X_\mathrm{tr} \bm{\alpha}.
\end{equation}
Alternatively, in the case of noise in which an exact representation may not be desirable (see the discussion at the beginning of Section \ref{MCD_Project_1}), one can solve the regularized optimization problem

\begin{equation} \label{eq:src_opt_noise}
\bm{\alpha}^* := \arg \min_{\bm{\alpha}\in \mathbb{R}^{N_\mathrm{tr}}} \Big \{\frac{1}{2} \|\bm{y}-X_\mathrm{tr}\bm{\alpha}\|_2^2 + \lambda\|\bm{\alpha}\|_1\Big \}.
\end{equation}
Here, $\lambda$ is the trade-off between error in the approximation and the sparsity of the coefficient vector.

For a classification problem with $L$ classes, define the indicator function $\delta_l: \mathbb{R}^{N_\mathrm{tr}}\rightarrow \mathbb{R}^{N_\mathrm{tr}}$, $l=1,\ldots,L$, to set all coordinates corresponding to training samples \emph{not} in class $l$ to 0 (and to act as the identity on all remaining coordinates). After obtaining $\bm{\alpha}^*$ from Eq.~\eqref{eq:src_opt} or \eqref{eq:src_opt_noise}, the class label of $\bm{y}$ is predicted using
\begin{equation} \label{eq:src_class}
\classlabel(\bm{y}) = \arg \min_{1\leq l \leq L} \big\|\bm{y} - X_\mathrm{tr}\delta_l(\bm{\alpha}^*)\big\|_2.
\end{equation}

As mentioned in the introduction, it is assumed that by constraining the number of nonzero representation coefficients, nonzeros will occur at training samples most similar to the test sample, and thus Eq.~\eqref{eq:src_class} will reveal the correct class. This works as follows: It is assumed that each class manifold is a linear subspace spanned by its set of training samples, so that if the number of classes $L$ is large with regard to $N_\mathrm{tr}$, there exists a sparse (in terms of the entire training set) representation of $\bm{y}$ using training samples in its ground truth class. The coefficient vector $\bm{\alpha}^*$ is an attempt at finding this class representation, and Eq.~\eqref{eq:src_class} is used to allow for a certain amount of error.

In essence, SRC classifies $\bm{y}$ to the class that contributes the most to its sparse (via $\ell^1$-minimization) representation (or approximation, if Eq.~\eqref{eq:src_opt_noise} is used). SRC is summarized in Algorithm \ref{alg:src}.

\begin{algorithm}
\caption[Sparse Representation-Based Classification (SRC)]{Sparse Representation-Based Classification (SRC) \cite{wri:src}}
\label{alg:src}
\begin{algorithmic}[1]
\REQUIRE Matrix of normalized training samples $X_\mathrm{tr} \in \mathbb{R}^{m\times N_{\mathrm{tr}}}$, test sample $\bm{y} \in \mathbb{R}^m$, number of classes $L$, and error/sparsity trade-off $\lambda$ (optional)
\ENSURE The computed class label of $\bm{y}$: $\classlabel(\bm{y})$
\STATE Solve either the constrained problem in Eq.~\eqref{eq:src_opt} or the regularized problem in Eq.~\eqref{eq:src_opt_noise}.
\FOR{each class $l=1,\ldots, L$,}
\STATE Compute the norm of the class $l$ residual:
$\mathrm{err}_l(\bm{y}) := \big\|\bm{y} - X_\mathrm{tr}\delta_l(\bm{\alpha}^*)\big\|_2$. Set \\
$\classlabel(\bm{y}) = \arg \min_{1\leq l \leq L} \{\mathrm{err}_l(\bm{y})\}$.
\ENDFOR
\end{algorithmic}
\end{algorithm}


\section{The Conflict}\label{sec:conflict}

In classification problems, samples from the same class may be highly correlated. As demonstrated in Table \ref{tab:mc_exs}, the mutual coherence (as defined in Eq.~\eqref{eqn:mu}) of a training matrix $X = X_\mathrm{tr}$ is often quite large.

\begin{table}[!htbp]
\begin{center}
\begin{tabular}{|c|c|c|c|c|c|}
\hline 
Database & $N_\mathrm{tr}$ & $m$ & $m_\mathrm{PCA} = 30$ & $m_\mathrm{PCA} = 56$ & $m_\mathrm{PCA} = 120$ \\
\hline 
AR-1 \cite{AR:face} & 700 & 19800 & 0.9991 & 0.9987 & 0.9985 \\
AR-2 \cite{AR:face} & 1000 & 19800 & 0.9993 & 0.9988 & 0.9984 \\
Extended Yale Face Database B \cite{geo:illum} & 1216 & 32256 & 0.9951 & 0.9954 & 0.9941 \\
Database of Faces (formerly ``ORL'') \cite{att:orl} & 200 & 10304 & 0.9971 & 0.9970 & 0.9966 \\
\hline
\end{tabular}
\end{center}
\caption[Average mutual coherence computed from the training samples of face databases]{Average mutual coherence (over 10 trials) computed from training set $X_\mathrm{tr}$ of some popular face databases after PCA pre-processing to dimension $m_\mathrm{PCA}$. The original sample dimension is given by $m$. The training sets were chosen by randomly selecting half of the samples from each database, for a total of $N_\mathrm{tr}$ training samples. AR-1 contains all the unoccluded images (no sunglasses or scarf) from both sessions of the AR Face Database \cite{AR:face}; AR-2 contains all the unoccluded images from both sessions, as well as the occluded images from Session 1.}
\label{tab:mc_exs}
\end{table}

When $\mu(X_\mathrm{tr}) \approx 1$, the mutual coherence bound in Theorem \ref{thm:mc} becomes 
\begin{align*}
\|\bm{\alpha}\|_0 < \frac{1}{2} \Big(1+\frac{1}{\mu(X_\mathrm{tr})}\Big) \approx 1.
\end{align*} 
Since $\|\bm{\alpha}\|_0$ denotes the number of nonzero coefficients in the representation of $\bm{y}$ over $X_\mathrm{tr}$, it will never satisfy $\|\bm{\alpha}\|_0 < 1$. Thus we cannot use Theorem \ref{thm:mc} to prove $\ell^1/\ell^0$-equivalence in SRC, for example, on the databases used in Table \ref{tab:mc_exs}. 

It follows that the ``theory'' behind sparse representation-based methods for learning (like SRC) is missing a significant piece. In the next three sections, we aim to provide insight into the following three questions:
\begin{enumerate}[noitemsep, nolistsep]
\item Can Theorem \ref{thm:mc} \emph{ever} be used to prove $\ell^1/\ell^0$-equivalence in SRC? 
\item Regardless of theoretical guarantees, is $\ell^1$-minimization finding the sparsest solution in practice in SRC? 
\item What is the role of sparsity in SRC's classification performance? 
\end{enumerate}

\section{Mutual Coherence Equivalence and Classification}\label{MCD_Project_1}

In this section, we identify cases in which the condition given in Eq.~\eqref{eq:k_max} from Theorem \ref{thm:mc} \emph{provably} does not hold, and thus we cannot use Theorem \ref{thm:mc} to prove $\ell^1/\ell^0$-equivalence. We also discuss analogous results in the noisy case, i.e., Eq.~\eqref{eq:k_max_noise} in Theorem \ref{thm:mc_noise}. In particular, we are concerned with the applicability of these theorems for classification problems.

Before we begin, we take a moment to clarify notation:
\begin{itemize}[noitemsep,nolistsep]
\item In discussing compressed sensing in Section \ref{Equivalence_Guarantees}, we used $\bm{y}_0$ to refer to a clean measurement vector and $\bm{y} := \bm{y}_0+\bm{z}$ to refer to its noisy version. In contrast, in this section and in Section \ref{MCD_Project_3}, $\bm{y}$ may represent \emph{either} a clean or noisy measurement vector, or an arbitrary test sample (as it does in Algorithm~\ref{alg:src}). We do this because, in the context of representation-based classification, there are reasons other than noise in the test sample for allowing the equality $\bm{y} = X_\mathrm{tr}\bm{\alpha}$ to hold only approximately: the training data could also be corrupted, or we may want to relax the assumption that class manifolds are linear subspaces (perhaps this is only approximately, or locally, the case). Additionally, it is difficult to determine the amount of noise in test samples in real-world problems. To keep the situation general and to avoid confusion, we will only differentiate between $\bm{y}$ and $\bm{y}_0$ when we explicitly consider $\bm{y} = \bm{y}_0 + \bm{z}$ with $\|\bm{z}\|_2 \leq \zeta$ the noise vector, as in Donoho et al.'s Theorems \ref{thm:mc_noise} and \ref{thm:mc_stab}. 

When we explicitly consider data from a classification problem, we will use the subscript ``tr.'' That is, in the general compressed sensing representation $\bm{y} = X \bm{\alpha}$, we set $X = X_\mathrm{tr}$ when we want to denote a matrix of training samples, and when this is done, it is assumed that $\bm{y}$ specifically designates a test sample. 

\item For the underdetermined system $\bm{y} = X\bm{\alpha} $, we have already seen several instantiations of the coefficient vector $\bm{\alpha}$. We denoted the sparsest coefficient vector, i.e., the solution to the $\ell^0$-minimization problem given in Eq.~\eqref{eq:cs_l0}, by $\bm{\alpha} = \bm{\alpha}_0$, and we used $\bm{\alpha} = \bm{\alpha}_1$ and $\bm{\alpha} = \bm{\alpha}_{1,\epsilon}$ to denote the coefficient vectors found using $\ell^1$-minimization (in particular, the solutions to Eq.~\eqref{eq:cs_l1} and Eq.~\eqref{eq:cs_l1_error}, respectively). In contrast, $\bm{\alpha} = \bm{\alpha}^*$ denotes the solution to the SRC optimization problem (the solution to Eq.~\eqref{eq:src_opt} or \eqref{eq:src_opt_noise}). It is possible to have $\bm{\alpha}^* = \bm{\alpha}_1$ or $\bm{\alpha}^* = \bm{\alpha}_{1,\epsilon}$, depending on the optimization problem used in SRC and the amount of noise in the test sample. In particular, $\bm{\alpha}^* = \bm{\alpha}_1$ if Eq.~\eqref{eq:src_opt} is used in SRC, and $\bm{\alpha}^* = \bm{\alpha}_{1,\epsilon}$ if Eq.~\eqref{eq:src_opt_noise} is used and the test sample satisfies $\bm{y} = \bm{y}_0 + \bm{z}$ with $\|\bm{z}\|_2 \leq \zeta$.
\end{itemize}

\subsection{Preliminary Results}

We will use the following lemma which gives a lower-bound on mutual coherence in the underdetermined setting:

\begin{mylem}[Welch \cite{WELCH}, Rosenfeld \cite{ros:gram}] \label{lem:low_bound}

For $X \in \mathbb{R}^{m\times N}$ with normalized columns and $m<N$, we have that
\begin{equation}
\mu(X) \geq \sqrt{\frac{N-m}{m(N-1)}}.
\end{equation}
\end{mylem}

It is straightforward to show that Lemma \ref{lem:low_bound} implies that $\mu(X) \geq 1/m$, since $\sqrt{\frac{N-m}{m(N-1)}}$ monotonically increases in $N \in \mathbb{N}$ for $N>m$, with a minimum value of $1/m$ attained at $N = m+1$. Thus to have even a \emph{chance} of Theorem \ref{thm:mc} or \ref{thm:mc_noise} holding, we must have
\begin{equation}
\|\bm{\alpha}\|_0 < \frac{1}{c} \Big(1+\frac{1}{\mu(X)}\Big) \leq \frac{1}{c} \Big(1+m\Big), 
\end{equation}
where $c=2$ in the noiseless case and $c=4$ in the noisy case. 

We next consider the smallest possible value of the number of nonzeros $\|\bm{\alpha}\|_0$ in any classification problem representation $X_\mathrm{tr}\bm{\alpha} = \bm{y}$. Let us assume that the test sample is not a scalar multiple of any training sample. It follows that $\|\bm{\alpha}\|_0 \geq 2$. Thus in order for Theorem \ref{thm:mc} or \ref{thm:mc_noise} to hold, we must have
\begin{align*}
2 \leq \|\bm{\alpha}\|_0 < \frac{1}{c}\Big(1+\frac{1}{\mu(X_\mathrm{tr})}\Big) &\Rightarrow \mu(X_\mathrm{tr}) < \frac{1}{2c-1} \\
&\Rightarrow \mu(X_\mathrm{tr}) <
\begin{cases}
  1/3, & \text{noiseless case} \\
  1/7, & \text{noisy setting}.
\end{cases}
\end{align*}
Note that these upper bounds for $\mu(X_\mathrm{tr})$ are very small compared to the values of $\mu(X_\mathrm{tr})$ in Table \ref{tab:mc_exs}. 
These findings produce the following small-scale result:

\begin{myprop}
Suppose that $X_\mathrm{tr} \bm{\alpha} = \bm{y}$. If $m \leq 3$ and $\bm{y}$ is not a scalar multiple of any training sample, then the inequality in Eq.~\eqref{eq:k_max} with $X = X_\mathrm{tr}$ does not hold. That is, we cannot use Theorem \ref{thm:mc} to prove $\ell^1/\ell^0$-equivalence in SRC. 
\end{myprop}

\begin{proof}
  By Lemma \ref{lem:low_bound}, we must have that $\mu(X_\mathrm{tr}) \geq \frac{1}{m} \geq \frac{1}{3}$. $\square$
An analogous statement holds in the noisy setting (Theorem \ref{thm:mc_noise}) for $m \leq 7$. 
\end{proof}

\subsection{Main Result}

\begin{myprop}[Main Result]\label{prop:main_res}
Suppose that the sparsest representation of $\bm{y} \in \mathbb{R}^m$ over the dictionary $X = [\bm{x}_1,\ldots,\bm{x}_N]\in\mathbb{R}^{m \times N}$ is given by $\bm{y} = \alpha_{j_1} \bm{x}_{j_1} + \ldots + \alpha_{j_k}\bm{x}_{j_k}$ for $\{j_1,\ldots,j_k\} \subset \{1,\ldots,N\}$. Set $\widetilde{N}$ to be the number of columns of $X$ contained in 
\begin{align*}
\widetilde{\mathcal{X}} := \spn\{ \bm{x}_{j_1},\ldots,\bm{x}_{j_k} \}, 
\end{align*}
where clearly $\widetilde{N}\geq k$. If $\widetilde{N} > k$, then the inequality in Eq.~\eqref{eq:k_max} does not hold. That is, we cannot use Theorem \ref{thm:mc} to prove $\ell^1/\ell^0$-equivalence. 
\end{myprop}

\begin{proof}
  Suppose that $\widetilde{N} >k$. Then there are more than $k$ dictionary elements in the subspace $\widetilde{\mathcal{X}}$. Since the vectors $\bm{x}_{j_1},\ldots,\bm{x}_{j_k}$ are linearly independent (because otherwise, $\bm{y}$ could be expressed more sparsely), the dimension of $\widetilde{\mathcal{X}}$ is exactly $k$.

Define $\widetilde{X} \in \mathbb{R}^{m\times \widetilde{N}}$ to be the matrix of the $\widetilde{N}$ dictionary elements contained in $\widetilde{\mathcal{X}}$. Let the singular value decomposition of $\widetilde{X}$ be given by $\widetilde{X} = U \Sigma V^\transp$, and set $U_{k}$ to contain the first $k$ columns of $U$, $V_{k}$ to contain the first $k$ columns of $V$, and $\Sigma_{k}$ to contain the first $k$ columns and rows of $\Sigma$. Because $\widetilde{X}$ has rank $k$, we can alternatively write
\begin{align*}
\widetilde{X} = U_{k} \Sigma_{k} V_{k}^\transp.
\end{align*}
The $k \times \widetilde{N}$ matrix $U_{k}^\transp \widetilde{X}$ has the same mutual coherence as $\widetilde{X}$, since they have the same Gram matrices:
\begin{align*}
(U_{k}^\transp \widetilde{X})^\transp (U_{k}^\transp \widetilde{X}) &= \widetilde{X}^\transp U_{k} U_{k}^\transp \widetilde{X} \\
&= (U_{k} \Sigma_{k} V_{k}^\transp)^\transp U_{k} U_{k}^\transp (U_{k} \Sigma_{k} V_{k}^\transp) \\
&= V_{k}\Sigma_{k}^\transp U_{k}^\transp U_{k} U_{k}^\transp U_{k} \Sigma_{k} V_{k}^\transp \\
&= V_{k}\Sigma_{k}^\transp U_{k}^\transp U_{k} \Sigma_{k} V_{k}^\transp \\
&= (U_{k} \Sigma_{k} V_{k}^\transp)^\transp (U_{k} \Sigma_{k} V_{k}^\transp) \\
&= \widetilde{X}^\transp \widetilde{X}.
\end{align*}

By Lemma \ref{lem:low_bound}, we have that
\begin{align*}
\mu(X) \geq \mu(\widetilde{X}) = \mu(U_{k}^\transp \widetilde{X}) \geq \sqrt{\frac{\widetilde{N}-k}{k(\widetilde{N}-1)}} \geq \sqrt{\frac{(k+1)-k}{k((k+1)-1)}} = \frac{1}{k}.
\end{align*}
Thus the bound on $k$ in Theorem \ref{thm:mc} requires that
\begin{equation}
  \label{eq:k_max_proof}
  k < \frac{1}{2}\Big(1+ \frac{1}{\mu(X)}\Big) \leq \frac{1}{2}(1+k) \Rightarrow k < 1, 
\end{equation}
which contradicts with $k$ being a natural number.
\end{proof}
We present several corollaries to Proposition \ref{prop:main_res}. The first is a consequence applicable to any $\ell^1$-minimization problem, regardless of whether or not the dictionary elements have class structure:

\begin{mycor}[Consequence for general $\ell^1$-minimization] \label{cor:gen_l1_min}
If a measurement vector $\bm{y}\in \mathbb{R}^m$ is not at all sparse over the dictionary $X \in \mathbb{R}^{m\times N}$, i.e., if every representation of $\bm{y}$ requires no less than $m$ dictionary elements, then the condition in Eq.~\eqref{eq:k_max} from Theorem \ref{thm:mc} does not hold. 
\end{mycor}
\begin{proof}
  Because the dimension of $\widetilde{\mathcal{X}}$ (as defined in Proposition \ref{prop:main_res}) $k$ is actually $m$, every dictionary element is contained in
  $\widetilde{\mathcal{X}}$.
\end{proof}

Corollary \ref{cor:gen_l1_min} illustrates the importance of choosing a dictionary that awards a sparse representation of $\bm{y}$ in any application of $\ell^1$-minimization, including compressed sensing. 

The following corollary follows from the proof of Proposition \ref{prop:main_res}:

\begin{mycor} \label{cor:spark}
Let $X \in \mathbb{R}^{m\times N}$ with $m < N$, and let $k$ be any positive integer such that $k < N$. If \emph{any} set of $k$ linearly independent columns of $X$ spans an additional, distinct column of $X$, then the bound 
\begin{align*}
k < \frac{1}{2}\Big(1 + \frac{1}{\mu(X)}\Big)
\end{align*}
does not hold.

\end{mycor}

Of course, this bound will not hold for any larger values of $k$, either. This means that if we can find an integer $k$ satisfying the conditions of Corollary \ref{cor:spark}, then any attempt to prove $\ell^1/\ell^0$-equivalence using Theorem \ref{thm:mc} will require $X\bm{\alpha} = \bm{y}$ with $\|\bm{\alpha}\|_0 < k$.\footnote{Corollary \ref{cor:spark} can alternatively be proven using the equivalence theorem involving \emph{spark}; see the work of Donoho and Elad \cite{don:osr}.}

The following corollary is an explicit consequence for dictionaries consisting of training samples:

\begin{mycor}[Consequence for Class-Structured Dictionaries] \label{cor:y_dist}
Suppose that $\bm{y}$ is a test sample with $\|\bm{y}\|_2 = 1$, and define $\mu := \mu(X_\mathrm{tr})$. If adding $\bm{y}$ to the set of training samples does not increase its mutual coherence, that is, if $|\ip{\bm{y},\bm{x}_i}|\leq \mu$ for all $1\leq i \leq N_\mathrm{tr}$, i.e., $\mu([\bm{y},X_\mathrm{tr}]) =\mu$, then we cannot have both that (i) $X_\mathrm{tr}\bm{\alpha} = \bm{y}$ and (ii) $\|\bm{\alpha}\|_0 < (1/2)(1+(1/\mu(X_\mathrm{tr})))$.
\end{mycor}

\begin{proof}
  If we can write $X_\mathrm{tr}\bm{\alpha} = \bm{y}$ for $\|\bm{\alpha}\|_0 =: k$, then the $k$ (linearly independent) training samples with nonzero coefficients in the representation span a $k$-dimensional subspace containing $\bm{y}$. Setting $X = [\bm{y},X_\mathrm{tr}]$ in Corollary \ref{cor:spark}, we have that
\begin{align*}
k \nless \frac{1}{2}\Big(1+\frac{1}{\mu(X)}\Big) = \frac{1}{2}\Big(1+\frac{1}{\mu}\Big).
\end{align*}
On the other hand, if 
\begin{align*}
k < \frac{1}{2}\Big(1+\frac{1}{\mu(X)}\Big) = \frac{1}{2}\Big(1+\frac{1}{\mu}\Big)
\end{align*}
for some positive integer $k < N_\mathrm{tr}$, then also by Corollary \ref{cor:spark}, it must be the case that $\bm{y}$ is not contained in the subspace spanned by any $k$ linearly independent distinct columns of $X$, i.e., columns of $X_\mathrm{tr}$. Thus we cannot write $X_\mathrm{tr}\bm{\alpha} = \bm{y}$ for any $\bm{\alpha}$ satisfying $\|\bm{\alpha}\|_0 = k$.
\end{proof}

It might initially seem that the hypothesis of Corollary \ref{cor:y_dist} is unlikely to hold. However, if one assumes that the data is sampled randomly with test samples having the same distribution as the training samples in their ground truth classes, then the hypothesis that $\mu([\bm{y},X_\mathrm{tr}]) =\mu(X_\mathrm{tr})$ becomes much more probable. We discuss this further in Section \ref{MCD_Project_3}.

Our final corollary determines conditions under which the bound in Eq.~\eqref{eq:k_max} from Theorem \ref{thm:mc} is theoretically incompatible with the explicit assumptions made in SRC \cite{wri:src}. We review these assumptions briefly:

\begin{myassump}[Linear Subspaces] \label{assump:src_1}
The ground truth class manifolds of the given dataset are linear subspaces.
\end{myassump}

\begin{myassump}[Spanning Training Set]\label{assump:src_2}
The training matrix $X_\mathrm{tr}$ contains sufficient samples in each class to span the corresponding linear subspace.
\end{myassump}

\begin{mycor}[Consequence for SRC]\label{cor:src}
Suppose that the SRC Assumptions \ref{assump:src_1} and \ref{assump:src_2} hold. Let $\bm{y}$ have ground truth class $l$, and suppose that the number of class $l$ training samples, $N_l$, is large, i.e., $N_l > d_l$, for $d_l$ the dimension of the linear subspace representing the class $l$ manifold. Then there exists a test sample $\bm{y}$ which requires the maximum number $d_l$ of class $l$ training samples to represent it. If this representation of $\bm{y}$ is its sparsest representation over the dictionary $X_\mathrm{tr}$, then the condition in Eq.~\eqref{eq:k_max} from Theorem \ref{thm:mc} cannot hold. Thus we cannot use Theorem \ref{thm:mc} to prove $\ell^1/\ell^0$-equivalence in SRC.
\end{mycor}

Corollary \ref{cor:src} says that if we have a surplus of class $l$ training samples (i.e., more than enough to span the class $l$ subspace), then, provided that the ``class representations'' (representations of the test samples in terms of their ground truth classes) truly are the sparsest representations of the test samples over the training set (as argued by the SRC authors \cite{wri:src}), there will be some test samples for which Theorem \ref{thm:mc} cannot hold. These test samples are exactly those requiring $k=d_l$ class $l$ training samples in their representations. In general, such test samples must exist; otherwise, the dimension of the class $l$ subspace would be less than $d_l$. To reiterate, if everything we \emph{want} to happen in SRC actually happens (large class sizes, sparse class representations), then we cannot consistently use Theorem \ref{thm:mc} to prove $\ell^1/\ell^0$-equivalence.

On a more positive note, the assumptions in SRC make it possible to estimate whether or not the conditions of Proposition \ref{prop:main_res} hold. Though these conditions are difficult to check in general (if we knew the sparsest solution of $\bm{y}$ over the dictionary, then we would not need to use $\ell^1$-minimization to find it), the linear subspace assumption in SRC gives us a heuristic for doing so. We could potentially estimate the dimension of each class (using a method such as \emph{multiscale SVD} \cite{mag:msvd} or \emph{DANCo} \cite{cer:dan}, for example) and compare this with the number of training samples in that class. If the latter is larger than the former, then we expect that Theorem \ref{thm:mc} cannot be applied for some test samples.

In typical applications, we must deal with noisy data. Thus we should consider the application of Theorem \ref{thm:mc_noise} instead of Theorem \ref{thm:mc}. But this is immediate: Since the mutual coherence condition is stricter in the case of noise, the consequences of Proposition \ref{prop:main_res} and the above corollaries hold whenever the conditions are assumed to hold on the clean version of the data. In particular, Theorem \ref{thm:mc_noise} requires the \emph{existence} of a clean test sample $\bm{y}_0$ (even if it is unknown to us) that satisfies $X\bm{\alpha} = \bm{y}_0$ with $\|\bm{\alpha}\|_0 \leq (1/4)(1+(1/\mu(X)))$. Under the hypothesis of Corollary \ref{cor:y_dist} (setting $\bm{y}_0=\bm{y}$), such a $\bm{y}_0$ cannot exist. 


In concluding this section, we stress that the mutual coherence conditions in Theorems \ref{thm:mc} and \ref{thm:mc_noise} are sufficient, but not necessary, for $\ell^1/\ell^0$-equivalence. Thus it is possible for $\ell^1$-minimization to find (or closely approximate) the sparsest solution even when the conditions of these theorems do not hold. Whether or not this happens in the context of SRC is the topic of the next section.

\section{Equivalence on Highly-Coherent Data} \label{MCD_Project_2}

In this section, we investigate whether sparsity is reliably achieved via $\ell^1$-minimization on highly-correlated data, such as class-structured databases.

\subsection{Inspiration}

We are inspired by the data model and subsequent work of Wright and Ma \cite{wri:dense} (see also the work of Wright et al.\ \cite{wri:srcv}), which produces an $\ell^1/\ell^0$-equivalence guarantee for dictionaries containing vectors assumed to model facial images. We summarize their result briefly. 

Previous work has shown that the set of facial images of a fixed subject (person) under varying illumination conditions forms a convex cone, called an \emph{illumination cone}, in pixel space \cite{geo:illum, bel:what}. Wright and Ma demonstrate that in fact the set of facial images under varying illuminations over \emph{all subjects combined} exhibits this cone structure. For example, they show that this is the case for the entire set of (raw) samples from the Extended Yale B Face Database \cite{geo:illum}. Further, this cone becomes extremely narrow, i.e., a ``bouquet,'' as the number of pixels grows large \cite{wri:dense}. These findings reiterate that class-structured data, particularly face databases, are highly-coherent. 

Lee et al.\ \cite{lee:linss} showed that any image from the illumination cone can be expressed as a linear combination of just a few images of the same subject under varying lighting conditions. In other words, illumination cones are well-approximated by linear subspaces. Thus the SRC condition that class manifolds are (approximately) linear subspaces presumably holds for databases made up of facial images under varying lighting conditions. Given a facial image $\bm{y} \in \mathbb{R}^m$ that may be occluded or corrupted by noise, $\bm{y}$ can thus be expressed as
\begin{equation} \label{eq:signal_error}
\bm{y} = X_\mathrm{tr} \bm{\alpha}_0 + \bm{z}_0,
\end{equation}
 given that certain requirements are satisfied in the sampling of the training data. By the above model, $\bm{\alpha}_0$ is assumed to be non-negative (a result of the illumination cone model \cite{wri:srcv, geo:illum}) and sparse, containing nonzeros at training samples that represent the same subject as $\bm{y}$ (i.e., are in the same class). Additionally, $\bm{z}_0$ is an (unknown) error vector with nonzeros in only a fraction of its coordinates; i.e., the model assumes that only a portion of the pixels are occluded or corrupted \cite{wri:srcv}. Note that this is not quite the same situation as in the condition for $\ell^1/\ell^0$-equivalence in the noisy setting given in Theorem \ref{thm:mc_noise}. One difference is that in Eq.~\eqref{eq:signal_error} above, $\bm{z}_0$ is bounded in terms of $\ell^0$-norm (sparsity) with no limit on $\ell^2$-norm (magnitude), whereas in Theorem \ref{thm:mc_noise}, $\bm{z}$ is bounded in terms of magnitude but not sparsity.

The goal, as one might expect, is to recover $\bm{\alpha}_0$ from Eq.~\eqref{eq:signal_error}. In the SRC paper \cite{wri:src}, Wright et al.\ use $\ell^1$-minimization to do this. In particular, they solve
\begin{equation} \label{eq:opt_prob}
(\widehatbv{\bm{\alpha}}_1,\bm{z}_1) := \arg\min \|\bm{\alpha}\|_1 + \|\bm{z}\|_1 \text{\textnormal{ subject to }} \bm{y} = X_\mathrm{tr} \bm{\alpha} + \bm{z},
\end{equation}
and they show that this version of SRC produces very good classification results on occluded or corrupted facial images. (Again, note that $\widehatbv{\bm{\alpha}}_1$ is different from both $\bm{\alpha}_1$ and $\bm{\alpha}_{1,\epsilon}$ discussed earlier, as there is a sparsity constraint instead of an $\ell^2$-norm bound on the noise component $\bm{z}_0$.)

In a later paper, Wright, et al.\ \cite{wri:srcv} correctly note that the usual $\ell^1/\ell^0$-equivalence theorems do not hold on the highly-correlated data in $X_\mathrm{tr}$, and so it cannot be determined whether or not the $\ell^1$-minimized solution $\widehatbv{\bm{\alpha}}_1$ in Eq.~\eqref{eq:opt_prob} is equal to (what is assumed to be) the true sparsest solution $\bm{\alpha}_0$. Fortunately, Wright and Ma \cite{wri:dense} proved a theorem that gives sufficient conditions for this equivalence under an assumed model (called the \emph{bouquet model}) of facial images; see also Wright et al.'s version \cite{wri:srcv}. To state the theorem, we will need the following definition:

\begin{mydef} [Proportional Growth \cite{wri:dense}]
A sequence of signal-error problems $\bm{y} = X \bm{\alpha}_0 + \bm{z}_0$, for $X \in \mathbb{R}^{m \times N}$, exhibits \emph{proportional growth} with parameters $\delta>0$, $\rho \in (0,1)$, and $\beta >0$, if $N = \floor{\delta m}$, $\|\bm{z}_0\|_0 = \floor{\rho m}$, and $\|\bm{\alpha}_0\|_0 = \floor{\beta m}$. 
\end{mydef}
It follows that $\delta$ is the redundancy factor in the dictionary $X$ and $\rho$ and $\beta$ control the sparsity of $\bm{z}_0$ and $\bm{\alpha}_0$, respectively. Here, $\beta$ is assumed to be small and may depend on $\delta$ and $\rho$.

We are now in a position to state Wright and Ma's main theorem:

\begin{mythm} [Wright and Ma \cite{wri:dense}]\label{thm:wri_mu}
Fix any $\delta >0$ and $\rho<1$. Suppose that $X$ is distributed according to the \emph{bouquet model} given by 
\begin{equation} \label{eq:cross}
X = [\bm{x}_1,\ldots,\bm{x}_N]\in\mathbb{R}^{m \times N},\: \bm{x}_i \stackrel{\mathrm{i.i.d.}}{\sim} \mathcal{N}(\bm{\mu},(\nu^2/m)I_m), \:\|\bm{\mu}\|_2 = 1, \: \|\bm{\mu}\|_\infty \leq C_\mu m^{-1/2},\: C_\mu \geq 1
\end{equation} 
for $\nu$ sufficiently small. Also suppose that the sequence of signal-error problems $\bm{y} = X \bm{\alpha}_0 + \bm{z}_0$ for $X \in \mathbb{R}^{m \times N}$ exhibits proportional growth with parameters $\delta$, $\rho$, and $\beta$. Suppose further that $J \subset \{1,\ldots,m\}$ is a uniform random subset of size $\rho m$, and that $\bm{\sigma} \in \mathbb{R}^m$ with entries of $\bm{\sigma}_J$ i.i.d.\ $\pm 1$ (independent of $J$) and $\bm{\sigma}_{J^C} = \bm{0}$. Lastly assume that $m$ is sufficiently large. Then with probability at least $1-C\exp(-\gamma^* m)$ in $X$, $J$, and $\bm{\sigma}$, for all $\bm{\alpha}_0$ with $\|\bm{\alpha}_0\|_0 \leq \beta^* m$ and any $\bm{z}_0$ with sign vector $\bm{\sigma}$ and support $J$, we have
\begin{align*}
(\bm{\alpha}_0,\bm{z}_0) = \arg\min_{\bm{\alpha},\bm{z}} \|\bm{\alpha}\|_1 + \|\bm{z}\|_1 \text{\textnormal{ subject to }} X\bm{\alpha} + \bm{z} = X\bm{\alpha}_0 + \bm{z}_0. 
\end{align*}
\end{mythm}
Here, $C$ is a numerical constant and $\beta^*$ and $\gamma^*$ are positive constants (independent of $m$) which depend on $\delta$, $\rho$, and $\nu$. By ``$\nu$ sufficiently small'' and ``$m$ sufficiently large,'' Wright and Ma mean that there exist constants $0<\nu<\nu^*$ and $m >m^*$ (independent of $m$) such that $\nu^*(\delta,\rho) >0$ and $m^*(\delta,\rho,\nu)>0$, respectively.\footnote{The relationship between $\beta^*$ and $\beta$ is not explicitly stated, but it makes sense that $\beta^*\leq \beta$ by the proportional growth assumption. Further, if $\beta = \beta(\delta,\rho)$, then since $\beta^* = \beta^*(\delta,\rho,\nu)$, we can likely alternatively write $\beta^* = \beta^*(\beta,\nu)$.} 
This theorem illustrates that $\ell^1/\ell^0$-equivalence can provably hold on
the classification of highly-coherent data via random database model.

\begin{remark}
Despite its applicability to highly-coherent data, Theorem \ref{thm:wri_mu} does not prove that $\ell^1/\ell^0$-equivalence holds in SRC. First of all, the theorem requires that $m$ be sufficiently large, which may not be the case, especially when feature extraction is used. Second, the model in Theorem \ref{thm:wri_mu} does not explicitly deal with class-structured data. A true face recognition model should account for the individual subjects, with samples in the same class being (on average) more correlated than those from different classes. Thus our model should contain ``sub-bouquets'' (i.e., the classes) inside the larger bouquet. 
\end{remark}

\subsection{Experiments}
With these changes in mind, we design a random database model that will allow us to study the relationship between sparsity and $\ell^1$-minimization on highly-coherent and class-structured data, such as the images used in face recognition. First, we specify the dimension $m$, the number of classes $L$, and the number of samples $N_l \equiv N_0$, $1\leq l \leq L$ in each training class. We require that $N_\mathrm{tr}=N_0 L > m$ so that the resulting dictionary of training samples leads to an underdetermined system. We then randomly generate training data with an increasing amount of cone/bouquet structure as well as class structure, along with a test sample---with known sparse coefficient vector $\bm{\alpha}_0$---generated as a linear combination of training samples from a single class. We run a fixed number of trials of the experiment at each of 11 increasing values of coherence (we call these \emph{stages}) and determine at which stages $\ell^1$-minimization can closely (or exactly) recover $\bm{\alpha}_0$.

\subsubsection{Experimental Setup}

For each generated training set $X_\mathrm{tr} = [X^{(1)},\ldots,X^{(L)}] \in \mathbb{R}^{m\times N_\mathrm{tr}}$, we set the (clean) test sample $\bm{y}_0$ to be a random vector in the positive span of the class 1 data. That is, we set 
\begin{align*}
\bm{y}_0 := \alpha_1^{(1)}\bm{x}_1^{(1)} + \ldots + \alpha_{N_0}^{(1)}\bm{x}_{N_0}^{(1)}, 
\end{align*}
where $X^{(1)} := [\bm{x}_1^{(1)},\ldots,\bm{x}_{N_0}^{(1)}]$ and $\alpha_j^{(1)} \sim \unif(0,1)$, $1\leq j \leq N_0$. We then define 
\begin{align*}
\bm{\alpha}_0 := [\alpha_1^{(1)},\ldots,\alpha_{N_0}^{(1)},0,\ldots,0]^\transp \in \mathbb{R}^{N_\mathrm{tr}}.
\end{align*}
Given this setup, we want to see if $\ell^1$-minimization will recover $\bm{\alpha}_0$, i.e., if the solution 
\begin{align*}
\bm{\alpha}_1 := \arg\min_{\bm{\alpha}\in\mathbb{R}^{N_\mathrm{tr}}} \|\bm{\alpha}\|_1 \text{ subject to } X_\mathrm{tr}\bm{\alpha} = \bm{y}_0
\end{align*}
is equal to $\bm{\alpha}_0$. Note that for large $L$, $\bm{\alpha}_0$ can be viewed as a sparse vector. 

In Stage 1 of our model, the training data has no class or cone structure and is randomly generated on the unit sphere $S^{m-1}$. It has been shown experimentally that, for $N_\mathrm{tr} = 2m$ and $m$ sufficiently large, an $\ell^1$-minimization solution with no more than $(3/10)m$ nonzeros is enough to ensure it is the sparsest solution with high probability \cite{don:und}. Thus we expect to see exact recovery in Stage 1 for values of $N_0$, $m$, and $L$ satisfying these requirements. 

To add both bouquet and class (or sub-bouquet) structure to the training set in subsequent stages, we define the cone mean $\overline{\bm{x}}$ and the class means $\{\overline{\bm{x}}_1,\ldots,\overline{\bm{x}}_L\}$. At Stage $i$, $1\leq i \leq 11$, we set $\overline{\bm{x}} \sim \mathcal{N}(\bm{0},I_m)$ and then modify $\overline{\bm{x}} \leftarrow \mu_i\overline{\bm{x}}/\|\overline{\bm{x}}\|_2$, where $\mu_i := (i-1)/10$ effectively increases the cone mean from $\bm{0}$ as $i$ increases. Next, each class mean is randomly generated depending on $\overline{\bm{x}}$ as follows: For each class $1,\ldots,L$, we sample $\overline{\bm{x}}_l$ from $\mathcal{N}(\overline{\bm{x}},\eta_i m^{-1/2}I_m)$ for $\eta_i := 2/i$ (so that each class mean becomes increasingly close to the cone mean) and then modify $\overline{\bm{x}}_l \leftarrow \mu_i\overline{\bm{x}}_l/\|\overline{\bm{x}}_l\|_2$, $1\leq l \leq L$. Lastly, to generate the training samples in class $1\leq l \leq L$, we sample $\bm{x}_j^{(l)}$ from $\mathcal{N}(\overline{\bm{x}}_l,(\eta_im^{-1/2}/L)I_m)$ and then modify $\bm{x}_j^{(l)} \leftarrow \bm{x}_j^{(l)} / \|\bm{x}_j^{(l)}\|_2$, $1\leq j \leq N_0$. Figure \ref{fig:Ex_Stages} shows an example of Stage $i \in \{1,3,\ldots,11\}$ with $m=3$, $N_0 = 5$, and $L=4$.

\begin{figure}
\centering
\begin{subfigure}[b]{0.26\textwidth}
\centering
	\includegraphics[width=\linewidth]{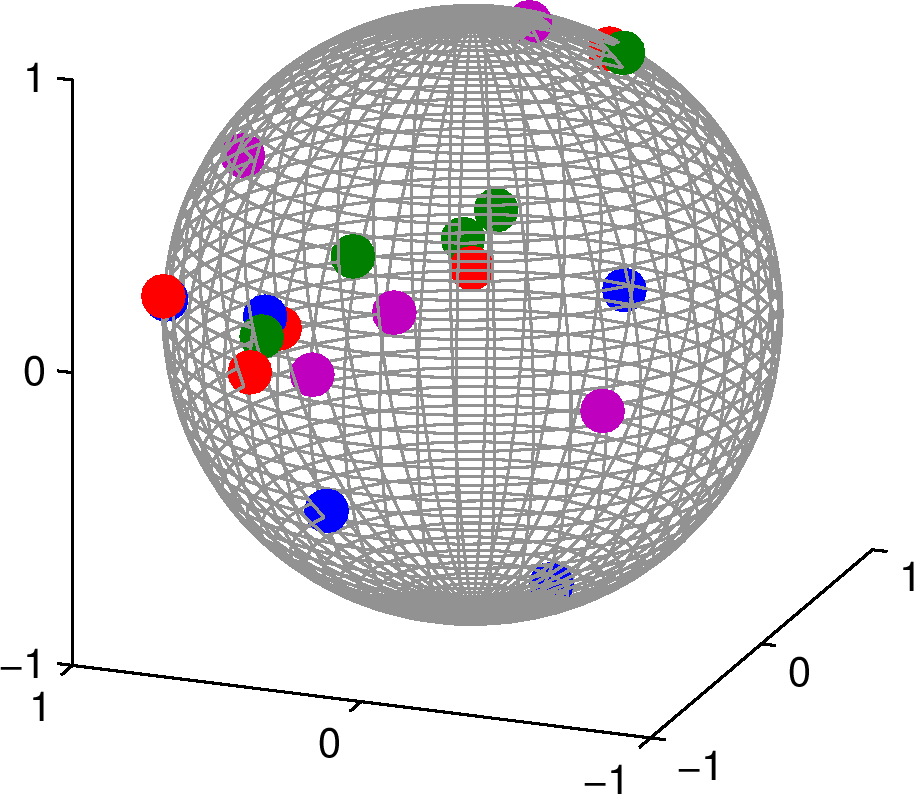}
		\caption{Stage 1}
	\label{fig:stage_1} \hfill%
\end{subfigure}
\begin{subfigure}[b]{0.26\textwidth} 
\centering
	\includegraphics[width=\linewidth]{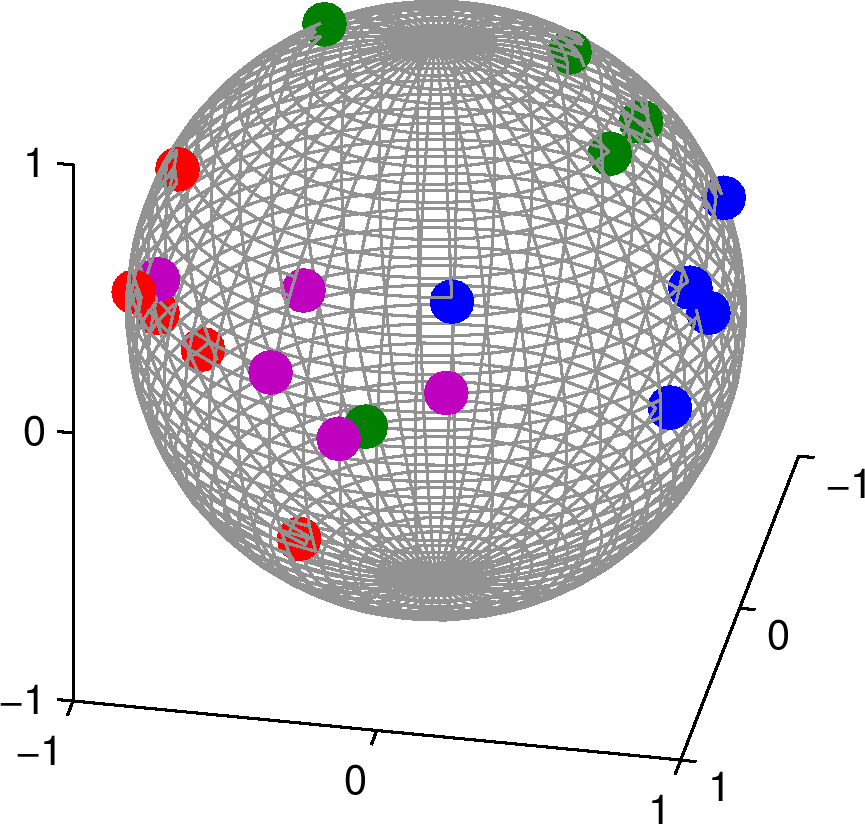}
		\caption{Stage 3}
	\label{fig:stage_3}
	\hfill%
\end{subfigure}
\begin{subfigure}[b]{0.26\textwidth} 
\centering
	\includegraphics[width=\linewidth]{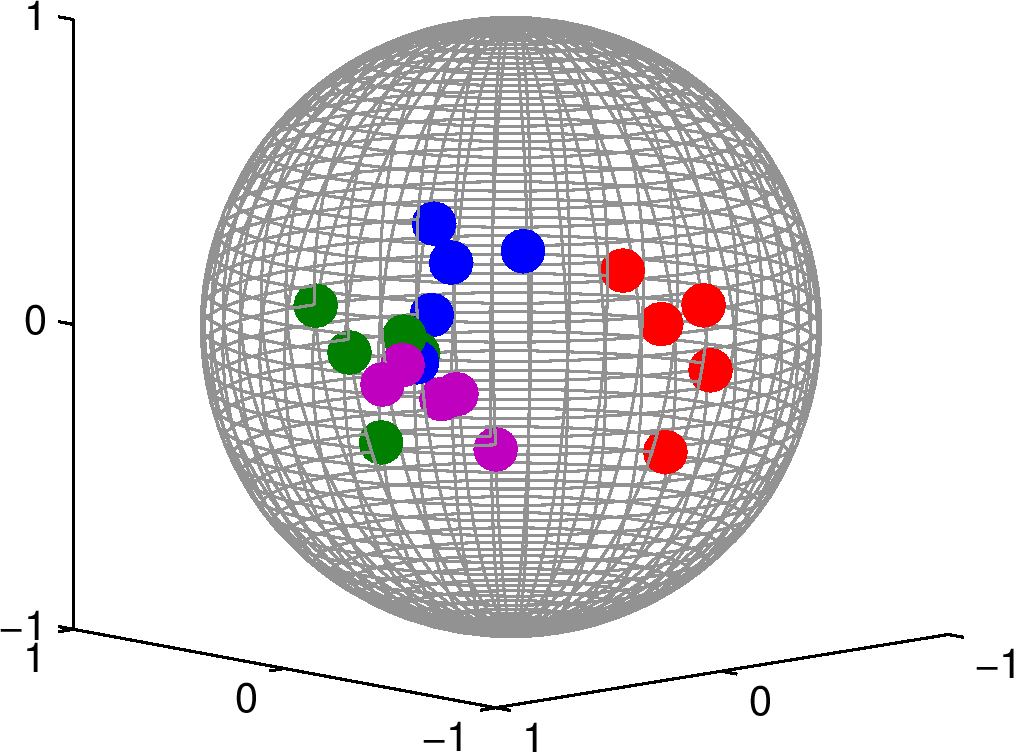}
		\caption{Stage 5}
	\label{fig:stage_5} \hfill%
\end{subfigure}
\begin{subfigure}[b]{0.26\textwidth}
\centering
	\includegraphics[width=\linewidth]{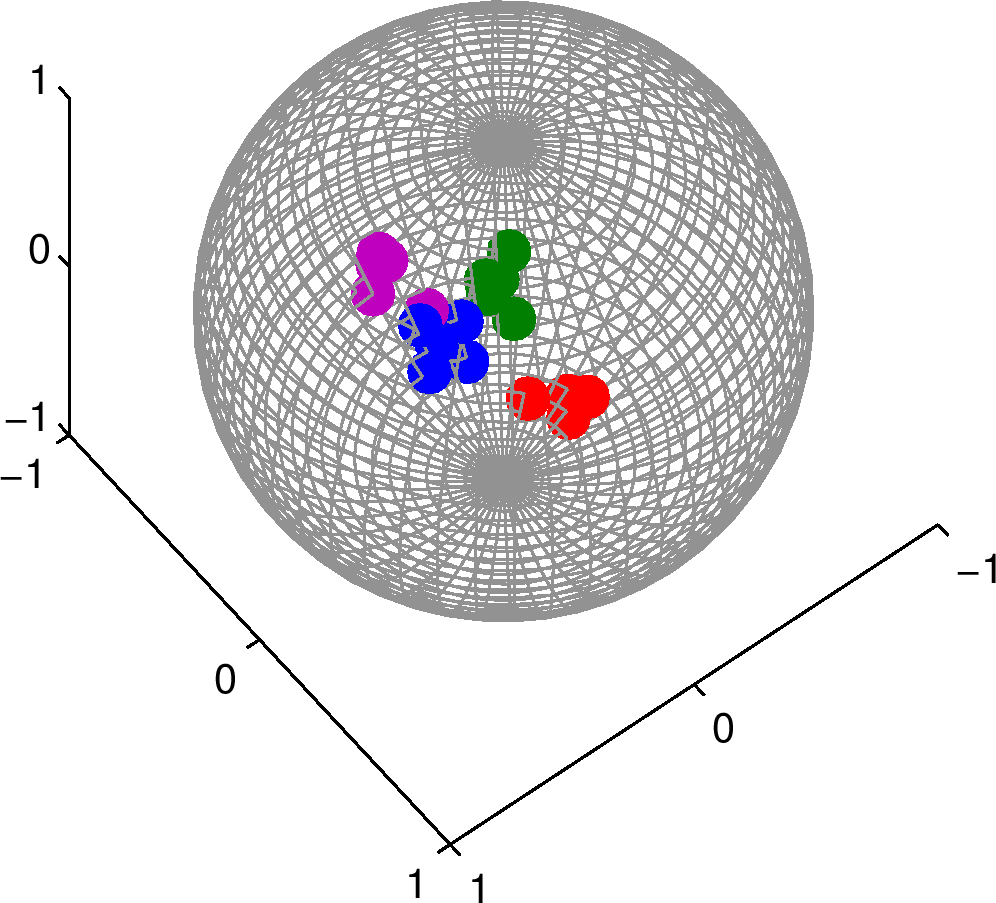}
		\caption{Stage 7}
	\label{fig:stage_7} \hfill%
\end{subfigure}
\begin{subfigure}[b]{0.26\textwidth} 
\centering
	\includegraphics[width=\linewidth]{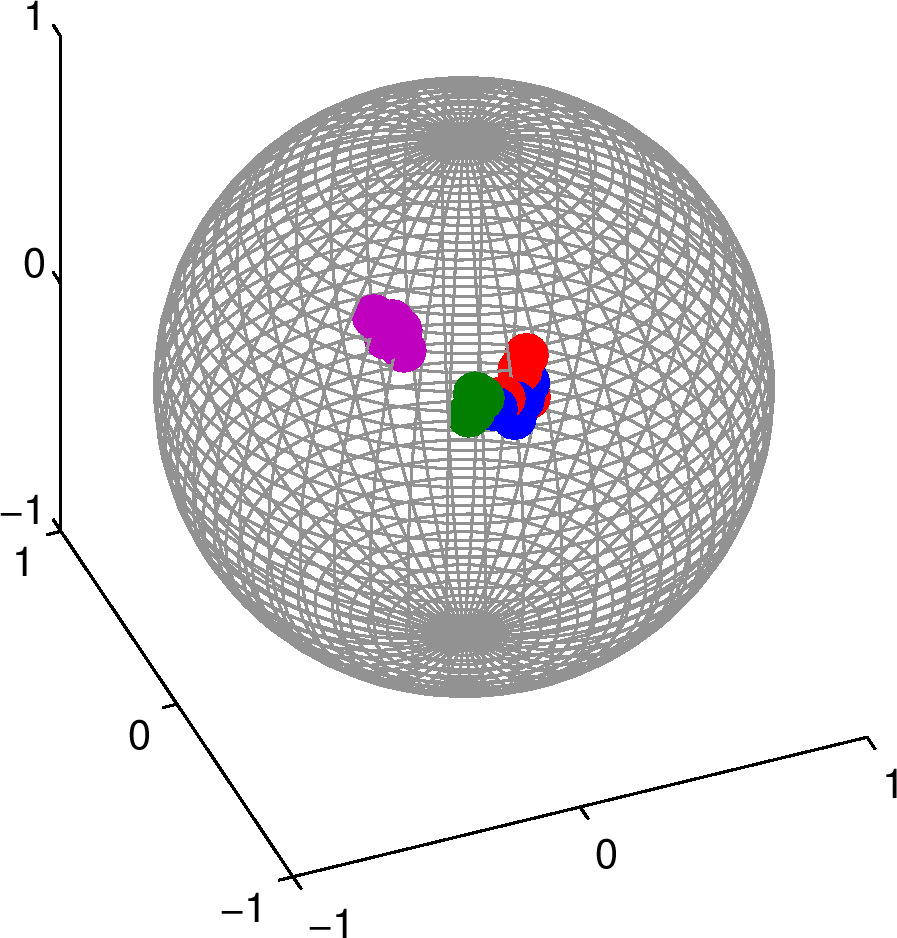}
		\caption{Stage 9}
	\label{fig:stage_9}
	\hfill%
\end{subfigure}
\begin{subfigure}[b]{0.26\textwidth} 
\centering
	\includegraphics[width=\linewidth]{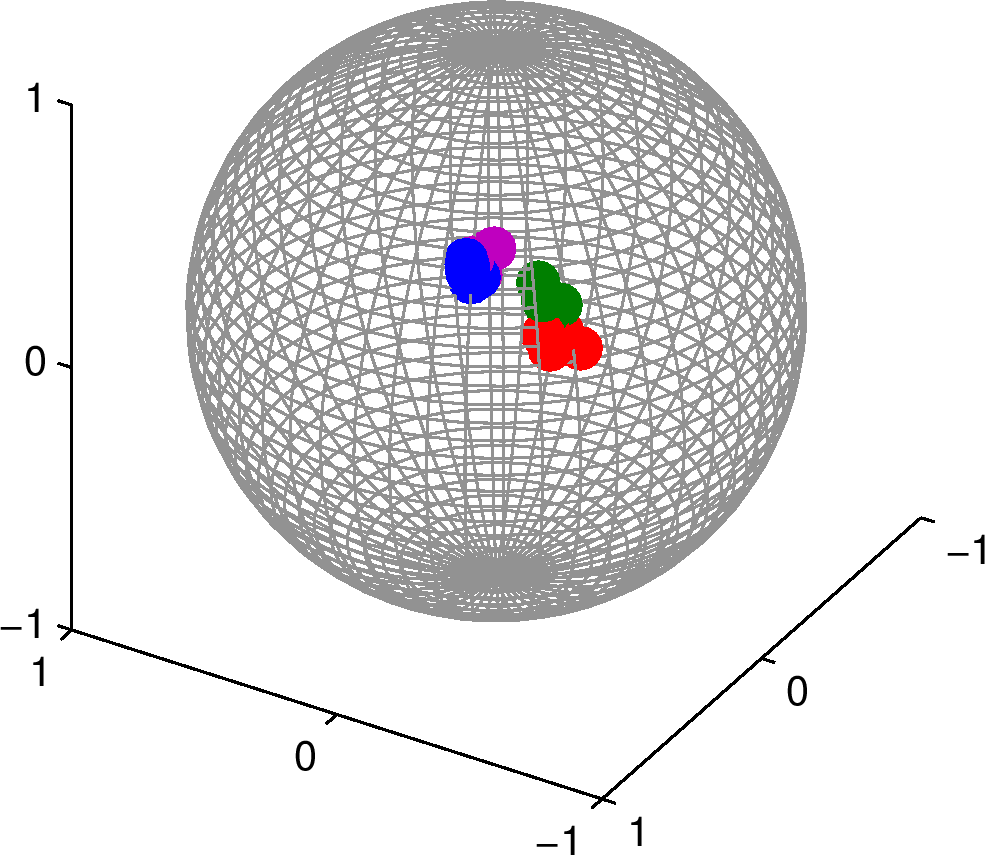}
		\caption{Stage 11}
	\label{fig:stage_11} \hfill%
\end{subfigure}
\caption[Illustration of random database model over increasing values of within-class correlation]{An example of the generated training data from the random database model across odd-numbered stages (as mutual coherence increases) with $m=3$, $N_0 =5$, and $L=4$. The colors denote the classes. Plots have been manually rotated to aid in visualization. (a) At Stage 1, data is uniformly spread out on the sphere; (b)-(f) At increasingly higher stages, the dataset as a whole becomes more bouquet-shaped, as does the data in each class.}
\label{fig:Ex_Stages}
\end{figure}

We perform experiments using four different specifications for the triples $(N_0,m,L)$, as shown in Table \ref{tab:triples}. By design, we have that $\|\bm{\alpha}_0\|_0 = N_0$ in our experiments (though we will also briefly look at the case that $\|\bm{\alpha}_0\|_0 < N_0$). Note that: (i) the inequality $\|\bm{\alpha}_0\|_0< (3/10)m$ is satisfied for each of the specifications in Table \ref{tab:triples}; and (ii) these numbers are similar to what we might expect to see in classification of a face database (after some method of feature extraction is applied, as is generally required by SRC for face classification). 

\begin{table}[!htbp]
\begin{center}
\resizebox{\columnwidth}{!}{%
\begin{tabular}{|c|c|c|c|c|c|}
\hline 
ID & $(N_0,m,L)$ &$\|\bm{\alpha}_0\|_0/m$ & $\|\bm{\alpha}_0\|_0/N_\mathrm{tr}$ & Redundancy ($N_0L/m$) & Comments \\
\hline 
DB-1 & (5,50,20) & 1/10 & $1/20 $ & 2:1 & \small{Baseline redundancy; $N_0$ small with respect to $m$, $N_\mathrm{tr}$} \\   
DB-2 & (10,50,10) & 1/5 & $1/10 $ & 2:1 & \small{Baseline redundancy; $N_0$ less small with respect to $m$, $N_\mathrm{tr}$} \\ 
DB-3 & (10,50,50) & 1/5 & $1/50$ & 10:1 & \small{High redundancy; large $L$} \\ 
DB-4 & (5,200,50) & 1/40 & $1/50$ & 5:4 & \small{Low redundancy; large $L$} \\ 
 \hline
\end{tabular}%
}
\end{center}
\caption[Specification of parameters in the random database model]{Specification of parameters in the random database model.}
\label{tab:triples}
\end{table}

\subsubsection{Experimental Results: No Noise} \label{sec:exps_no_noise}

\textit{Accuracy of recovery:} We consider the following quantities for evaluating the success of $\ell^1/\ell^0$-recovery:
\begin{itemize}[noitemsep,nolistsep]
\item The average normalized $\ell^2$-error 
\begin{equation}
\err_{\ell^2}:=\|\bm{\alpha}_1 - \bm{\alpha}_0\|_2/\|\bm{\alpha}_0\|_2
\end{equation} 
between the $\ell^1$-minimized solution $\bm{\alpha}_1$ and $\bm{\alpha}_0$,
\item The average number of nonzeros of $\bm{\alpha}_1$ occurring at training samples \emph{not} in class 1 (we call these ``off-support'' nonzeros, because they are nonzeros not in the support of $\bm{\alpha}_0$), divided by the total number of nonzeros. That is, let $\bm{\alpha}_1^{\mathrm{off-supp}} $ be the result of setting all entries in $\bm{\alpha}_1$ that are in class 1 to zero. Then this error is defined as 
\begin{align*}
\err_{\supp}:= \frac{\|\bm{\alpha}_1^{\mathrm{off-supp} }\|_0}{\|\bm{\alpha}_1\|_0},
\end{align*}
\item Since $\err_\mathrm{supp}$ does not provide information regarding the \emph{size} of the off-support nonzero coefficients, we also consider
\begin{align*}
\err_{\mathrm{supp}({\ell^2})} := \frac{\|\bm{\alpha}_1^{\mathrm{off-supp} }\|_2}{\|\bm{\alpha}_1\|_2} \text{ and }
\err_{\mathrm{supp}({\ell^1})} := \frac{\|\bm{\alpha}_1^{\mathrm{off-supp} }\|_1}{\|\bm{\alpha}_1\|_1},
\end{align*}
\item The average mutual coherence of the training set, $\mu(X_\mathrm{tr}) =: \mu$.
\end{itemize}

It is informative to consider the effect that the support error quantities would (hypothetically) have on the classification performance of SRC. Recall that, in the case that the clean test sample $\bm{y}_0$ is known, SRC computes the class residuals $\err_l(\bm{y}_0):=\|\bm{y}_0 - X_\mathrm{tr}\delta_l(\bm{\alpha}_1)\|_2$, $1\leq l\leq L$, and assigns $\bm{y}_0$ to the class with the smallest residual. Thus if $\err_{\supp}$, $\err_{\mathrm{supp}({\ell^2})}$, and $\err_{\mathrm{supp}({\ell^1})}$ are small, we expect that SRC will have an easier time classifying the test sample correctly (recall that these quantities measure the residual from the correct class $l=1$). For example, if all the support error quantities are 0, then $\delta_1(\bm{\alpha}_1) = \bm{\alpha}_1$ and it follows that the class 1 residual $\err_1(\bm{y}_0) = 0$ and $\err_l(\bm{y}_0) = \|\bm{y}_0\|_2$ for $2\leq l \leq L$. This corresponds to the ideal classification scenario.

We compute the average quantities $\err_{\ell^2}$, $\err_{\supp}$, $\err_{\mathrm{supp}({\ell^2})}$, $\err_{\mathrm{supp}({\ell^1})}$, and $\mu$ over 1000 trials at each stage, using the $\ell^1$-minimization algorithm HOMOTOPY \cite{don:hom, asif:hom} with error/sparsity trade-off parameter $\lambda = 10^{-10}$ (to force near-exactness in the approximation). The results are shown in Figure \ref{fig:HOM_no_noise}.

\begin{figure}%
\centering
\begin{subfigure}[b]{0.45\textwidth}
\centering
	\includegraphics[width=\linewidth]{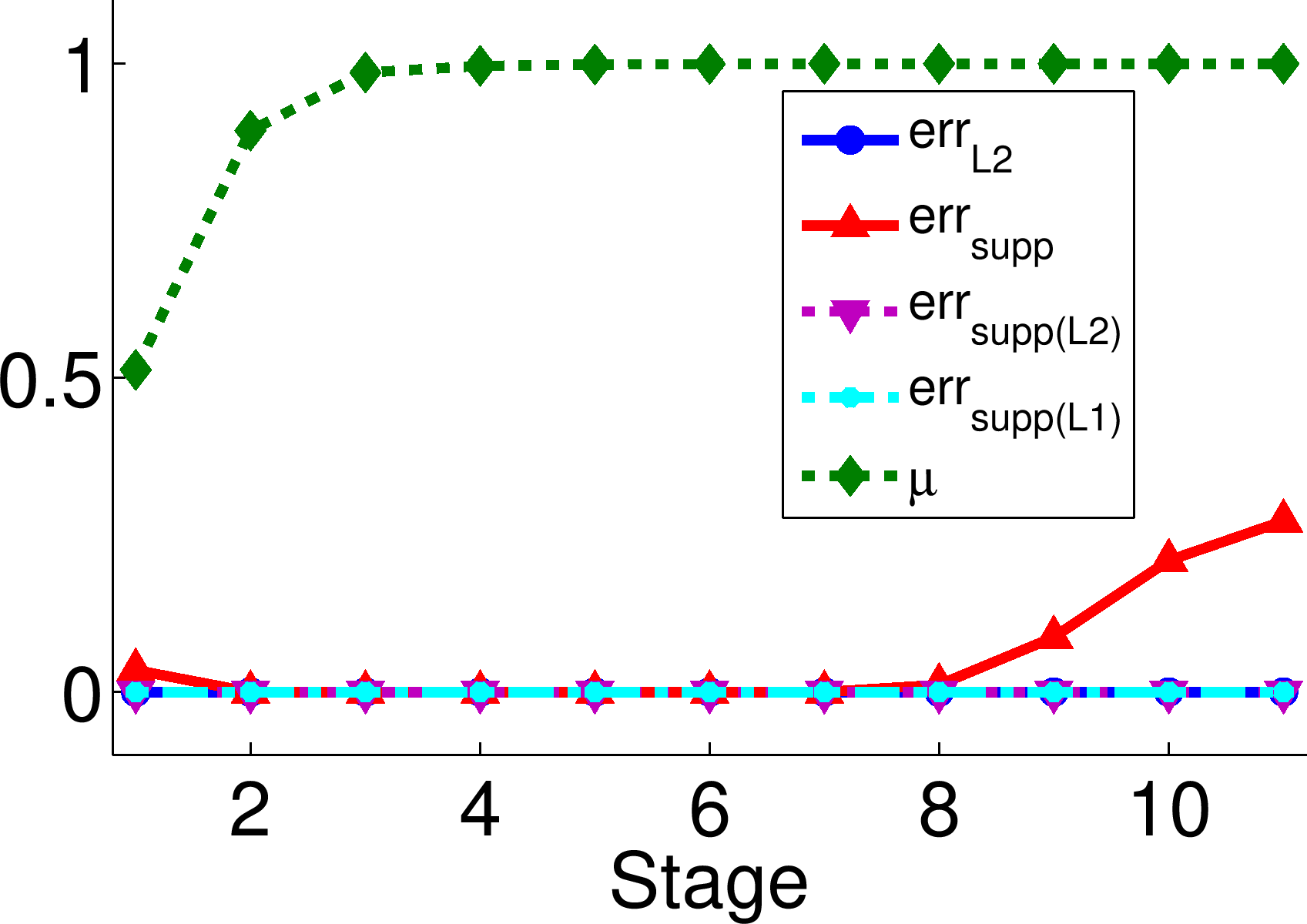}
		\caption{DB-1: $(N_0,m,L) = (5,50,20)$}
	\label{fig:DB_1_HOM_no_noise} \hfill%
\end{subfigure}
\begin{subfigure}[b]{0.45\textwidth} 
\centering
	\includegraphics[width=\linewidth]{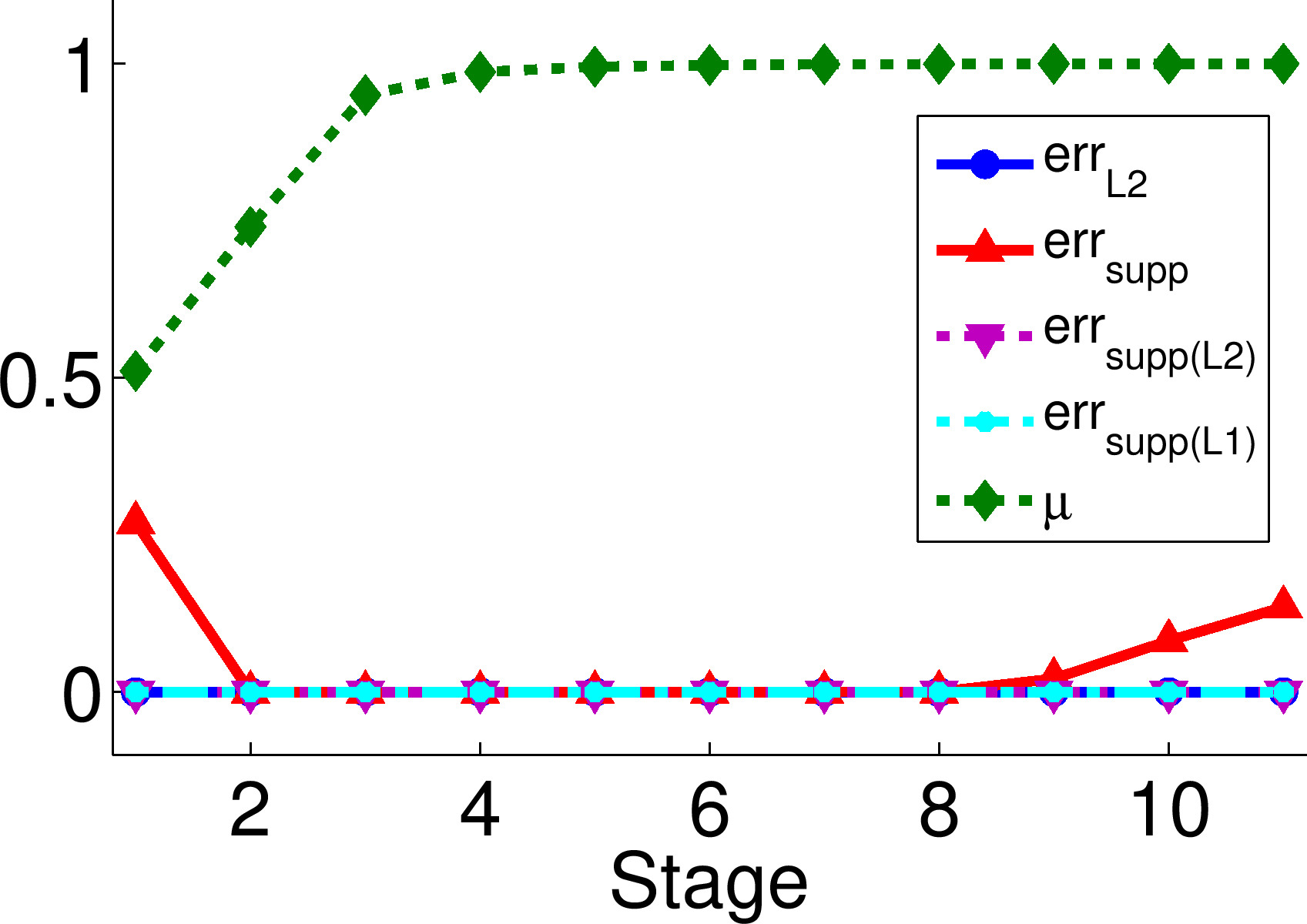}
		\caption{DB-2: $(N_0,m,L) = (10,50,10)$}
	\label{fig:DB_2_HOM_no_noise} \hfill%
\end{subfigure}
\begin{subfigure}[b]{0.45\textwidth} 
\centering
	\includegraphics[width=\linewidth]{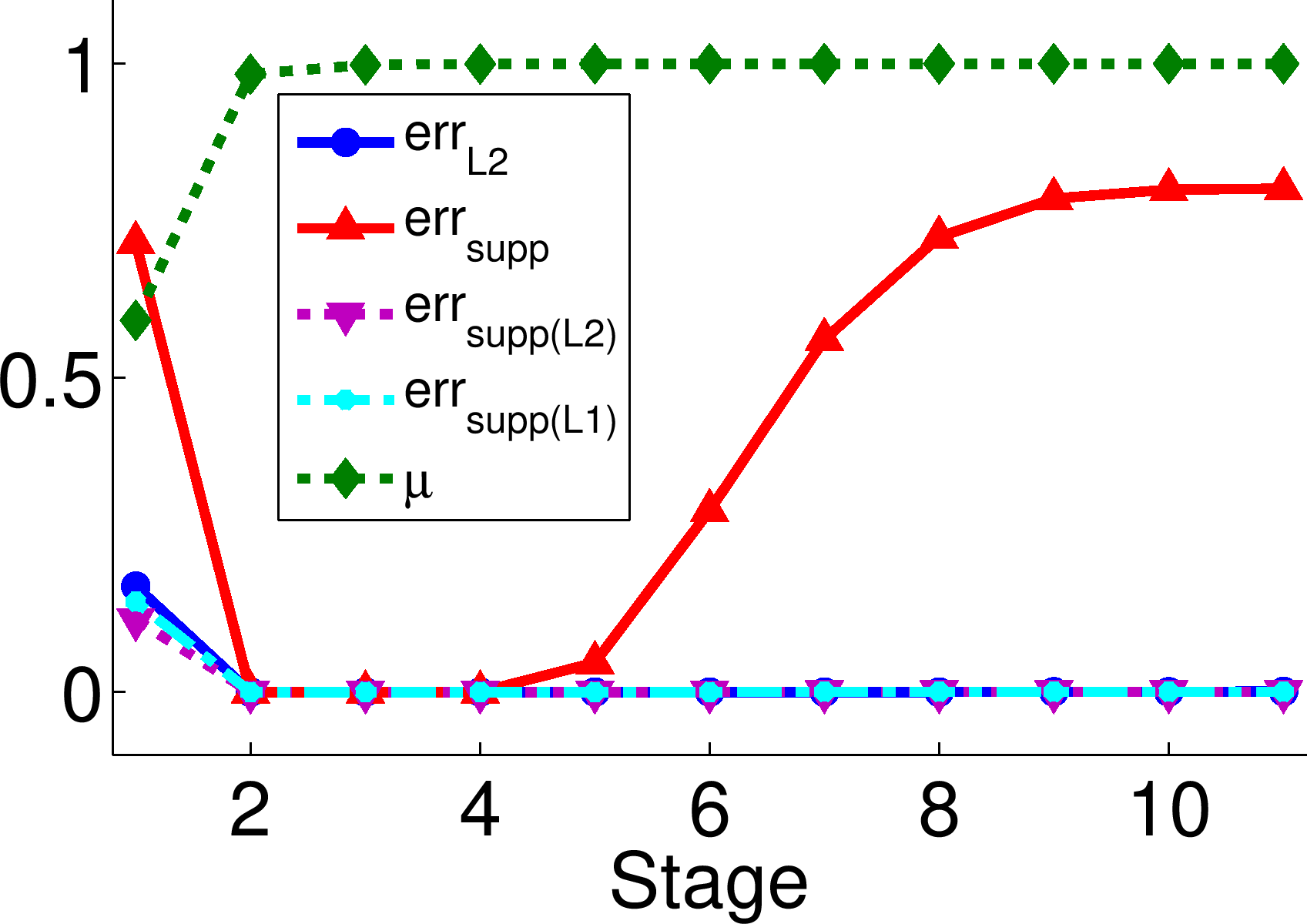}
		\caption{DB-3: $(N_0,m,L) = (10,50,50)$}
	\label{fig:DB_3_HOM_no_noise}
	\hfill%
\end{subfigure}
\begin{subfigure}[b]{0.45\textwidth}
\centering
	\includegraphics[width=\linewidth]{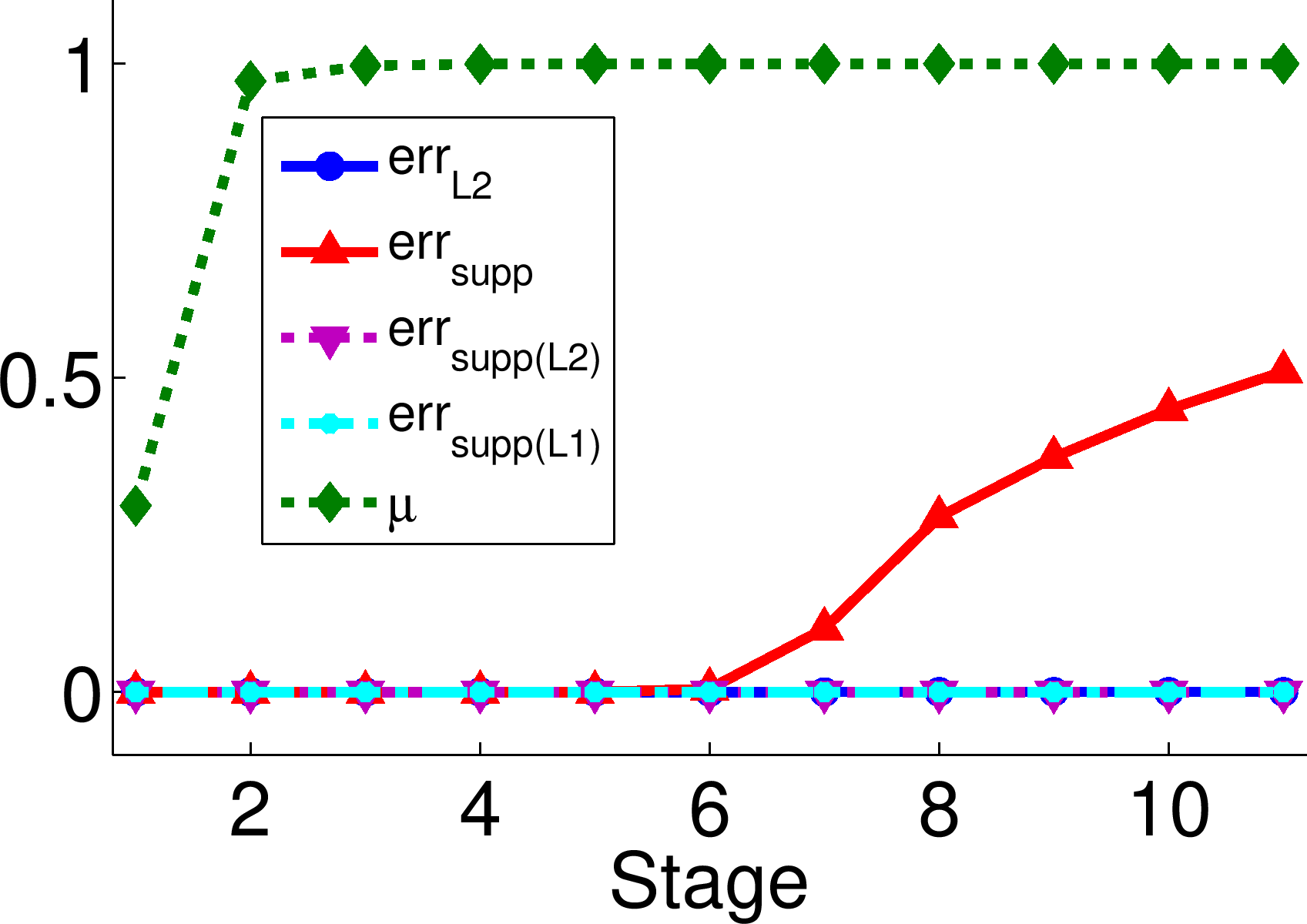}
		\caption{DB-4: $(N_0,m,L) = (5,200,50)$}
	\label{fig:DB_4_HOM_no_noise} \hfill%
\end{subfigure}
\caption[Recovery results on random database model in the case of no noise]{Recovery results on random database model (average of 1000 trials) in the case of no noise.}
\label{fig:HOM_no_noise}
\end{figure}

Considering that $\err_\mathrm{supp}$ records \emph{any} off-support nonzeros, regardless of how small, the results are quite good. In many cases, $\ell^1$-minimization was able to recover the exact solution $\bm{\alpha}_0$ on highly-correlated data, and when errors in the support occurred, they were generally small.

We see two different things happening at either end of the Stage axis. At Stage 1, we see support errors in every database except DB-4 (the low-redundancy case). Further, there are nonzero values of $\err_{\ell^2}$, $\err_{\mathrm{supp}(\ell^2)}$, and $\err_{\mathrm{supp}(\ell^1)}$ for DB-3 (the high-redundancy case) at this stage. At high stages, we see similar small support errors as the data became very correlated; these support errors were numerous (accounting for around half the nonzero coefficients) for both DB-3 and DB-4.

We start by explaining the results at Stage 1. Given the plots in Figure \ref{fig:HOM_no_noise}, our instinct may be to suspect that something wrong happened here, especially considering the exact recovery on all databases at Stage 2. For the cases that we had a ratio of 2-to-1 redundancy, does this contradict the experimental result \cite{don:und} that having $ N_0 = \|\bm{\alpha}_0\|_0 < (3/10)m$ nonzeros guarantees $\ell^1/\ell^0$-equivalence with high probability? It would, but for the fact that this result holds asymptotically. To test this, we repeated the experiments for increasing values of $m$, scaling $N_0$ and $L$ accordingly so that the redundancy remained constant. More precisely, we defined $r_1 := m/N_\mathrm{tr}$ and $r_2 := N_0/L$ and then set $\tilde{L} := [\sqrt{\tilde{m}/(r_1 r_2)}]$ and $\tilde{N}_0 :=r_2 \tilde{L}$. Here, $[\;\cdot\;]$ denotes the nearest integer function and $\tilde{m}$, $\tilde{L}$, and $\tilde{N}_0$ denote the increased values of $m$, $L$, and $N_0$, respectively. As we illustrate in Figure \ref{fig:HOM_no_noise_assympt}, the value of $\err_\mathrm{supp}$ decreased to 0 as $\tilde{m}$ increased. As is to be expected, both the amount of redundancy and the relationship $N_0/m$ affected the speed of convergence. We exclude results for DB-4, as we already see perfect recovery at Stage 1 in Figure \ref{fig:DB_4_HOM_no_noise}. 

\begin{figure}
\centering
\begin{subfigure}[b]{0.45\textwidth}
\centering
	\includegraphics[width=\linewidth]{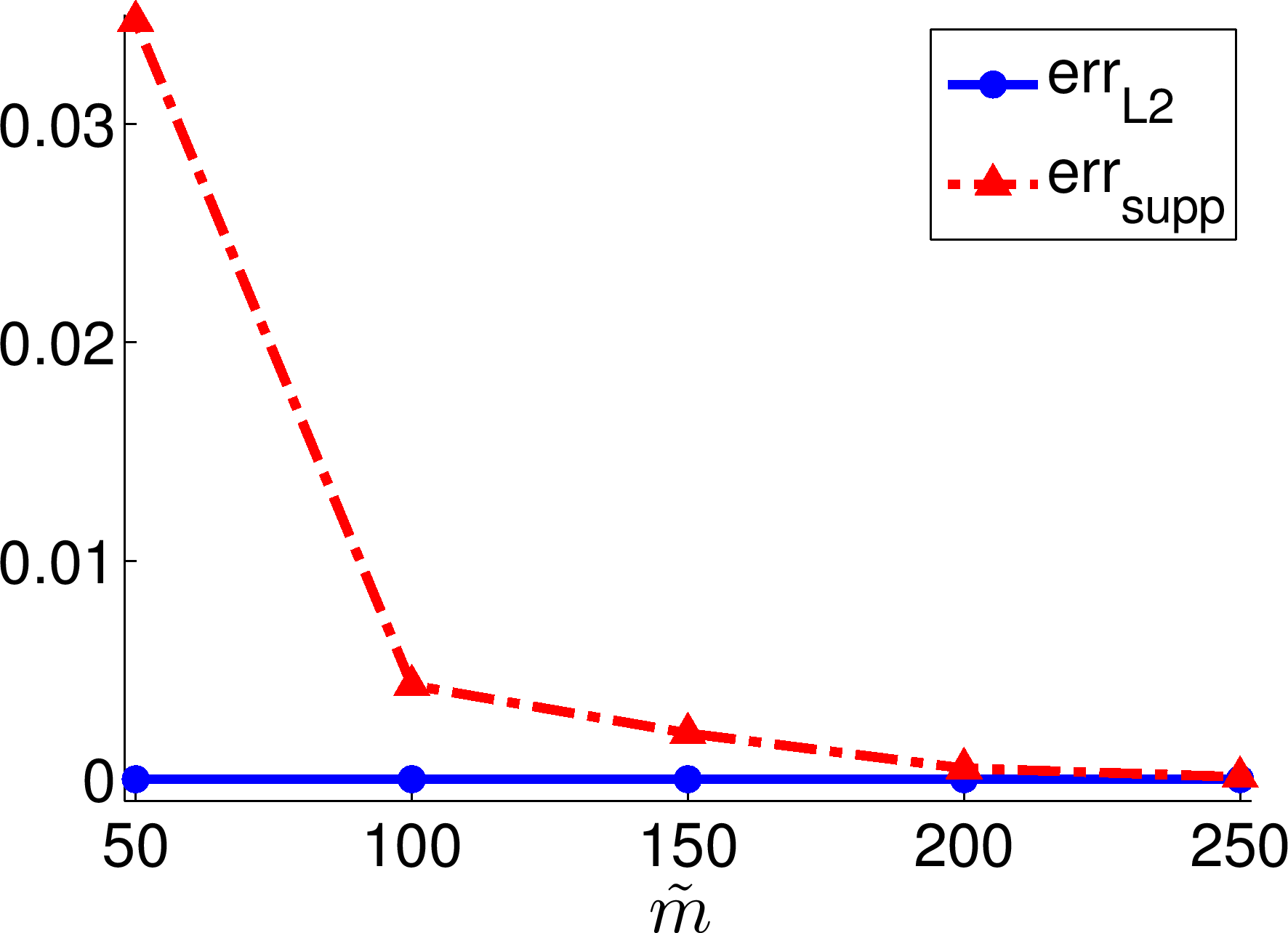}
		\caption{DB-1: $r_1 = 1/2$, $r_2 = 1/4$}
	\label{fig:DB_1_HOM_no_noise_assympt} \hfill%
\end{subfigure}
\begin{subfigure}[b]{0.45\textwidth} 
\centering
	\includegraphics[width=\linewidth]{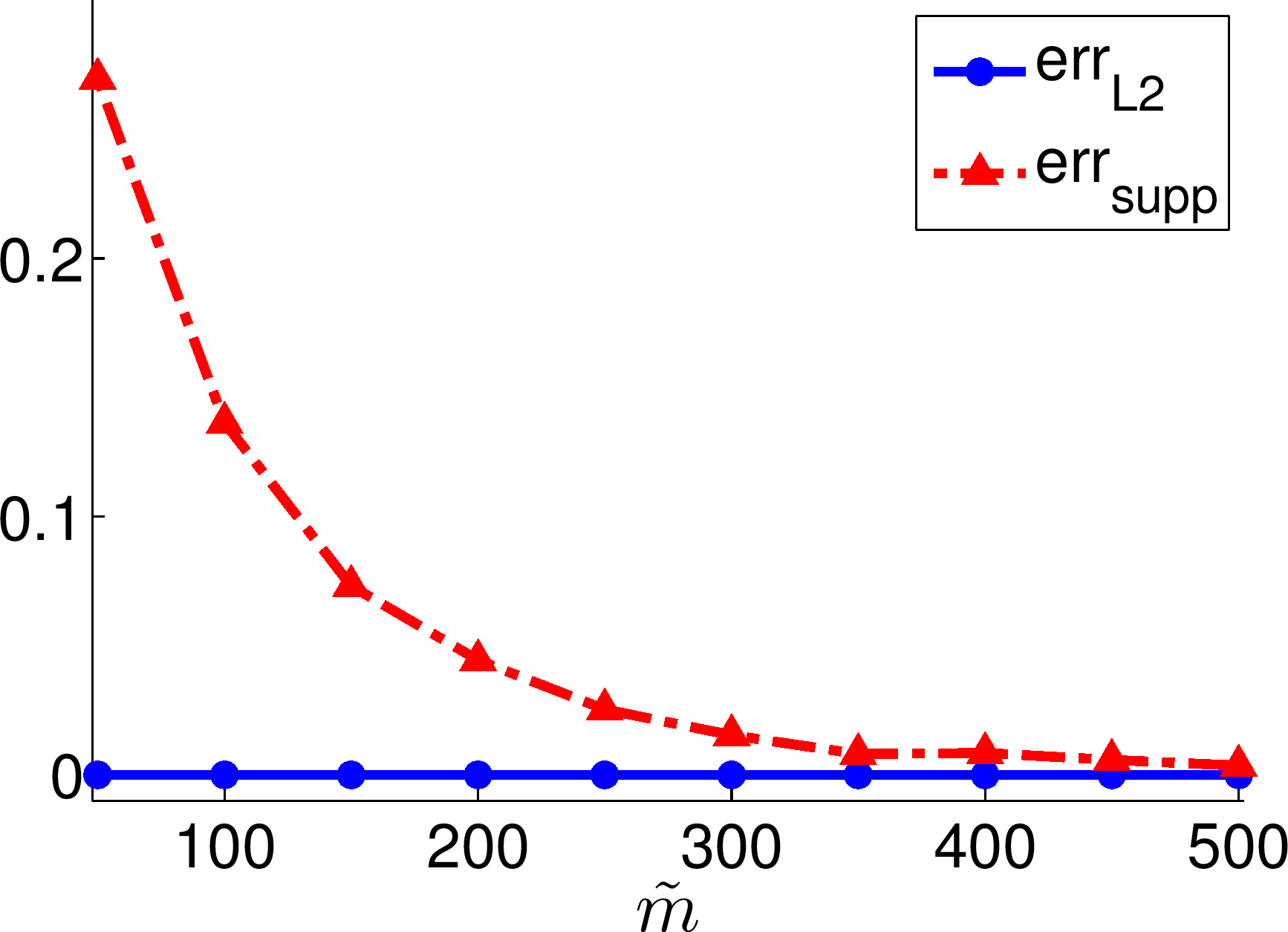}
		\caption{DB-2: $r_1 = 1/2$, $r_2 = 1$}
	\label{fig:DB_2_HOM_no_noise_assympt} \hfill%
\end{subfigure}
\begin{subfigure}[b]{0.45\textwidth} 
\centering
	\includegraphics[width=\linewidth]{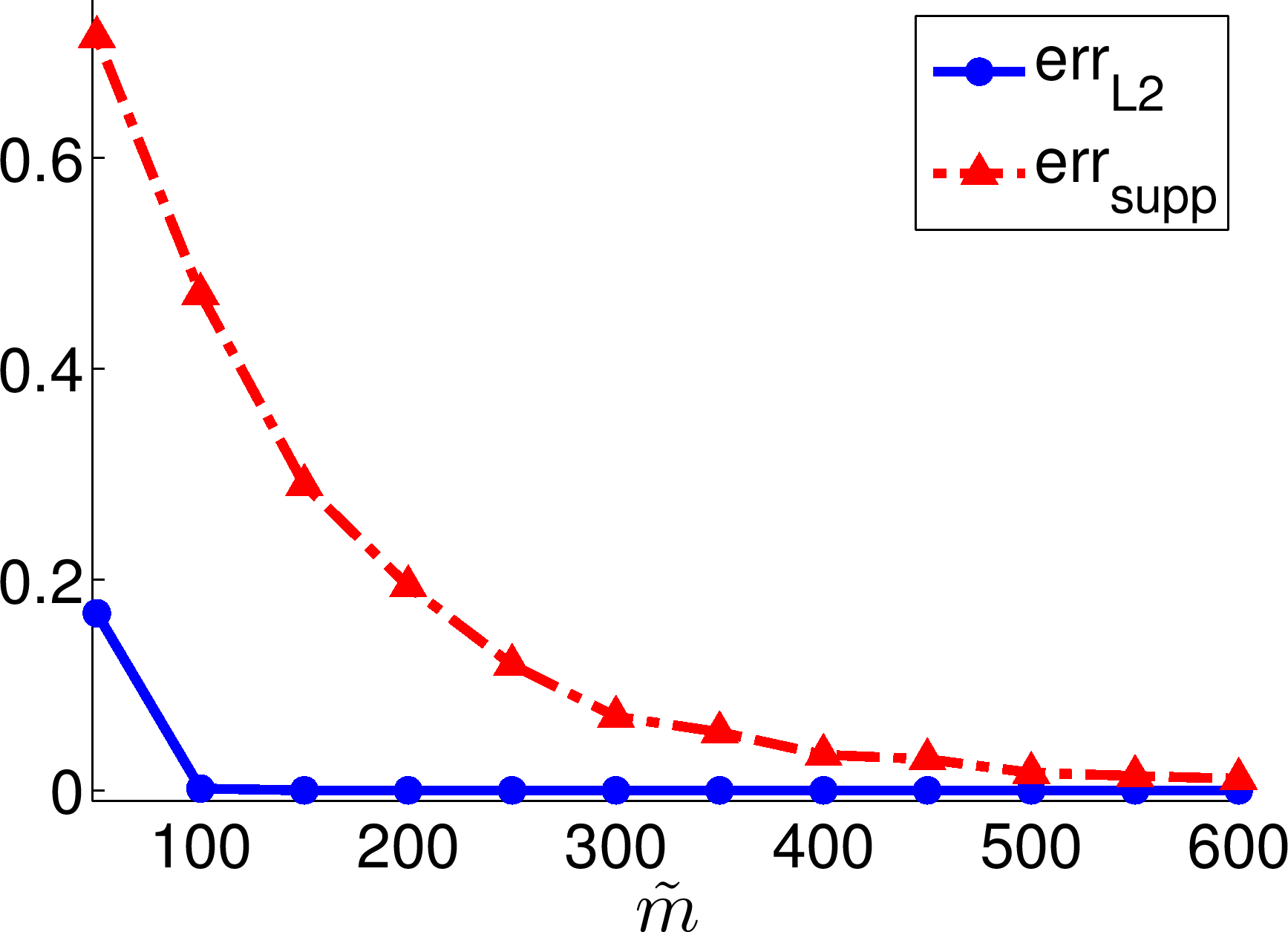}
		\caption{DB-3: $r_1 = 1/10$, $r_2 = 1/5$}
	\label{fig:DB_3_HOM_no_noise_assympt}
	\hfill%
\end{subfigure}
\caption[Asymptotic recovery results on the random database model]{Asymptotic recovery at Stage 1 of the random database model (average of 1000 trials). Note the different scales.}
\label{fig:HOM_no_noise_assympt}
\end{figure}

In comparing the Stage 1 results to those from data with bouquet/cone structure (i.e., Stages 2-11), it is initially surprising that small to moderate levels of correlation in the data samples appear to \emph{improve} sparse recovery. As mentioned, we see near-perfect recovery of $\bm{\alpha}_0$ at Stage 2 for every tried $(N_0,m,L)$ triple; this is in stark contrast to the recovery accuracy at Stage 1, especially for DB-3 (Figure \ref{fig:DB_3_HOM_no_noise}). This sharp change coincides with a significant increase in the within-class correlation between Stages 1 and 2 in our model, whereas the correlation between classes essentially remains unchanged. Though the exact specifics will depend on the $\ell^1$-minimization algorithm used, we strongly suspect that the relative clustering of the samples in the support of $\bm{\alpha}_0$ at Stage 2 (as compared to their random distribution at Stage 1) make it much easier for the algorithm to recover the desired solution. 

Conversely, at high stages, it appears that the loss of class structure negatively affected the recovery of $\bm{\alpha}_0$. As the standard deviation of the class mean distributions grew small, the class cones began to significantly overlap, and $\ell^1$-minimization could not exactly recover the support of $\bm{\alpha}_0$. Notice that we see an especially large number of support errors $\err_{\mathrm{supp}}$ for databases with large values of $L$, namely, DB-3 (Figure \ref{fig:DB_3_HOM_no_noise}) and DB-4 (Figure \ref{fig:DB_4_HOM_no_noise}). For DB-3, the nonzero values of $\err_\mathrm{supp}$ at Stages 5 and 6 (compared to $\err_\mathrm{supp}\approx 0$ at these stages for DB-4) confirms that redundancy, as well as the number of classes, affects recovery.

\textit{Effect on classification:} We earlier discussed the relationship between the support error quantities $\err_{\supp}$, $\err_{\mathrm{supp}({\ell^2})}$, and $\err_{\mathrm{supp}({\ell^1})}$ on the classification performance of SRC, in particular, their effect on the class residuals $\err_l(\bm{y}_0):=\|\bm{y}_0 - X_\mathrm{tr}\delta_l(\bm{\alpha}_1)\|_2$. Here, we consider these residuals explicitly. For each of the four databases, we computed the average residual $\err_l(\bm{y}_0)$ (over 1000 trials) for each class $1\leq l \leq L$ at each of the 11 values of coherence. 

Not surprisingly given the small support error quantities determined in the previous section, there is a stark difference between the residual of class 1 and those of the other classes at all stages. More precisely, the ideal classification scenario occurs in all cases, with $\err_1(\bm{y}_0) \approx 0$ and $\err_l(\bm{y}_0) \approx \|\bm{y}_0\|_2$ for all $2\leq l \leq L$. The approximations are of the order $10^{-8}$ (or better), except for the highly-redundant database DB-3 at Stage 1. In this case, the average quantities were $\err_1(\bm{y}_0) = 0.230$ and 
\begin{align*}
\|\bm{y}_0\|_2 - \underset{2\leq l \leq L}\mean \err_l(\bm{y}_0) = 0.004. 
\end{align*}
These findings are consistent with the results in Figure \ref{fig:DB_3_HOM_no_noise}. Even though these quantities at Stage 1 are nonzero, it is important to note that good classification would still be achieved, as $\min_{2\leq l \leq L} \err_l(\bm{y}_0) = 1.806$, which is much greater than $\err_1(\bm{y}_0) = 0.230$.

\textit{Varying the sparsity level:} We next consider what happens when the sparsity level $\|\bm{\alpha}_0\|_0$ is strictly less than the number of class 1 training samples $N_0$. This is important to investigate: can $\ell^1$-minimization identify the correct training samples from among the rest of the (highly-correlated) training data in that class? For DB-2 and DB-3, we generated $\bm{\alpha}_0$ (and subsequently $\bm{y}_0$) using the first five samples in class 1. Figures \ref{fig:DB_2_HOM_no_noise_k_5} and \ref{fig:DB_3_HOM_no_noise_k_5} show the recovery results, and Figures \ref{fig:DB_2_HOM_no_noise_repeat} and \ref{fig:DB_3_HOM_no_noise_repeat} repeat the plots in Figures \ref{fig:DB_2_HOM_no_noise} and \ref{fig:DB_3_HOM_no_noise} (in which $\|\bm{\alpha}_0\|_0 = N_0$) for convenient comparison.

\begin{figure}
\centering
\begin{subfigure}[b]{0.45\textwidth}
\centering
	\includegraphics[width=\linewidth]{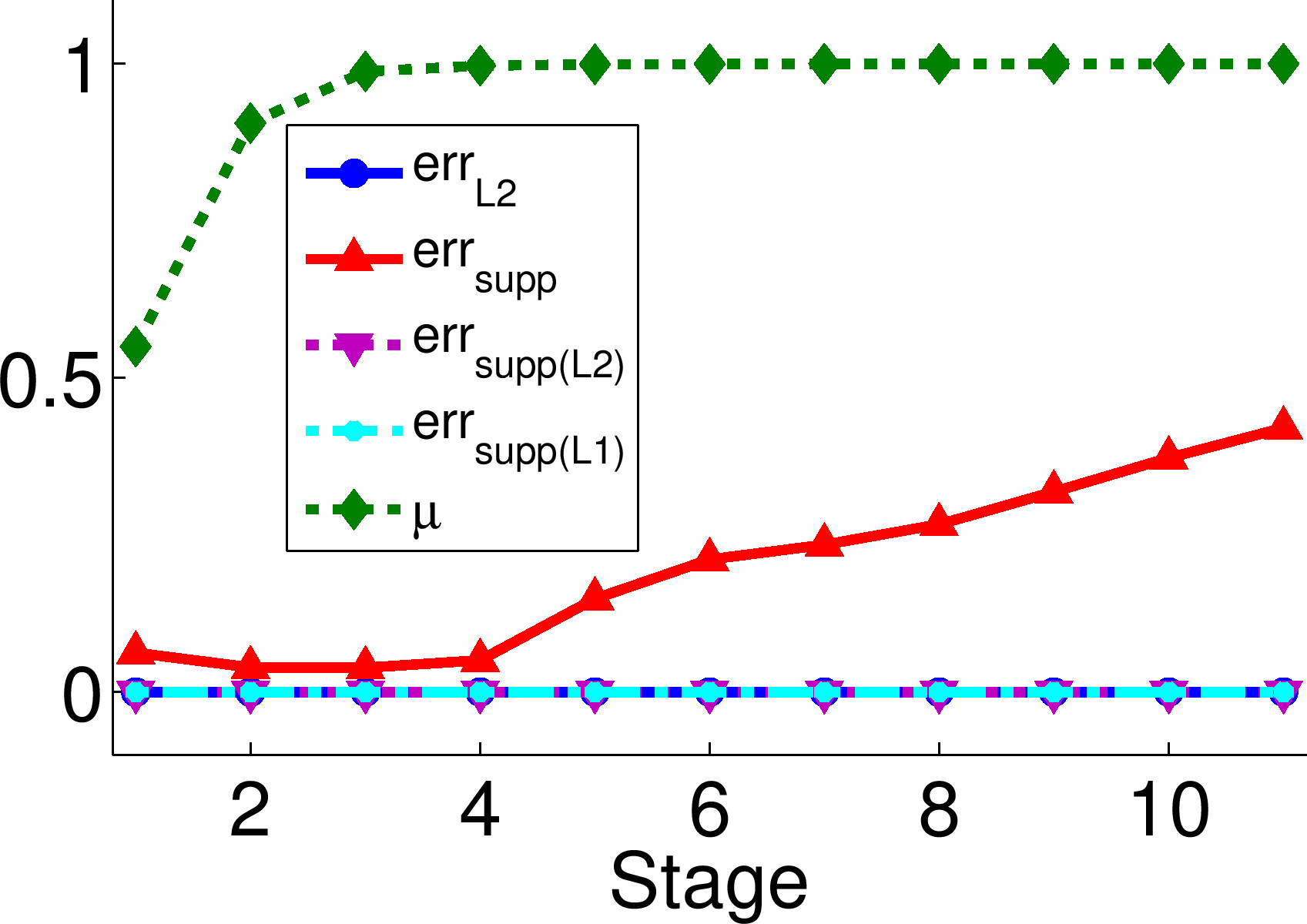}
		\caption{DB-2: $(N_0,m,L) = (10,50,10)$, $\|\bm{\alpha}_0\|_0=5$}
	\label{fig:DB_2_HOM_no_noise_k_5} \hfill%
\end{subfigure}
\begin{subfigure}[b]{0.45\textwidth} 
\centering
	\includegraphics[width=\linewidth]{DB_2_HOM_2_no_noise}
		\caption{DB-2: $(N_0,m,L) = (10,50,10)$, $\|\bm{\alpha}_0\|_0=N_0$}
	\label{fig:DB_2_HOM_no_noise_repeat} \hfill%
\end{subfigure}
\begin{subfigure}[b]{0.45\textwidth} 
\centering
	\includegraphics[width=\linewidth]{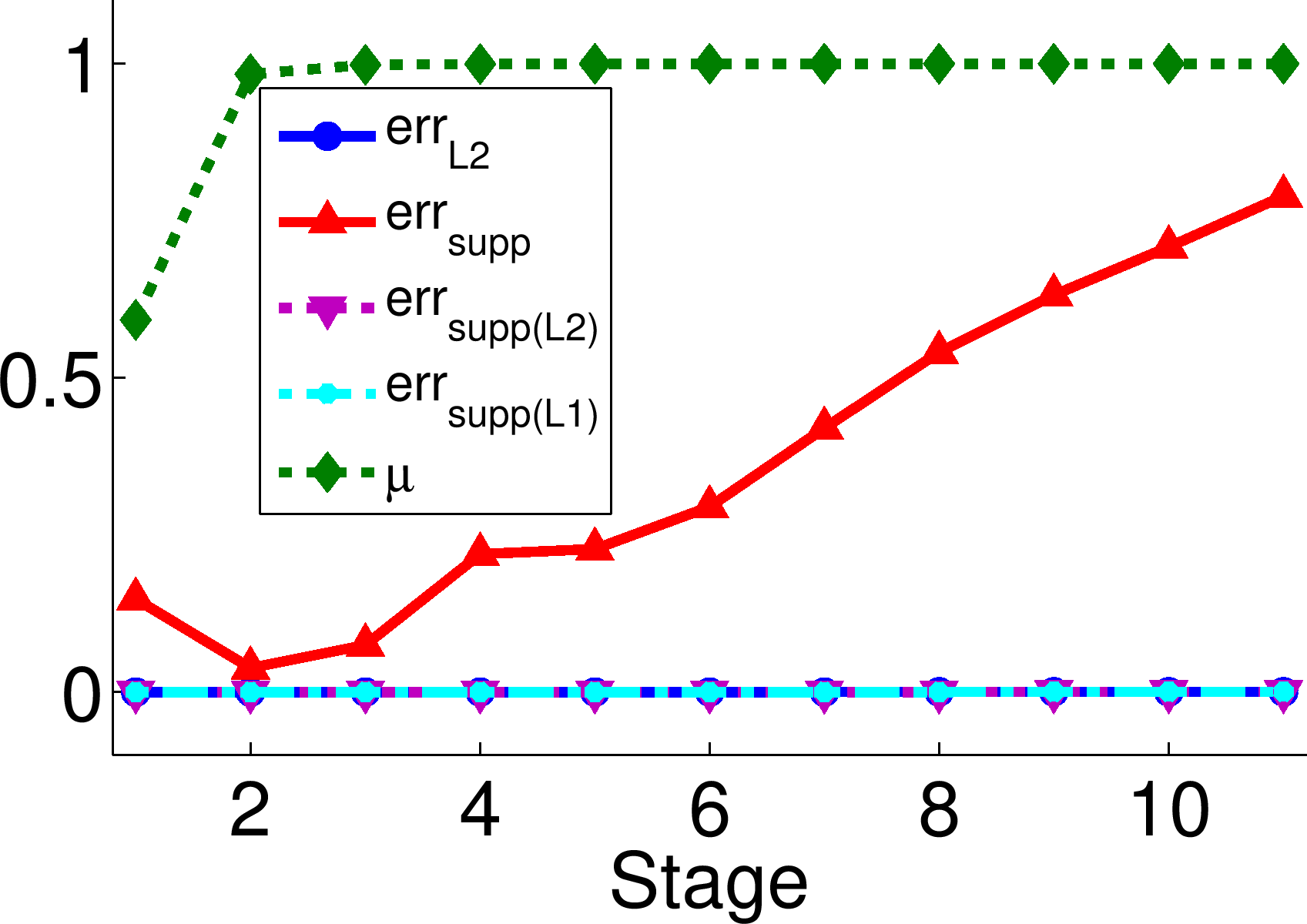}
		\caption{DB-3: $(N_0,m,L) = (10,50,50)$, $\|\bm{\alpha}_0\|_0=5$}
	\label{fig:DB_3_HOM_no_noise_k_5}
	\hfill%
\end{subfigure}
\begin{subfigure}[b]{0.45\textwidth} 
\centering
	\includegraphics[width=\linewidth]{DB_3_HOM_2_no_noise}
		\caption{DB-3: $(N_0,m,L) = (10,50,50)$, $\|\bm{\alpha}_0\|_0=N_0$}
	\label{fig:DB_3_HOM_no_noise_repeat}
	\hfill%
\end{subfigure}
\caption[Recovery results on the random database model as the sparsity level is varied]{Comparing $\|\bm{\alpha}_0\|_0 < N_0$ and $ \|\bm{\alpha}_0\|_0 = N_0$ sparsity levels (average of 1000 trials) on the random database model in the case of no noise.}
\label{fig:HOM_no_noise_k}
\end{figure}

At Stage 1, we see that the support of $\bm{\alpha}_1$ was more concentrated on the correct training samples when $\|\bm{\alpha}_0\|_0$ was smaller, evidenced by smaller values of $\err_\mathrm{supp}$. This is to be expected, as the ground truth solution became sparser. For the lower-redundancy case DB-2, we see far more support errors as the correlation increased when $\|\bm{\alpha}_0\|_0=5$ (Figure \ref{fig:DB_2_HOM_no_noise_k_5}) than for the case $\|\bm{\alpha}_0\|_0=N_0$ (Figure \ref{fig:DB_2_HOM_no_noise_repeat}); however, the values of these off-support coefficients were very small, as demonstrated by the near-zero values of $\err_{\mathrm{supp}(\ell^2)}$ and $\err_{\mathrm{supp}(\ell^1)}$. Though class 1 training samples not in the support of $\bm{\alpha}_0$ were mistakenly selected as the data in class 1 became more correlated, these samples played a negligible role in the representation. For the high-redundancy case DB-3, we similarly see more small-valued, off-support coefficients at Stages 2-5 when $\|\bm{\alpha}_0\|_0=5$ (Figure \ref{fig:DB_3_HOM_no_noise_k_5}) than in the case $\|\bm{\alpha}_0\|_0=N_0$ (Figure \ref{fig:DB_3_HOM_no_noise_repeat}). The value of $\err_\mathrm{supp}$ for $\|\bm{\alpha}_0\|_0<N_0$ was actually smaller than it was for $\|\bm{\alpha}_0\|_0=N_0$ at many of the higher stages, however, suggesting that the added degree of sparsity helped to counter-balance the high redundancy of this database (and its negative effect on recovery) in these cases.

\textit{Eliminating errors by thresholding:} Before we turn to the noisy setting, we demonstrate that the small support errors in $\bm{\alpha}_1$ depicted in Figure \ref{fig:HOM_no_noise} can be completely remedied using thresholding in all but the high-redundancy case DB-3. After determining $\bm{\alpha}_1$ as before, we set its small coefficients (those with absolute value less than some threshold $\tau$) to zero, obtaining the vector $\bm{\alpha}_1^{\tau}$. We then re-solved the equation $X_\mathrm{tr}\bm{\alpha} = \bm{y}_0$ with the constraint that the solution, denoted $\hat{\bm{\alpha}}_1$, had the same support as the thresholded $\bm{\alpha}_1^{\tau}$. For simplicity, we did this by setting the columns of $X_\mathrm{tr}$ corresponding to zero-coordinates in $\bm{\alpha}_1^{\tau}$ to $\bm{0}$, thus obtaining the matrix $\hat{X}_\mathrm{tr}$. We then used MATLAB's ``$\backslash$'' operator to define $\hat{\bm{\alpha}}_1 := \hat{X}_\mathrm{tr}\backslash \bm{y}_0$. In our case, since $X_\mathrm{tr}$ was not square, the desired least squares solution was found by (MATLAB's implementation of) QR-factorization. 

For all but the highly-redundant database DB-3, $\hat{\bm{\alpha}}_1$ was equal to the sparsest solution $\bm{\alpha}_0$ (up to nearly machine-precision) for the thresholding value $\tau = 10^{-5}$. For $\tau \in \{0.001,0.01\}$ on these three databases (DB-1, DB-2, and DB-3), we saw small nonzero values of $\err_{\ell^2}$, but these errors were indiscernible in plots on the same scale as those in Figure \ref{fig:HOM_no_noise}, and so we do not show them here. For $\tau = 0.1$, there was a consistent, small but nontrivial $\ell^2$-error across all stages, as small coefficients corresponding to class 1 training samples were incorrectly set to 0. For all four values of $\tau$, there were no support errors.

For the high-redundancy case DB-3, we continued to see errors at Stage 1, similar to those in Figure \ref{fig:DB_3_HOM_no_noise}. For the thresholding values $\tau \in \{0.001,0.01,0.1\}$ (i.e., for $\tau$ large enough), there were no support errors at other stages. However, similarly to the other databases, we saw nontrivial $\ell^2$-error when $\tau = 0.1$. We plot the results for DB-3 in Figure \ref{fig:th}, stressing that the results for the other databases contained errors too small to produce nontrivial plots.

\begin{figure}
\centering
\begin{subfigure}[b]{0.45\textwidth}
\centering
	\includegraphics[width=\linewidth]{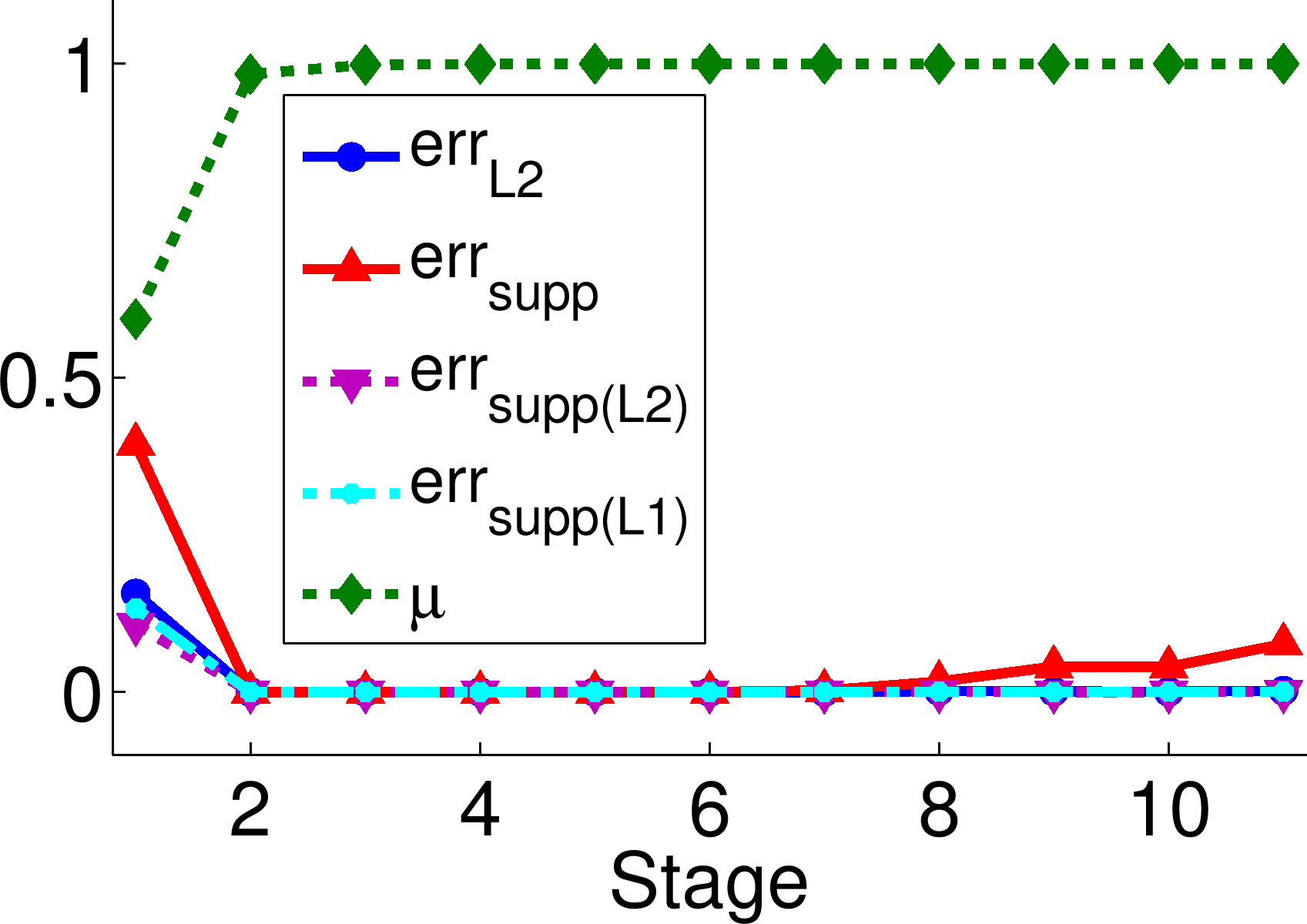}
		\caption{$\tau = 10^{-5}$}
	\label{fig:DB_3_HOM_no_noise_th_1e_5} \hfill%
\end{subfigure}
\begin{subfigure}[b]{0.45\textwidth} 
\centering
	\includegraphics[width=\linewidth]{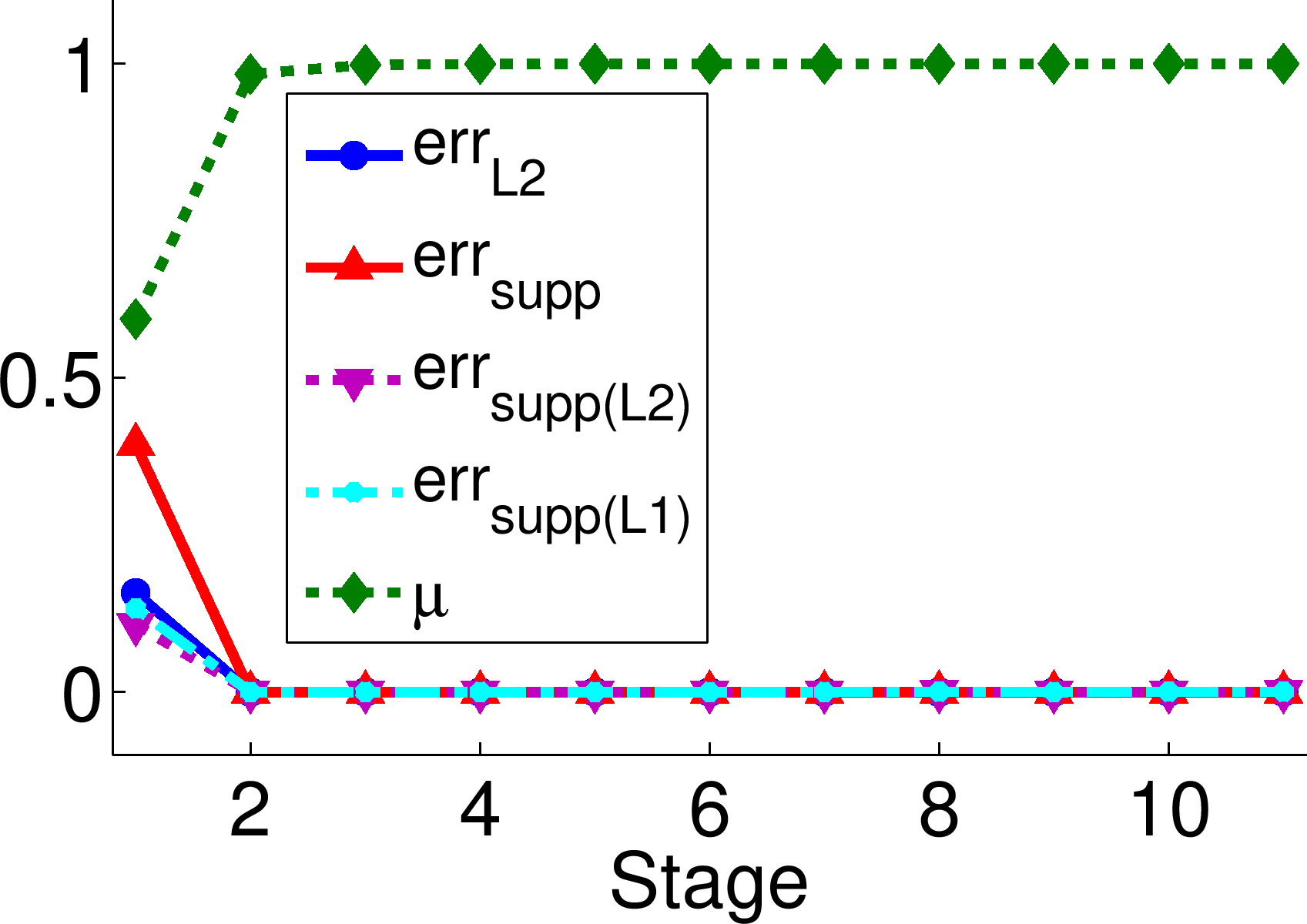}
		\caption{$\tau = 0.001$}
	\label{fig:DB_3_HOM_no_noise_th_1e_3} \hfill%
\end{subfigure}
\begin{subfigure}[b]{0.45\textwidth} 
\centering
	\includegraphics[width=\linewidth]{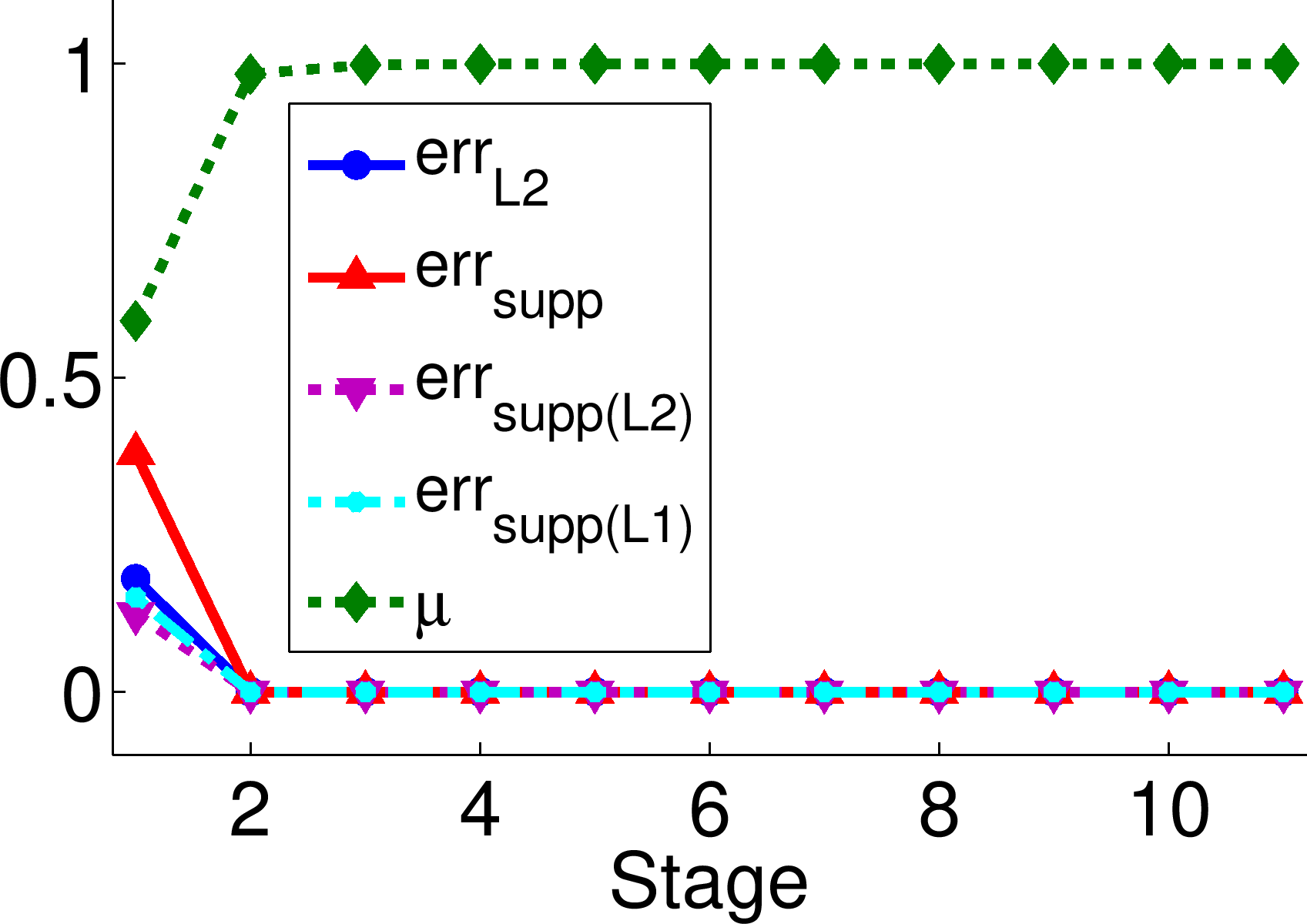}
		\caption{$\tau = 0.01$}
	\label{fig:DB_3_HOM_no_noise_th_1e_2}
	\hfill%
\end{subfigure}
\begin{subfigure}[b]{0.45\textwidth} 
\centering
	\includegraphics[width=\linewidth]{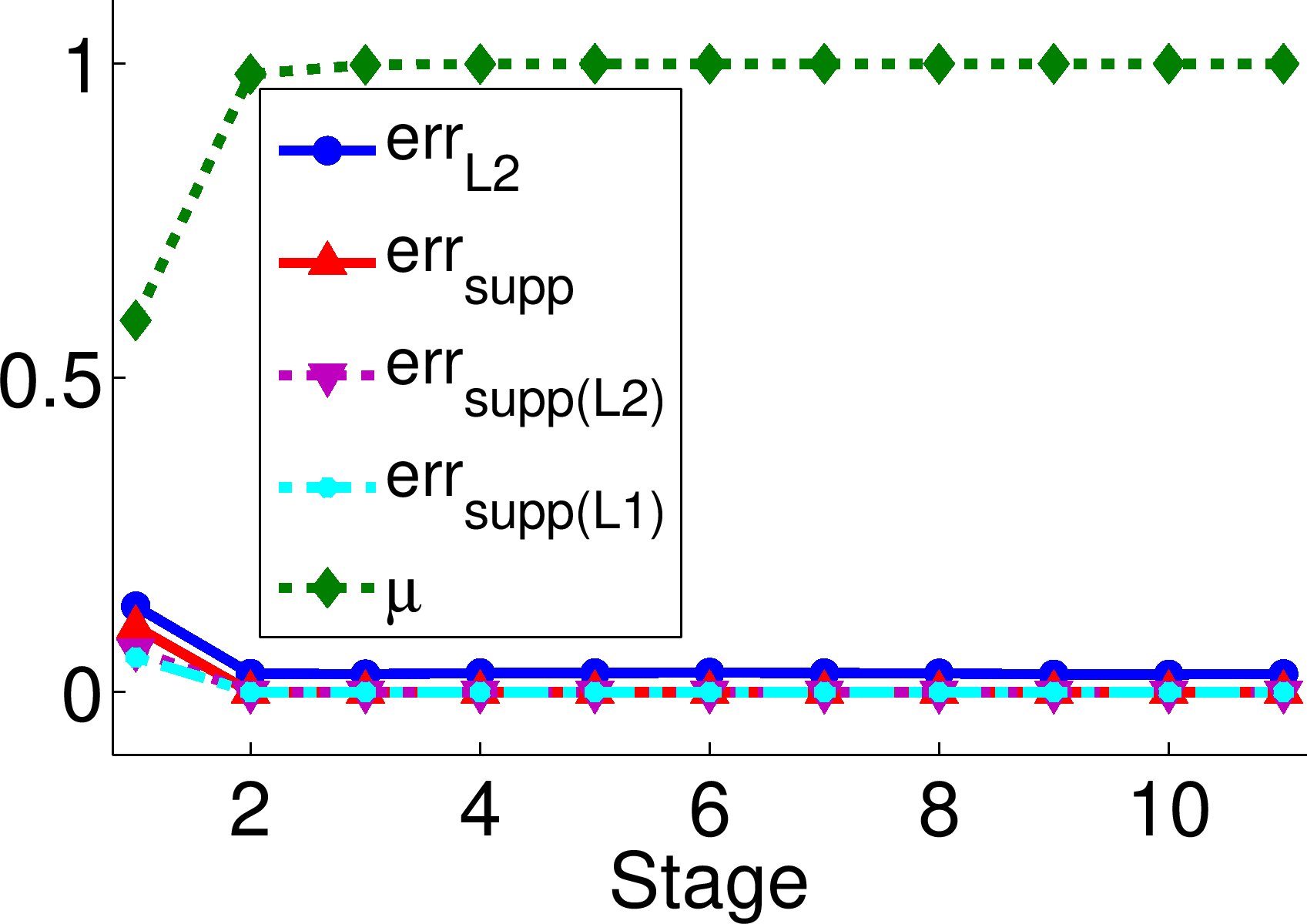}
		\caption{$\tau = 0.1$}
	\label{fig:DB_3_HOM_no_noise_th_1e_1}
	\hfill%
\end{subfigure}
\caption[Recovery results on the random database model in the case of thresholding]{The results of thresholding (average of 1000 trials) on the highly-redundant database DB-3: $(N_0,m,L) = (10,50,50)$ in the case of no noise.}
\label{fig:th}
\end{figure}

\subsubsection{Experimental Results: Noisy Setting}

In these experiments, we examine $\ell^1/\ell^0$-recovery when noise is added to the test sample $\bm{y}_0$. Recall the theorems by Donoho et al.\ regarding $\ell^1/\ell^0$-equivalence in the noisy setting stated in Theorems \ref{thm:mc_noise} and \ref{thm:mc_stab}.

\textit{Accuracy of recovery:} Unfortunately, Eq.~\eqref{eq:mc_noise} and Eq.~\eqref{eq:mc_stab} in the referenced theorems do not make sense for large mutual coherence $\mu(X)$. However, we can still look for a correlation between $\|\bm{\alpha}_0-\bm{\alpha}_{1,\epsilon}\|_2$ (where $\boldsymbol{\alpha}_{1,\epsilon}$ is the solution to Eq.~\eqref{eq:cs_l1_error_2} below) and the values of the noise tolerance $\zeta$ (see that statement of Theorem \ref{thm:mc_noise}), the approximation error bound $\epsilon$, $N_0 = \|\bm{\alpha}_0\|_0 := k$, and $\mu(X_\mathrm{tr})$, with $\epsilon = :C\zeta$ for some constant $C>0$. We modify the experiments in Section \ref{sec:exps_no_noise} as follows: First, we specify the noise tolerance $\zeta$ and the constant $C$. After generating the training data and the (noise-free) test sample $\bm{y}_0$, we set $\bm{y} := \bm{y}_0 + \bm{z}$, where the entries of $\bm{z}$ are drawn from $\mathcal{N}(0,\zeta/(2\sqrt{m}))$. Then $\|\bm{z}\|_2 \leq \zeta$ with probability at least $95\%$. From here, we set $\epsilon := C\zeta$ and find
\begin{equation} \label{eq:cs_l1_error_2}
\bm{\alpha}_{1,\epsilon} := \arg \min_{\alpha} \|\bm{\alpha}\|_1 \text{ subject to } \|\bm{y} - X_\mathrm{tr}\bm{\alpha}\|_2 \leq \epsilon.
\end{equation}
We set $\zeta = 0.01$, and we used two values of $C$: $C=5$, and $C=10$, producing the $(\zeta,\epsilon)$-pairs $(0.01,0.05)$ and $(0.01,0.1)$. In order to ensure that the reconstruction error $\|X_\mathrm{tr}\bm{\alpha}_{1,\epsilon} - \bm{y}\|_2$ was less than $\epsilon$, we used the basis pursuit denoising version of the $\ell^1$-minimization algorithm SPGL1 \cite{fried:par,fried:spgl1}. 

In Figure \ref{fig:spgl1_noise}, we plot the normalized $\ell^2$-error, the fraction of off-support nonzeros, the normalized $\ell^2$ and $\ell^1$-norms of the off-class support vectors, and the mutual coherence $\mu(X_\mathrm{tr})=:\mu$. Note that we modify the corresponding definitions given in Section \ref{fig:spgl1_noise} (for $\err_{\ell^2}$, $\err_\mathrm{supp}$, $\err_{\mathrm{supp}(\ell^2)}$ and $\err_{\mathrm{supp}(\ell^1)}$) to use $\bm{\alpha}_{1,\epsilon}$ instead of $\bm{\alpha}_1$ and do not change the notation. We report the averages over 1000 trials at each stage. 

As we can see, there is clearly a relationship between $\err_{\ell^2}$ and the amount of correlation in the data. As the data became increasingly bouquet-shaped, both within each class and as a dataset as a whole, the normalized $\ell^2$-distance between $\bm{\alpha}_{1,\epsilon}$ and $\bm{\alpha}_0$ increased. The rate of increase of this error appears to be related the redundancy of the database. It is evident that mutual coherence was not a good indicator of $\err_{\ell^2}$, as the plots show that $\err_{\ell^2}$ could be relatively low even after $\mu(X_\mathrm{tr})$ had reached its maximum value. 

Perhaps more importantly, the supports of the solution vectors $\bm{\alpha}_{1,\epsilon}$ and $\bm{\alpha}_0$ were nearly identical at stages greater than 1. This means that the vast majority of nonzeros in $\bm{\alpha}_0$ occurred at positions corresponding to class 1 training samples. To fix the small support errors, we could use the thresholding technique discussed in the previous section, choosing $\tau$ by trial-and-error. This method could also be used to ameliorate the numerous support errors for the databases DB-2 and DB-3 at Stage 1. In this case, we found that $\tau = 0.01$ greatly reduced the Stage 1 support errors but did not eliminate them completely. 

Lastly, we consider the differences between setting $C=5$ and $C=10$. For the most part, the plots are quite similar. We see that setting $C=5$ produced slightly better recovery than $C=10$ at Stage 1, but in general, the normalized $\ell^2$-error $\err_{\ell^2}$ was the same for the two settings at higher stages. This is informative, as it tells us that $\ell^1/\ell^0$-recovery on this kind of highly-correlated data is potentially quite robust to the setting of $C$ in the approximation error tolerance $\epsilon = C\zeta$. Once again, we attribute this to the class structure of the data making it easier for the $\ell^1$-minimization algorithm to find the class solution $\bm{\alpha}_0$. 

\begin{figure}
\centering
\begin{subfigure}[b]{0.35\textwidth}
\centering
	\includegraphics[width=\linewidth]{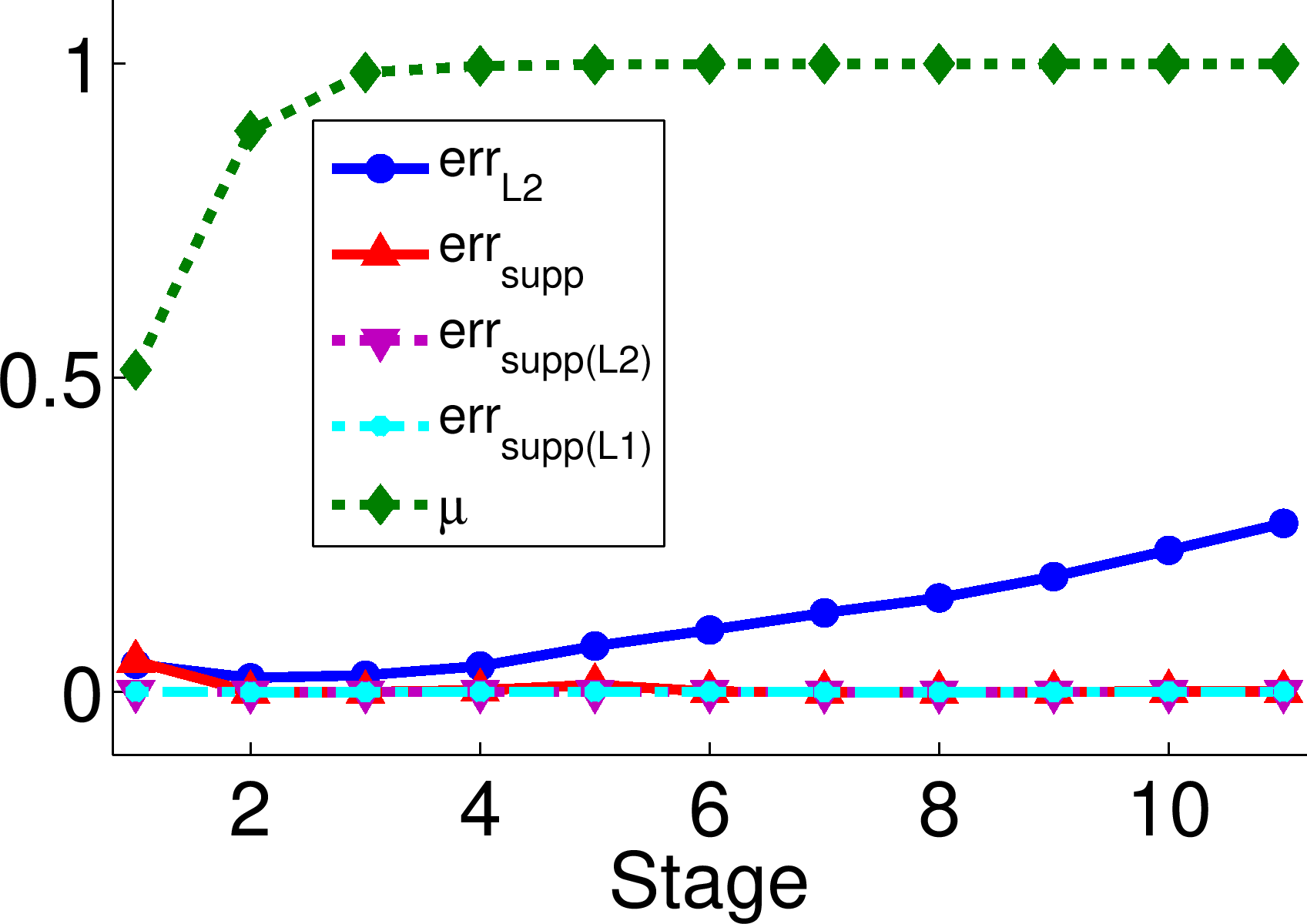}
		\caption{DB-1, $C=5$}
	\label{fig:DB_1_spgl1_noise_C_5} \hfill%
\end{subfigure}
\begin{subfigure}[b]{0.35\textwidth}
\centering
	\includegraphics[width=\linewidth]{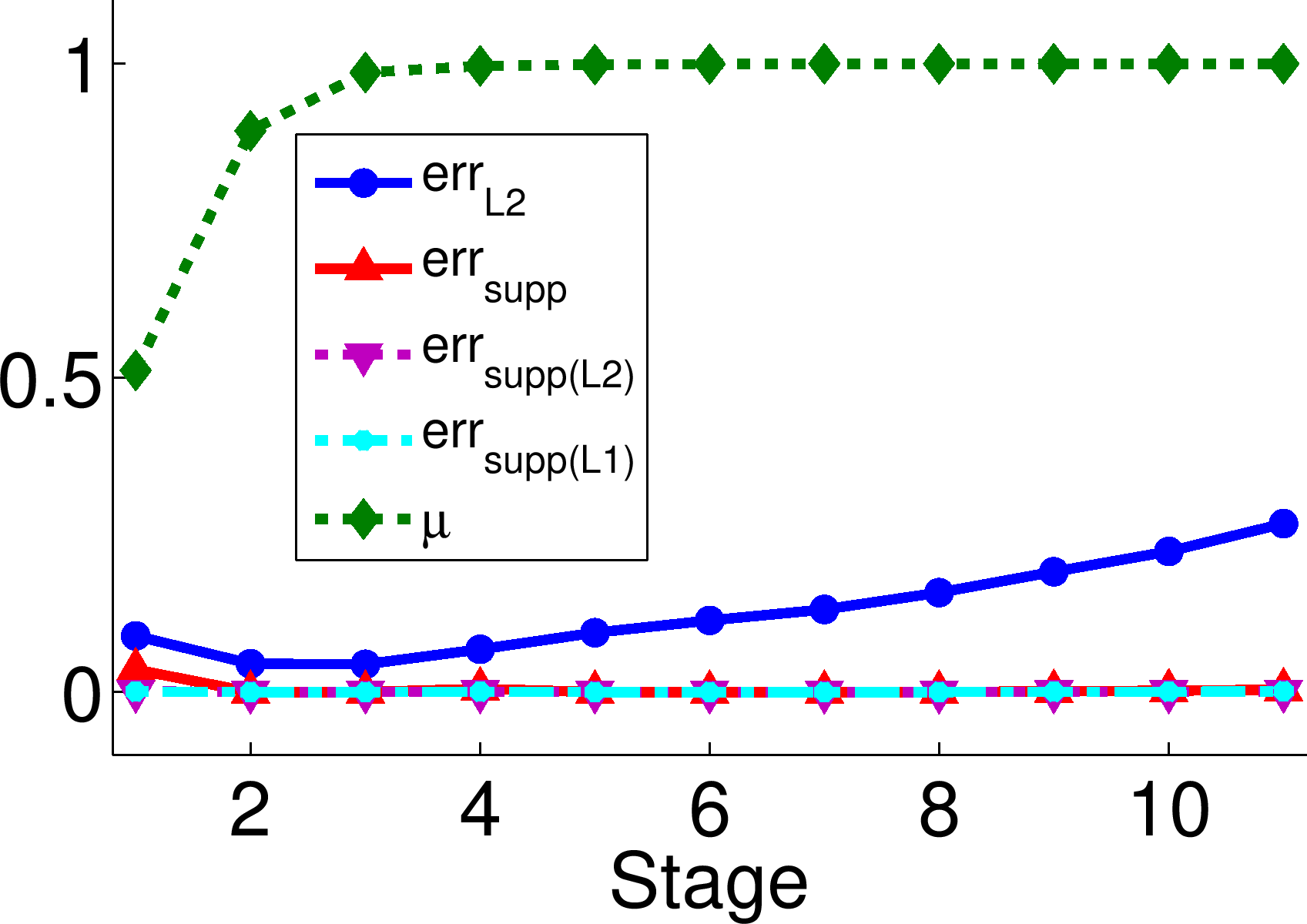}
		\caption{DB-1, $C=10$}
	\label{fig:DB_1_spgl1_noise_C_10} \hfill%
\end{subfigure}
\begin{subfigure}[b]{0.35\textwidth} 
\centering
	\includegraphics[width=\linewidth]{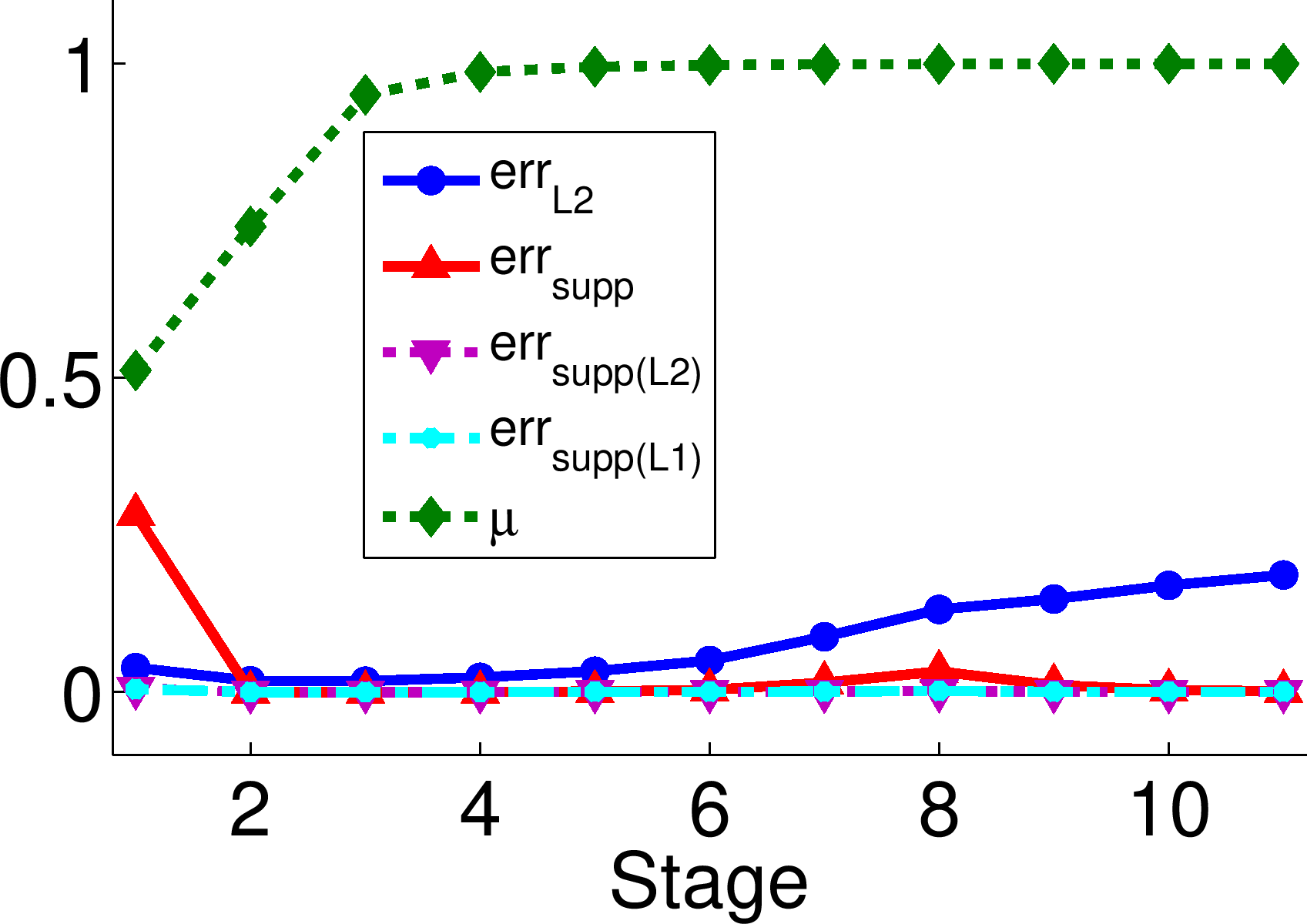}
		\caption{DB-2, $C=5$}
	\label{fig:DB_2_spgl1_noise_C_5}
	\hfill%
\end{subfigure}
\begin{subfigure}[b]{0.35\textwidth} 
\centering
	\includegraphics[width=\linewidth]{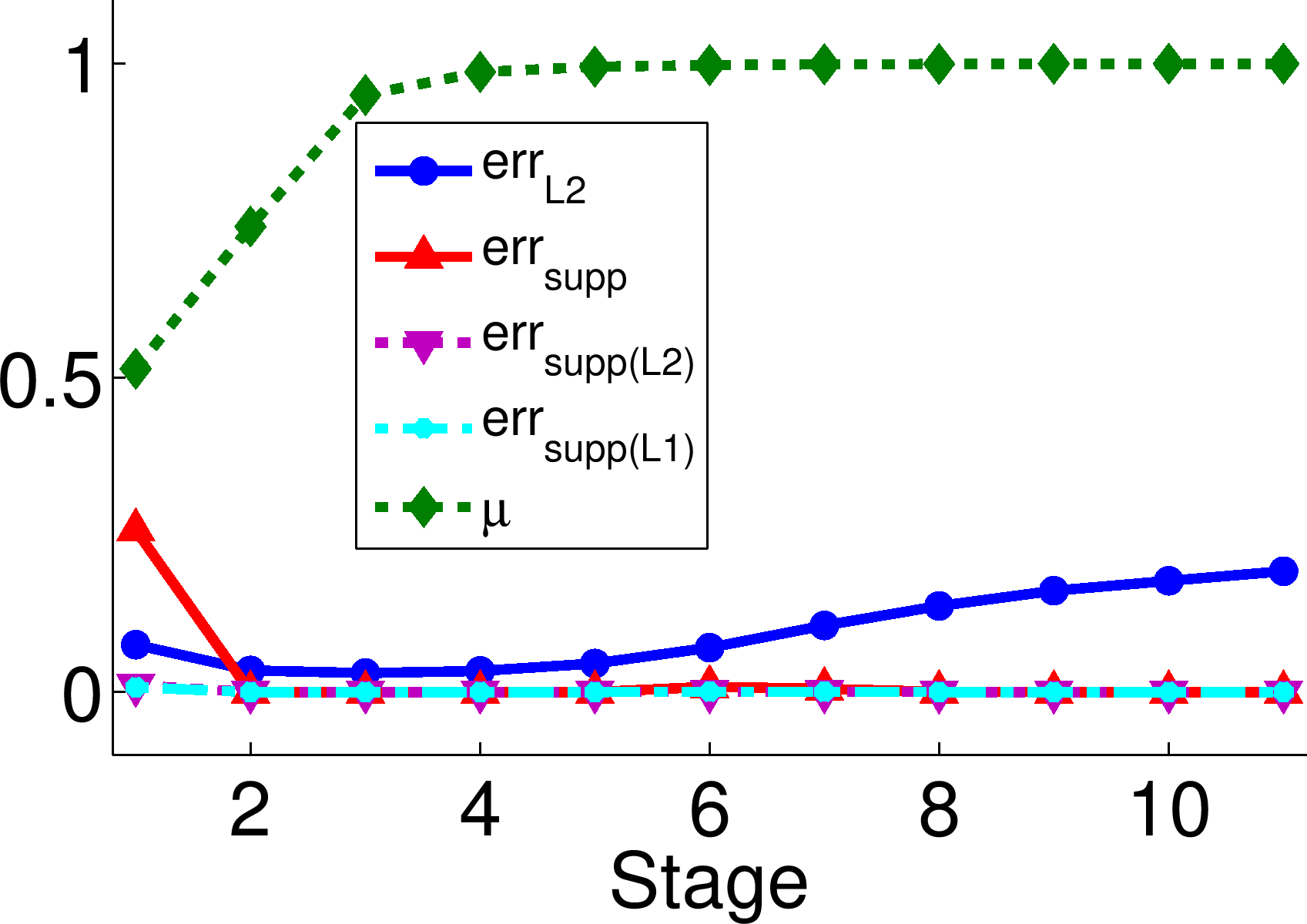}
		\caption{DB-2, $C=10$}
	\label{fig:DB_2_spgl1_noise_C_10}
	\hfill%
\end{subfigure}
\begin{subfigure}[b]{0.35\textwidth}
\centering
	\includegraphics[width=\linewidth]{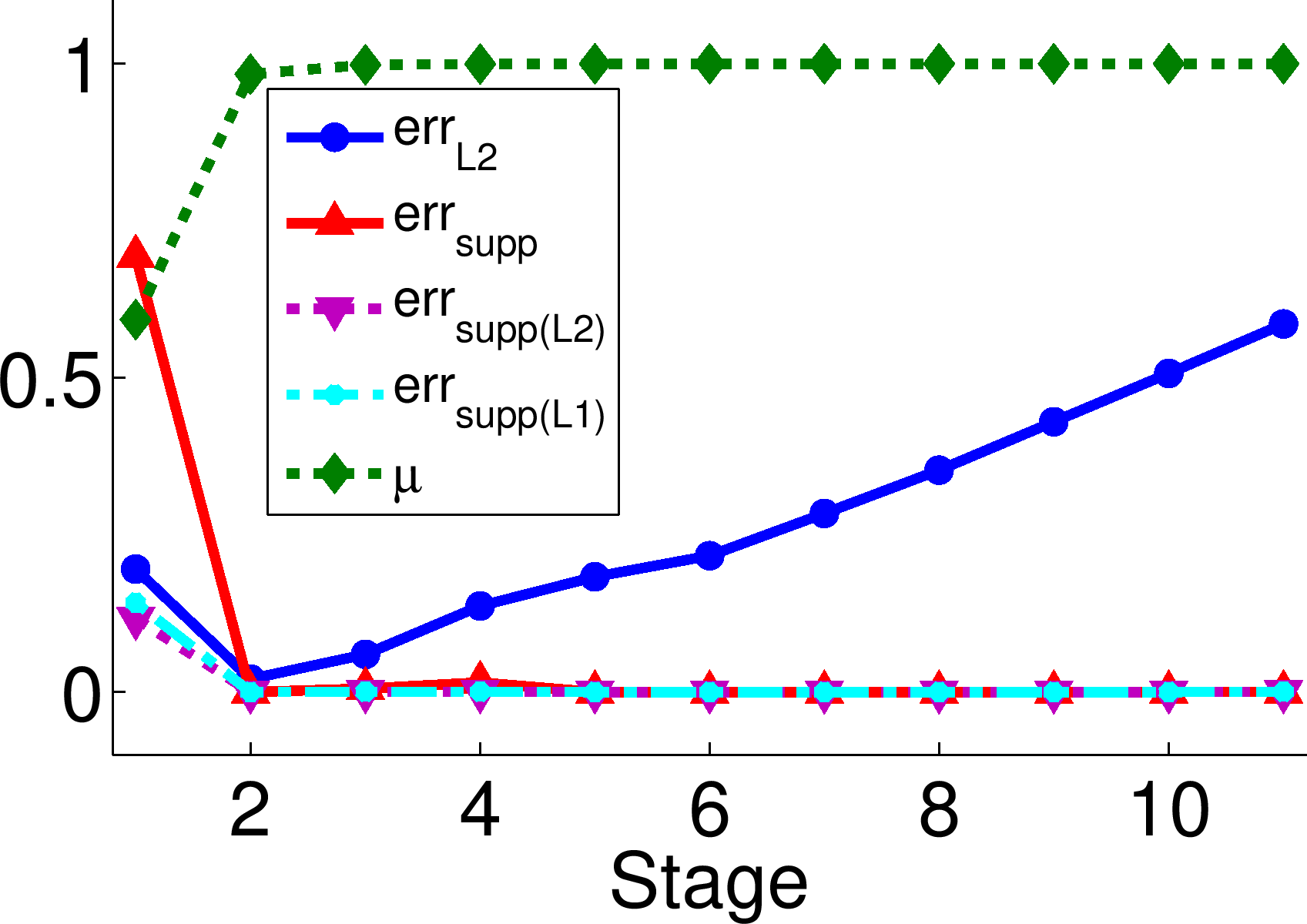}
		\caption{DB-3, $C=5$}
	\label{fig:DB_3_spgl1_noise_C_5} \hfill%
\end{subfigure}
\begin{subfigure}[b]{0.35\textwidth}
\centering
	\includegraphics[width=\linewidth]{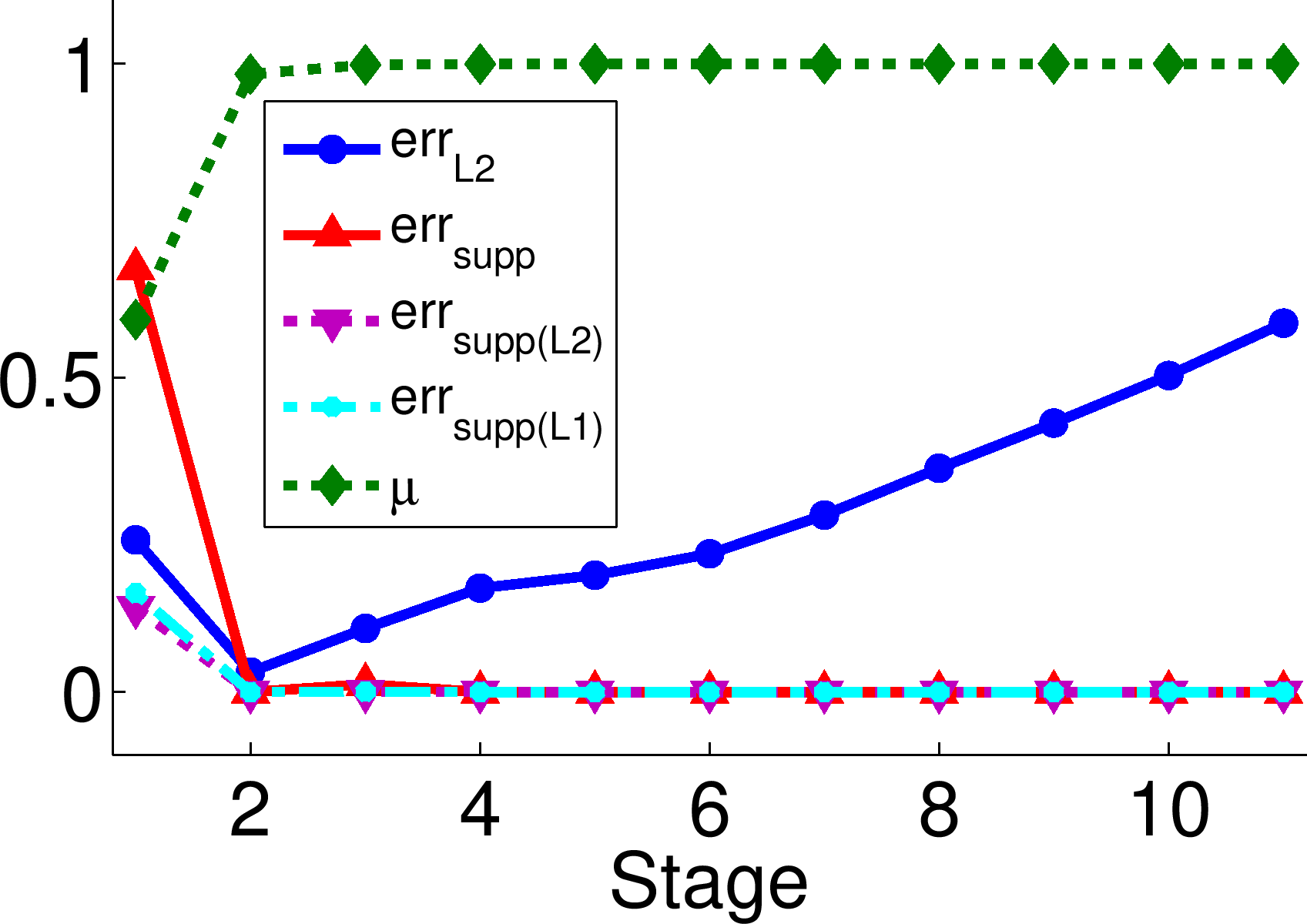}
		\caption{DB-3, $C=10$}
	\label{fig:DB_3_spgl1_noise_C_10} \hfill%
\end{subfigure}
\begin{subfigure}[b]{0.35\textwidth} 
\centering
	\includegraphics[width=\linewidth]{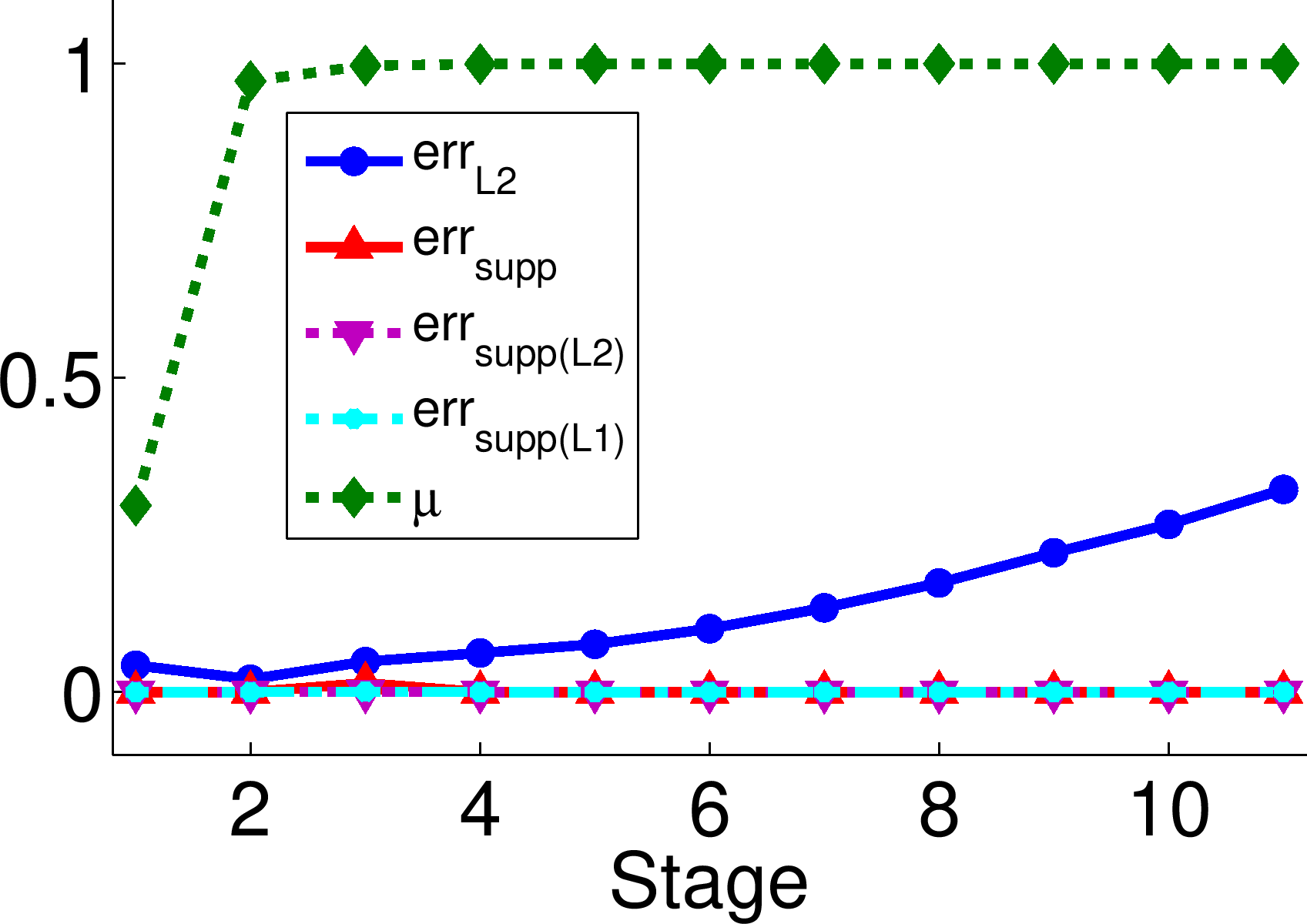}
		\caption{DB-4, $C=5$}
	\label{fig:DB_4_spgl1_noise_C_5}
	\hfill%
\end{subfigure}
\begin{subfigure}[b]{0.35\textwidth} 
\centering
	\includegraphics[width=\linewidth]{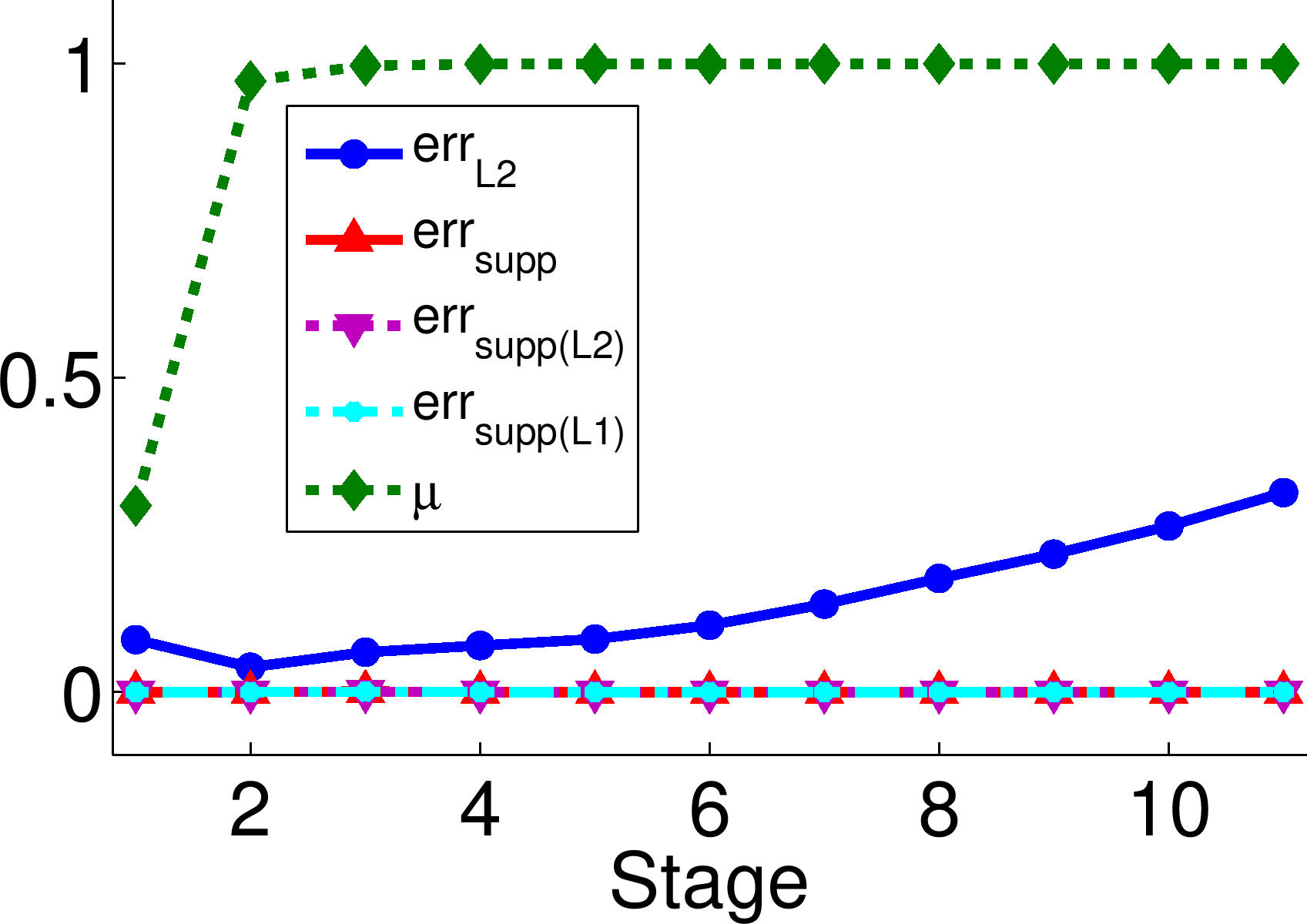}
		\caption{DB-4, $C=10$}
	\label{fig:DB_4_spgl1_noise_C_10}
	\hfill%
\end{subfigure}
\caption[Recovery results on the random database model in the case of noise]{Recovery results on the random database model in the case of noise.}
\label{fig:spgl1_noise}
\end{figure}

\textit{Effect on classification:} As in the noise-free scenario, we compute the class residuals $\err_l(\bm{y}):=\|\bm{y} - X_\mathrm{tr}\delta_l(\bm{\alpha}_{1,\epsilon})\|_2$ for each of the four databases at each of the 11 values of coherence. Specifically, we are interested in how close the class 1 residual is to 0 (signifying perfect reconstruction of $\bm{y}$ using class 1) and how close the next smallest class residual $\min_{2\leq l \leq L} \err_l(\bm{y})$ is to this value. If it is close, then it means that we should have less confidence in the SRC classification assignment than if these quantities were far apart, i.e., that SRC distinguishes the correct class less clearly.

The average relevant class residuals (over 1000 trials) are displayed in Table \ref{tab:class_resids_noise}. Since the results for $C=5$ and $C=10$ were very similar, we only include the results for $C=5$. 

\begin{table*}[!htb] 
\centering
\resizebox{\columnwidth}{!}{%
\begin{tabular}{|c|c|c|c|c|c|c|c|c|}
\hline
& \multicolumn{2}{|c|}{DB-1} & \multicolumn{2}{|c|}{DB-2} & \multicolumn{2}{|c|}{DB-3} & \multicolumn{2}{|c|}{DB-4} \\
\hline
Stage & $\err_1(\bm{y})$ & $\displaystyle\min_{2\leq l \leq L} \err_l(\bm{y})$ & $\err_1(\bm{y})$ & $\displaystyle\min_{2\leq l \leq L} \err_l(\bm{y})$ & $\err_1(\bm{y})$ & $\displaystyle\min_{2\leq l \leq L} \err_l(\bm{y})$ & $\err_1(\bm{y})$ & $\displaystyle\min_{2\leq l \leq L} \err_l(\bm{y})$ \\
\hline
1	& 0.05	&	1.27	&	0.06	&	1.81	&	0.28	&	1.80	&	0.05	&	1.27	\\
2	& 0.05	&	2.30	&	0.05	&	3.77	&	0.05	&	4.92	&	0.05	&	2.47	\\
3	& 0.05	&	2.47	&	0.05	&	4.76	&	0.04	&	4.96	&	0.04	&	2.46	\\
4	& 0.05	&	2.52	&	0.05	&	4.92	&	0.04	&	5.00	&	0.05	&	2.53	\\
5	& 0.05	&	2.51	&	0.05	&	5.04	&	0.05	&	5.02	&	0.05	&	2.49	\\
6	& 0.05	&	2.50	&	0.05	&	4.99	&	0.05	&	5.03	&	0.05	&	2.51	\\
7	& 0.05	&	2.51	&	0.05	&	5.01	&	0.05	&	4.99	&	0.05	&	2.50	\\
8	& 0.05	&	2.53	&	0.05	&	4.99	&	0.05	&	4.99	&	0.05	&	2.48	\\
9	& 0.05	&	2.51	&	0.05	&	5.00	&	0.05	&	5.05	&	0.05	&	2.50	\\
10	& 0.05	&	2.56	&	0.05	&	4.96	&	0.05	&	4.97	&	0.05	&	2.52	\\
11	& 0.05	&	2.50	&	0.05	&	5.01	&	0.05	&	5.00	&	0.05	&	2.50	\\

\hline
\end{tabular}}
\caption[Average SRC class residuals on the random database model in the case of noise]{Average SRC class residuals $\err_1(\bm{y}) := \|\bm{y} - X_\mathrm{tr}\delta_1(\bm{\alpha}_{1,\epsilon})\|_2$ and $\min_{2\leq l \leq L} \{\err_l(\bm{y}):=\|\bm{y} - X_\mathrm{tr}\delta_l(\bm{\alpha}_{1,\epsilon})\|_2\}$ (over 1000 trials) on the random database model in the case of noise.}
\label{tab:class_resids_noise}
\end{table*}

Noting that $\epsilon:= C\zeta = 0.05$, we see that the ideal classification scenario occurred in nearly all cases. That is, since $\err_1(\bm{y})\approx \epsilon$ almost always, class 1 training samples made up essentially the entire approximation of the test sample. The exception, again, was DB-3 at Stage 1, for which $\err_1(\bm{y})$ and $\min_{2\leq l \leq L} \err_l(\bm{y})$ were the least separated (i.e., relatively close in value). However, correct classification would still be achieved. 

The reader might notice that the quantities $\min_{2\leq l \leq L} \err_l(\bm{y})$ at Stage 1 are lower than at higher stages; this is because 
\begin{align*}
\min_{2\leq l \leq L} \err_l(\bm{y}) = \min_{2\leq l \leq L} \|\bm{y}-X_\mathrm{tr}\delta_l(\bm{\alpha}_{1,\epsilon})\|_2 \approx \|\bm{y}-X_\mathrm{tr}\bm{0}\|_2= \|\bm{y}\|_2 
\end{align*}
is smaller in this case, due to the class 1 training samples being uniformly distributed on $S^{m-1}$.

\subsection{Summary}

In this section, we designed a model, inspired by the work of Wright and Ma \cite{wri:dense}, for facial recognition and other similar classification databases. To model the mechanisms of SRC \cite{wri:src}, we randomly generated a test sample as a non-negative linear combination of a single class's training samples. We computed the corresponding (sparse) coefficient vector and then ran experiments to test whether or not $\ell^1$-minimization, as it is used in the SRC setting, could recover this vector under increasing values of correlation, both within-class and in the database as a whole.

The results demonstrate that the within-class correlation in this model consistently improves $\ell^1/\ell^0$-recovery when compared to randomly-generated uniform data on the sphere. This is an important empirical result, as this latter type of data is one of the ``golden children'' of $\ell^1/\ell^0$-equivalence; i.e., these type of dictionaries produce, in some sense, ideal recovery (see, e.g., the work of Donoho \cite{don:und}). However, those results are strongly asymptotic, and our experiments dealt only with small databases. More work is needed to determine if our findings hold up on larger datasets. 

It is not too surprising, given the mutual coherence recovery condition studied in the last section, that very large correlation in the database as a whole can degrade recovery. When the global correlation in our model was very high, so that the classes, or sub-bouquets, began to overlap, we saw that $\ell^1$-minimization did not find the correct support of the sparse solution. However, we showed that the support could be completely fixed by a simple thresholding technique.

We also demonstrated that $\ell^1$-minimization achieved a good approximation of the sparsest solution in the case of noise in our model. Though the accuracy of the approximation generally decreased as the data became more correlated, this deterioration was slow compared to the increase in mutual coherence of the database. Further, the amount of $\ell^2$-error appeared to be less dependent on the relationship between noise $\zeta$ and error tolerance $\epsilon$ than it was on the amount of redundancy in the database.

Assuming that test samples truly are linear combinations of their ground truth class training samples, as is done in SRC, these experiments suggest that $\ell^1$-minimization will recover this class representation, leading to good classification in SRC and similar classification algorithms. This of course assumes that our model is appropriate for the given dataset, and that its values of $N_0$, $m$, and $L$ are comparable to those used in our experiments, so that the class representation is sparse. 

Our results are purely empirical; however, they strongly suggest that theoretical recovery results are possible. We conjecture that exact recovery can be provably obtained whenever the classes are sufficiently non-overlapping and that a similar result can be obtained in the case of noise. The amount of redundancy in the database and the number of classes will play a crucial role in this analysis.

Finally, though we explicitly modeled the cone structure of facial images, our results are likely applicable to other areas of classification as well. In particular, as long as it is assumed that the training samples within each class are highly correlated, we could amend our model so that the sign of each training sample was chosen randomly and so that the test sample was generated in the linear (not necessarily positive) span of its same-class training samples. However, since $\ell^1$-minimization is invariant to multiplication of the dictionary elements by $\pm 1$, we suspect that our results would be the same.

\section{Proving Equivalence via Nonlinear Embedding}\label{MCD_Project_3}

\subsection{The Idea}

As we have seen, class structure often results in the training set having high mutual coherence, making it impossible to apply the mutual coherence recovery guarantees given in Theorems \ref {thm:mc} and \ref{thm:mc_noise} in the context of SRC. We consider a resolution to this conflict through the use of \emph{more space}. That is, if we had many ``extra'' dimensions, the data in each class could conceivably be spread out and we would still have enough ``room'' to keep the classes well-separated from each other, allowing for both low mutual coherence and class-structured data.

Let us illustrate this in low dimension. Consider the toy example in which we have $L=2$ classes, each containing $2$ samples in $\mathbb{R}^m$ for $m=2$. First, let the goal be to arrange the samples in a way that minimizes their mutual coherence while at the same time provides some indication of class. Assuming that the samples must be normalized (as in SRC), this class-structure criterion can reasonably be interpreted as the requirement that 
\begin{align*}
\Big|\ip{\bm{x}_i^{(1)},\bm{x}_j^{(1)}}\Big| > \Big|\ip{\bm{x}_i^{(1)},\bm{x}_j^{(2)}}\Big| \text{ and }
\Big|\ip{\bm{x}_i^{(2)},\bm{x}_j^{(2)}}\Big| > \Big|\ip{\bm{x}_i^{(2)},\bm{x}_j^{(1)}}\Big|,
\end{align*}
for $i,j \in \{1,2\}$. In other words, the samples in the same class must be more correlated than samples in different classes.

One solution is given by the class matrices

\begin{align*}
X^{(1)} = \big[\bm{x}_1^{(1)},\bm{x}_2^{(1)}\big] = \begin{bmatrix} 1 & \cos(\frac{\pi}{4}-\epsilon)\\ 0 & \sin(\frac{\pi}{4}-\epsilon) \end{bmatrix}, \;\;\; X^{(2)} = \big[\bm{x}_1^{(2)},\bm{x}_2^{(2)}\big] =\begin{bmatrix} 0 & \cos(\frac{3\pi}{4}-\epsilon)\\ 1 & \sin(\frac{3\pi}{4}-\epsilon) \end{bmatrix},
\end{align*}
where $\epsilon > 0$ is small.
The magnitude of the inner product between samples in the same class is $\cos(\frac{\pi}{4}-\epsilon)$, and that of samples in different classes is $\cos(\frac{\pi}{4}+\epsilon)$. Clearly, the former quantity is the mutual coherence of the dataset. This arrangement is illustrated in Figure \ref{fig:2D} with $\epsilon = 0.2$. 

Now, consider the same problem but in the case that we are given a third dimension. It is clear that we will be able to decrease the mutual coherence of the dataset by moving samples into this extra space. One solution is given by the class matrices
\begin{align*}
X^{(1)} = \begin{bmatrix} 1 & \cos(\theta_1)\sin(\phi_1)\\ 0 & \sin(\theta_1)\sin(\phi_1) \\ 0 & \cos(\phi_1) \end{bmatrix}, \;\;\; X^{(2)} = \begin{bmatrix} 0 & \cos(\theta_2)\sin(\phi_2)\\ 1 & \sin(\theta_2)\sin(\phi_2) \\ 0 & \cos(\phi_2) \end{bmatrix},\\
\end{align*}
for $\theta_1 = \pi/4 - \epsilon$, $\theta_2 = \pi/4 + \epsilon$, $\phi_1 = 3\pi/4$, and $\phi_2 = \pi/4$.
The mutual coherence of the dataset is $ \cos(\frac{\pi}{4}-\epsilon)\sin(\frac{3\pi}{4}) = \sin(\frac{\pi}{4}+\epsilon)\cos(\frac{\pi}{4})$. This arrangement is illustrated in Figure \ref{fig:3D} with $\epsilon = 0.2$. 

For $\epsilon = 0.2$, for example, adding an additional dimension allows us to decrease the mutual coherence of the dataset from $\cos(\frac{\pi}{4}-\epsilon) \approx 0.8335$ to $ \cos(\frac{\pi}{4}-\epsilon)\sin(\frac{3\pi}{4}) \approx 0.5894$. This is a substantial decrease. 

\begin{figure}[!htb]
\hspace*{\fill}%
\centering
\subcaptionbox{$m= 2$, $\mu = 0.8335$\label{fig:2D}}
 [.49\linewidth]{\includegraphics[height=5cm]{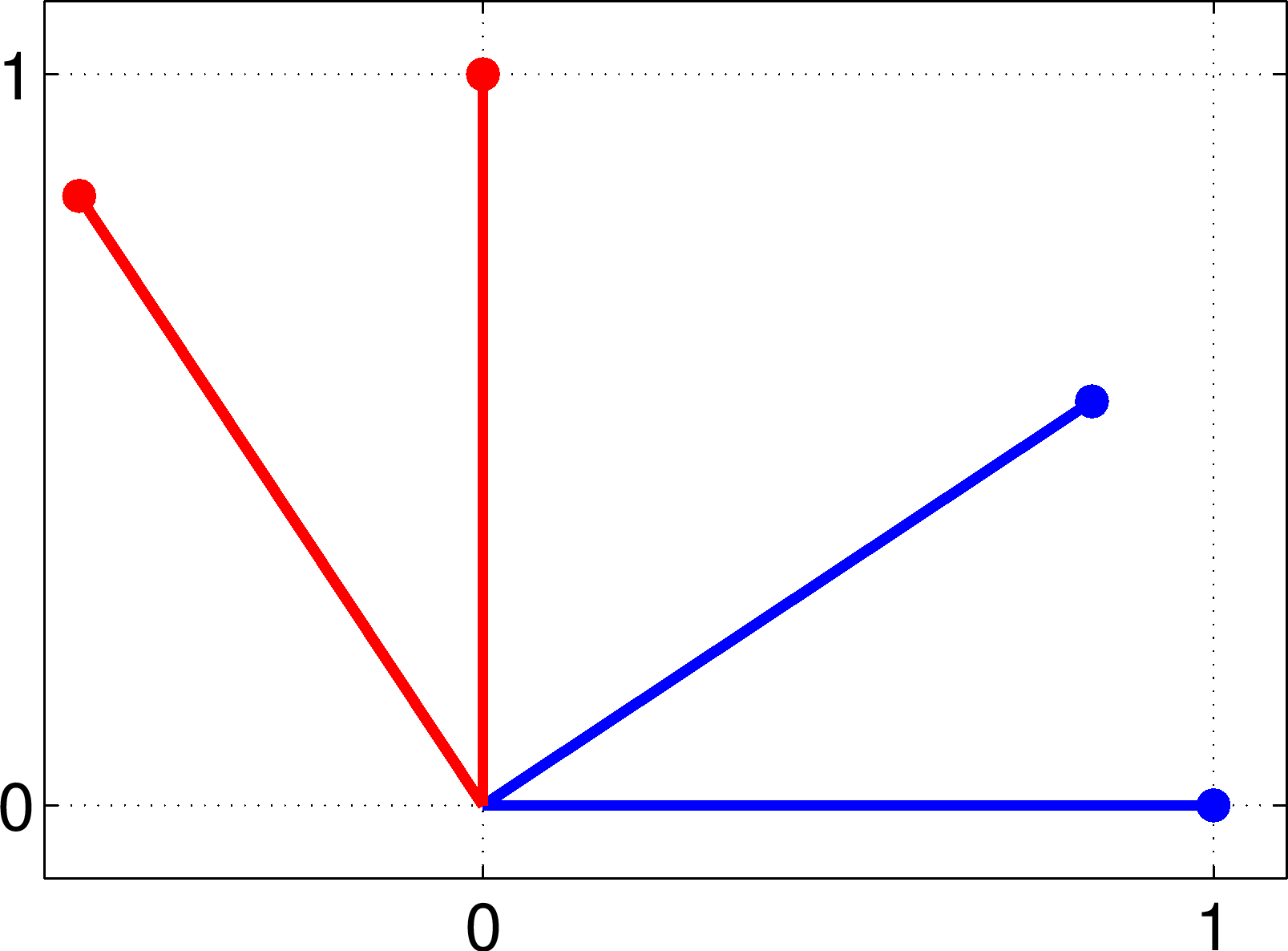}}
\hfill%
\subcaptionbox{$m=3$, $\mu = 0.5894$\label{fig:3D}}
 [.49\linewidth]{\includegraphics[height=6cm]{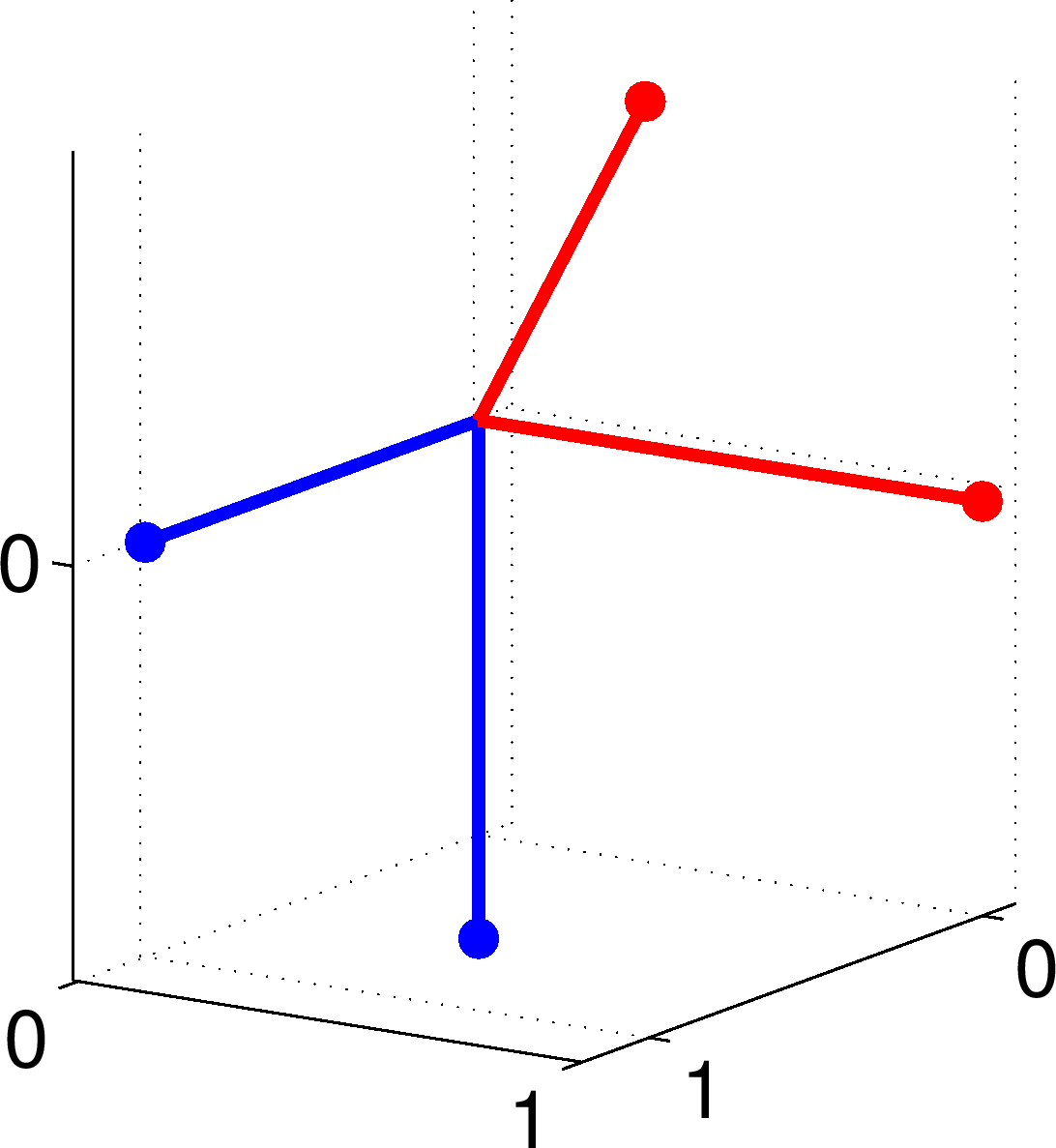}}
\caption[Illustration of decreasing the mutual coherence of a dataset by embedding it into higher dimension]{Illustration of decreasing the mutual coherence of a dataset by embedding it into higher dimension. (a) Samples in original space $\mathbb{R}^2$, (b) Samples in transform space $\mathbb{R}^3$. Colors denote classes. }
\label{fig:syn_tvs}
\hspace*{\fill}%
\end{figure}

\subsection{Formulation and Obstacles}

As discussed above, we consider \emph{forcing} the mutual coherence criterion of $\ell^1/\ell^0$-equivalence to hold via data transformation. Is it possible to learn a class-preserving transform from the training data and then classify test samples in a space in which $\ell^1$-minimization provably produces the sparsest solution? Such a transform would allow us to investigate the extent (if any) to which obtaining the sparsest solution affects SRC's classification accuracy.

For a transform $\phi$, set $\Phi(X_\mathrm{tr}) := [\phi(\bm{x}_1),\ldots,\phi(\bm{x}_{N_\mathrm{tr}})]$ for notional ease. Formally, we desire a transform $\phi^*: \mathbb{R}^m \rightarrow \mathbb{R}^{\tilde{m}}$ with $m < \tilde{m}$ that satisfies 
\begin{equation}\label{eq:explicit}
\phi^* = \arg\max_{\phi \in \mathcal{C}} f_\mathrm{cs}(\Phi(X_\mathrm{tr})) \text{ subject to } \mu(\Phi(X_\mathrm{tr}))\leq \tilde{\mu},
\end{equation}
where $\mathcal{C}$ is some compact set (so that $f_\mathrm{cs}$ obtains a maximum). Here, $f_\mathrm{cs}$ evaluates the amount of class structure in the transformed training set (``cs'' stands for ``class structure''). For example, $f_\mathrm{cs}$ might denote the inverse of the sum of within-class distances or the inverse of the Frobenius norm of the within-class scatter matrix used in \emph{linear discriminant analysis} \cite{fish:lda,rao:lda}. Clearly, $\tilde{\mu}$ is an upper bound on the mutual coherence of the transformed training set. Ideally, we want to choose $\tilde{\mu}$ small enough so that the $\ell^1/\ell^0$-equivalence condition in Theorem \ref{thm:mc} or Theorem \ref{thm:mc_noise} can be applied. 

We note that the desired transform $\phi^*$ must be a \emph{nonlinear} transform, otherwise the dimension of the subspace containing the embedded samples will be no greater than that of the original space ($m$). Thus we will have failed to utilize the extra space (needed to achieve our objective) awarded by the increased ambient dimension $\tilde{m}$.

Though this setup seems promising, we have a problem when we consider how the transform $\phi^*$ should treat (new) test samples. In order for us to classify the test sample in the transform space, $\phi^*$ must treat $\bm{y}$ similarly to a training sample in its own class. However, this leads to the following conflict:

\begin{myprop} \label{prop:up_bound_restate}
Let $\phi: \mathbb{R}^m \rightarrow \mathbb{R}^{\tilde{m}}$ be a data transform and $\bm{y}$ a test sample so that $\mu([\Phi(X_\mathrm{tr}),\phi(\bm{y})]) \approx \mu(\Phi(X_\mathrm{tr})) \leq \tilde{\mu}$ for some $\tilde{\mu}$, i.e., the transform $\phi$ \emph{treats} test samples in the same way as their same-class training samples. For any vector $\bm{\alpha} \in \mathbb{R}^m$, if $\bm{\alpha}$ satisfies $\|\bm{\alpha}\|_0<\frac{1}{2}\Big(1+\frac{1}{\tilde{\mu}}\Big)$, then
\begin{align*}
\Phi(X_\mathrm{tr})\bm{\alpha} \neq \phi(\bm{y})
\end{align*}
with high probability.
\end{myprop}

\begin{proof}
  This is a direct consequence of Corollary \ref{cor:y_dist}.
\end{proof}
This demonstrates the extent to which the assumptions in SRC conflict with the mutual coherence recovery guarantees. We \emph{cannot} construct a transform which can be applied to the entire dataset and allows for both sufficiently-low mutual coherence and adequate grouping of the classes, so that (transformed) test samples can be expressed as linear combinations of their same class training samples. 

However, we can still use the nonlinear transformation approach to \emph{study the relationship between classification accuracy of SRC and the sparsity of its solution vector}. We do this by artificially generating transformed test samples $\phi(\bm{y})$ as linear combinations of the columns of the transformed training data, in particular, with nonzero coefficients occurring at training samples in the ground truth class of $\bm{y}$. Thus we can ensure that $\Phi(X_\mathrm{tr})\bm{\alpha} = \phi(\bm{y})$ always has a solution. However, this will mean that we never actually compute or handle the test sample $\bm{y}$ in the original space and only \emph{assume that it exists implicitly}.

The reader may object that we cannot just make up test samples in this manner, and in general, this is absolutely true. Nevertheless, we stress that our goal in this experiment is not to classify an arbitrary database but to determine the effect of sparsity in SRC on classification accuracy, and so the implied existence of $\bm{y}$ is acceptable in this context. 

In the next two subsections, we reveal our approach to determining the desired transform $\phi^*$ and further discuss the consequences of Proposition \ref{prop:up_bound_restate} (and our approach to handling them) in this particular context.

\subsection{Using Gaussian Kernels}\label{sec:method}

Rather than constructing an explicit transform, we consider the reduction of mutual coherence via the so-called \emph{kernel trick.} We will use the \emph{Gaussian kernel} as a method of controlling the mutual coherence of the transformed training data. 

To review, the kernel trick allows us to perform operations in a space of dimension $\tilde{m} > m$ (possibly infinite-dimensional) without having to actually compute the transformed samples. The ``trick'' is to work only with the inner-products between transformed samples, which are given to us by some kernel function $\kappa:\mathbb{R}^m\times \mathbb{R}^m \rightarrow \mathbb{R}$. More formally, denote the transform by $\phi_\kappa$. We define the inner-product in the kernel space as
\begin{align*}
\ip{\phi_\kappa(\bm{x}_i),\phi_\kappa(\bm{x}_j)}: = \kappa(\bm{x}_i,\bm{x}_j),
\end{align*}
for $1\leq i,j\leq N_\mathrm{tr}$. The kernel function $\kappa$ should satisfy \emph{Mercer's condition}\footnote{The kernel $\kappa$ satisfies Mercer's condition if $\iint{\kappa(\bm{x},\bm{y}) g(\bm{x})g(\bm{y})\dd{\bm{x}}\dd{\bm{y}}}\geq 0$ for all square-integrable functions $g$.} so that $\kappa$ defines a proper inner-product \cite{vap:svm}. 

Kernel methods can be particularly effective when used to ``non-linearize'' linear classifiers. In \emph{kernel support vector machines}, for example, classes that are not linearly-separable in the original space may be separated linearly in kernel space (see the work of Boser et al.\ \cite{bos:ksvm}). Though SRC is not linear, it does assume a linear relationship between the test sample and the training samples in its ground truth class. When such a relationship does not hold in the original space, it may hold in kernel space given that an appropriate kernel is selected \cite{yin:ksrc}. 

Consider the Gaussian kernel, which is given by
\begin{align*}
\kappa(\bm{x}_i,\bm{x}_j) := \mathrm{e}^{-\frac{\|\bm{x}_i-\bm{x}_j\|_2^2}{\sigma^2}}.
\end{align*}
Essentially, the Gaussian kernel adds \emph{inverse exponential scaling }to the Euclidean distance function. Points close together obtain values of $\kappa$ that are close to 1, whereas points that are faraway from each other have kernel values approaching 0. The \emph{window} or \emph{width} parameter $\sigma$ controls the drop off (or steepness) of this trade-off. 

The Gaussian kernel is a natural choice for our transform, since the mutual coherence of the (transformed) training set will be given by
\begin{align*}
\mu(\Phi_\kappa(X_\mathrm{tr})) &= \max_{1\leq i \neq j \leq N_\mathrm{tr}} |\ip{\phi_\kappa(\bm{x}_i),\phi_\kappa(\bm{x}_j)}| \\
&= \max_{1\leq i \neq j \leq N_\mathrm{tr}} |\kappa(\bm{x}_i,\bm{x}_j)| \\
& = \max_{1\leq i \neq j \leq N_\mathrm{tr}} \mathrm{e}^{-\frac{\|\bm{x}_i-\bm{x}_j\|_2^2}{\sigma^2}}.
\end{align*}
Since the vectors $\hat{\bm{x}}_i$ and $\hat{\bm{x}}_j$ satisfying $\|\hat{\bm{x}}_i - \hat{\bm{x}}_j\|_2 = \max_{i \neq j}\|\bm{x}_i-\bm{x}_j\|_2$ are fixed for a given training set, the mutual coherence $\mu(\Phi_\kappa(X_\mathrm{tr}))$ depends completely on $\sigma$. Thus we can write $\mu(\Phi_\kappa(X_\mathrm{tr})) =: \mu = \mu(\sigma)$. To reiterate, we can completely control the mutual coherence of the data in the kernel space by adjusting $\sigma$. 

Our goal is to use the kernel trick with the Gaussian kernel to investigate what happens to the classification accuracy of SRC when $\ell^1/\ell^0$-equivalence is achieved in kernel space. We will do this as follows: In order to ensure $\ell^1/\ell^0$-equivalence, the Gaussian width parameter $\sigma$ must be chosen so that the mutual coherence is small enough that Theorem \ref {thm:mc} holds. Let us set
\begin{align*}
k_\mathrm{sup} := \frac{1}{2}\Big(1+\frac{1}{\mu}\Big).
\end{align*}
Clearly, $k_\mathrm{sup}$ completely depends on $\mu$, or equivalently, on $\sigma$. As $\sigma$ approaches 0, $k_\mathrm{sup} = k_\mathrm{sup}(\sigma)$ blows up. Suppose we choose $\sigma$ to be the largest value such that, with high probability (whp), the sparsity level $\|\bm{\alpha}_1\|_0$ is less than $k_\mathrm{sup}$, where $\bm{\alpha}_1 = \bm{\alpha}^* \in \mathbb{R}^{N_\mathrm{tr}}$ is the solution to the exact $\ell^1$-minimization problem in SRC given by Eq.~\eqref{eq:src_opt} (replacing $X_\mathrm{tr}$ with $\Phi_\kappa(X_\mathrm{tr})$ and $\bm{y}$ with $\phi_\kappa(\bm{y})$). This will ensure that $\bm{\alpha}_1$ is the sparsest solution by Theorem \ref{thm:mc}. Using ``mc'' to denote ``mutual coherence,'' we define
\begin{equation}\label{eq:def_sigma_mc}
\sigma_\mathrm{mc} := \max\Big\{\sigma : \|\bm{\alpha}_1\|_0 \stackrel{\mathrm{whp}}< k_\mathrm{sup} \Big\}.
\end{equation}
It follows that $\phi_\kappa = \phi^*$, our desired transform, when $\sigma = \sigma_\mathrm{mc}$ and the class-structure evaluation $f_\mathrm{cs}$ in Eq.~\eqref{eq:explicit} is defined as the \emph{minimum spread} of vectors in transform space. (We assume that the database already has class-structure in the original space---so that the mutual coherence is high---and by the continuity of the Gaussian kernel, $\phi_\kappa$ with $\sigma = \sigma_\mathrm{mc}$ separates the data in each class only as much as necessary to achieve the mutual coherence bound.)

To relate $\sigma$ and classification accuracy, we consider the set of values of $\sigma$ such that maximum classification accuracy is achieved for all values in this set (whp). (We can think of this as the range of $\sigma$ values that produce the maximum amount---without a mutual coherence constraint---of class structure.) Defining the maximum value in this set by $\sigma_\mathrm{acc}$, we want to investigate the relationship between $\sigma_\mathrm{mc}$ and $\sigma_\mathrm{acc}$. We are also interested in the sparsity level $\|\bm{\alpha}_1\|_0$ of the $\ell^1$-minimized coefficient vector at both $\sigma = \sigma_\mathrm{mc}$ and $\sigma = \sigma_\mathrm{acc}$. Since some coefficients may be small, we also consider the size of the coefficients of training samples corresponding to the ground truth class of $\bm{y}$. In analyzing these quantities and relationships, we aim to provide insight into the role of sparsity in classification. 

\subsection{Handling Test Samples}

We elaborate on the effect of Proposition \ref{prop:up_bound_restate} in the kernel setup: For a fixed training set and test sample (in the original space), we lose the ability to write $\Phi_\kappa(X_\mathrm{tr})\boldsymbol{\alpha} = \phi_\kappa(\bm{y})$ for \emph{any} coefficient vector $\bm{\alpha}$ as $\sigma \rightarrow 0$. Recall that this equality is a key aspect of the mutual coherence condition in Theorem \ref {thm:mc}. In decreasing $\sigma$, we cause not only the training samples to become more orthogonal to each other, but also the test sample to become more orthogonal to each training sample, to the point that when $\sigma = \sigma_\mathrm{mc}$, $\phi_\kappa(\bm{y})$ is likely not contained in the span of the columns of $\Phi_\kappa(X_\mathrm{tr})$. In other words, the resulting system is overdetermined with no solution to $\Phi_\kappa(X_\mathrm{tr})\bm{\alpha} = \phi_\kappa(\bm{y})$ when $\sigma \leq \sigma_\mathrm{mc}$. By Theorem \ref {thm:mc}, the minimal $\ell^1$-norm solution satisfying $\Phi_\kappa(X_\mathrm{tr})\bm{\alpha}_1 = \phi_\kappa(\bm{y})$ with $\|\bm{\alpha}_1\|_0 < (1/2)(1+(1/\mu))$ is necessarily the sparsest such solution. However, if there is no solution satisfying $\Phi_\kappa(X_\mathrm{tr})\bm{\alpha} = \phi_\kappa(\bm{y})$, then there can be no sparsest solution. 

Even when the equality in SRC is relaxed and the constrained $\ell^1$-minimization problem in Eq.~\eqref{eq:cs_l1_error} is used,\footnote{Note that the formulation in Eq.~\eqref{eq:cs_l1_error} is equivalent to the regularized $\ell^1$-minimization problem in Eq.~\eqref{eq:src_opt_noise} in the formal SRC algorithm statement.} relating the found solution $\bm{\alpha}_{1,\epsilon}$ and the true sparsest solution $\bm{\alpha}_0$ using Theorem \ref{thm:mc_noise} requires the \emph{existence} of some $\bm{\alpha} = \bm{\alpha}_0$ satisfying the equality $\Phi_\kappa(X_\mathrm{tr})\bm{\alpha} = \phi_\kappa(\bm{y})$. Since the bound in Eq.~\eqref{eq:k_max_noise} in the noisy case is more restrictive than Eq.~\eqref{eq:k_max} in the noiseless case, to satisfy Theorem \ref{thm:mc_noise} we must have $\sigma < \sigma_\mathrm{mc}$. By Proposition \ref{prop:up_bound_restate}, no such $\bm{\alpha}$ exists, and it follows that Theorem \ref{thm:mc_noise} cannot be applied in this setup, either. 

As discussed earlier, we will side-step this conflict by artificially generating test samples in transform (kernel) space. This approach affects the accuracy of SRC as follows: As $\sigma \rightarrow 0$ and the training data become closer to orthogonal, we will never lose the relationship $\phi_\kappa(\bm{y}) \in \spn\{\phi_\kappa(\bm{x}_1^{(l)}),\ldots,\phi_\kappa(\bm{x}_{N_l}^{(l)})\}$. Thus we will not see the classification performance deteriorate \emph{at all} as $\sigma \rightarrow 0$.\footnote{We stress that this is certainly not the case in general: consider the increasing difficulty of identifying class structure in a dataset whose samples become more and more uncorrelated (as $\sigma \rightarrow 0$). Thus generating $\phi_\kappa(\bm{y})$ in this manner adds an undesirable---but necessary---degree of artificiality into our experiment.} In other words, decreasing $\mu$ so that we can provably obtain $\ell^1/\ell^0$-equivalence in this setup can only \emph{help} classification accuracy, as doing so isolates the linear relationship between the test sample and the training samples in its ground truth class (in kernel space). Thus our investigation of the relationship between $\sigma_\mathrm{acc}$ and $\sigma_\mathrm{mc}$ can be more precisely stated in terms of how much larger $\sigma_\mathrm{acc}$ is than $\sigma_\mathrm{mc}$, i.e., how quickly does classification accuracy deteriorate after we no longer have $\ell^1/\ell^0$-equivalence?

\subsection{Experiments}

\subsubsection{Experimental Setup}\label{mcd_3_exp_setup}
For a fixed training set (that will be described in detail in Section \ref{sec:database}) and fixed $\sigma$, we generate $N_l$ test samples in kernel space for each class $1\leq l \leq L$ as linear combinations of the training samples in that class (in kernel space) with coefficients randomly drawn from $\unif(0,1)$ distribution. Non-negative coefficients are used so that $\ip{\phi_\kappa(\bm{y}),\phi_\kappa(\bm{x}_j)}\geq 0$ for $1\leq j \leq N_\mathrm{tr}$, as is consistent with the Gaussian kernel. We then apply SRC in kernel space to classify the resulting test samples, using the Kernel SRC algorithm of Kang et al., in particular, their \emph{kernel coordinate descent} (KCD) algorithm \cite{kang:ksrc}. Note that in their paper, the authors apply this algorithm to the \emph{local binary patterns} of the original samples instead of the original samples themselves, and since other types of kernels are more appropriate for these type of features, they do not use the Gaussian kernel, as we do. 

In our experiments, we determine $\sigma_\mathrm{mc}$ and $\sigma_\mathrm{acc}$ by trial-and-error. 
Given the randomness inherent in the database construction (again, see Section \ref{sec:database} for a description of the database used), determining these values is not an exact science, and we do our best to make judicious and consistent choices in terms of rounding, etc. Additionally, note that we thresholded the entries of each $\ell^1$-minimized coefficient vector $\bm{\alpha}_1$ by $10^{-10}$ to help avoid rounding errors.

\subsubsection{An Upper Bound}

We saw in Section \ref{MCD_Project_1} that we cannot apply Theorem \ref {thm:mc} unless $\mu < \frac{1}{3}$. Since we are using the kernel approach, this means that we must have
\begin{align*}
\mu(\Phi_\kappa(X_\mathrm{tr})) = \max_{1\leq i\neq j \leq N_\mathrm{tr}} \ip{\phi_\kappa(\bm{x}_i),\phi_\kappa(\bm{x}_j)} = \max_{1\leq i\neq j \leq N_\mathrm{tr}} \kappa(\bm{x}_i,\bm{x}_j) < \frac{1}{3}.
\end{align*}
In particular, since we are using the Gaussian kernel, it must be the case that
\begin{align*}
\max_{i\neq j}\kappa(\bm{x}_i,\bm{x}_j) &= \max_{1\leq i\neq j \leq N_\mathrm{tr}} \mathrm{e}^{-\frac{\|\bm{x}_i-\bm{x}_j\|_2^2}{\sigma^2}} < \frac{1}{3} \\
\Rightarrow \sigma &< \frac{1}{\sqrt{\ln 3}}\max_{1\leq i\neq j \leq N_\mathrm{tr}} \|\bm{x}_i - \bm{x}_j\|_2. 
\end{align*}
Since the training samples (in the original space) are normalized, this means that
\begin{equation} \label{eq:sigma_ub}
\sigma < \frac{2}{\sqrt{\ln 3}} \approx 1.35.\end{equation}
Thus in searching for $\sigma_\mathrm{mc}$, we only need to consider values of $\sigma$ less than 1.35.

\subsubsection{Database Description} \label{sec:database}

We constructed a very simple toy database in the original space as follows: Samples in the $l$th class were initially $N_0$ copies of the canonical basis vector $\bm{e}_l \in \mathbb{R}^L$, where $L$ was the number of classes. The feature dimension $m$ was user-specified, and then $m-L$ coordinates were added to each canonical basis vector and set to zero. Lastly, random noise from $\mathcal{N}(0,\eta^2)$ was added to all (training) samples in all coordinates.

We set $N_0 = 5$, $m = 50$, and $L=20$, so that each class would consist of a relatively small portion of the dictionary $X_\mathrm{tr}$, as is ideal in SRC. Recall our method of generating test samples as linear combinations of their same-class training samples (in kernel space) in Section \ref{mcd_3_exp_setup}. We set the number of test samples in each class to $N_l = N_0 = 5$, so that we had the same number of test samples as training samples. We used three different values of noise level $\eta \in \{0.001,0.1,0.5\}$. As in the $\ell^1$-minimization algorithm HOMOTOPY, KCD requires an error/sparsity tradeoff parameter $\lambda$. To force near-exactness in the representations, we set $\lambda= 10^{-10}$.


\begin{remark}
The reader may question why we used a different synthetic database than the one in the last section: Would not this be better, so that we might obtain a \emph{fair comparison}? We are making our best effort to stress that this line of thinking misconstrues the point of this experiment. Here, we only care about the classification results of Kernel SRC as they relate to the sparsity level $\|\bm{\alpha}_1\|_0$ and the mutual coherence bound in Eq.~\eqref{eq:k_max}. We are not at all interested in whether the kernel approach improves the classification accuracy of SRC (for a positive answer to this question, see, for example, Kang et al.'s paper \cite{kang:ksrc}). Further, the previous synthetic database was designed for a specific purposes: $\ell^1/\ell^0$-equivalence---and not classification performance in SRC---could be studied at increasing levels of data correlation. So that the aim of this previous experiment did not bleed into our goals here, we used a completely new (and very simple) database.
\end{remark}

\subsubsection{Results}

In Figure \ref{fig:Clean}, we plot the synthetic database results for each value of $\eta$ over various values of $\sigma$, annotating the values of $\sigma_\mathrm{mc}$ and $\sigma_\mathrm{acc}$. We report the averages over 100 instantiations of the training and test sets (``trials''). In particular, we report the average sparsity level, Kernel SRC classification accuracy, and the (relative) $\ell^2$ and $\ell^1$-norms of the correct class support. These quantities are defined rigorously as
\begin{align*}
\mathrm{Sparsity} := \operatorname*{mean}_{\mathrm{all}\; \mathrm{trials}} \Big\{\operatorname*{median}_{\mathrm{all} \; \mathrm{test} \: \mathrm{samples}} \frac{\|\bm{\alpha}_1\|_0}{N_\mathrm{tr}}\Big\}
\end{align*}
for $\bm{\alpha}_1$ thresholded at $10^{-10}$ (we compute the median sparsity over all test samples so that the result is more robust to atypical very sparse or very dense coefficient vectors),
\begin{align*}
\mathrm{Accuracy} := \operatorname*{mean}_{\mathrm{all}\; \mathrm{trials}} \Big\{\operatorname*{mean}_{\mathrm{all} \; \mathrm{test} \: \mathrm{samples}} \mathbbm{1}_{\{\classlabel{(\bm{y})} = \operatorname{ground\_truth\_class}(\bm{y})\}}\Big\}
\end{align*}
where $\mathbbm{1}_{\{x = y\}}$ is the indicator function that returns 1 if $x=y$ and 0 otherwise, and
\begin{align*}
\supp(\ell^2) := \operatorname*{mean}_{\mathrm{all}\; \mathrm{trials}} \Big\{\operatorname*{mean}_{\mathrm{all} \; \mathrm{test} \: \mathrm{samples}} \frac{\|\delta_{\mathrm{GT}}(\bm{\alpha}_1)\|_2}{\|\bm{\alpha}_1\|_2}\Big\},\:\:\: \supp(\ell^1) := \operatorname*{mean}_{\mathrm{all}\; \mathrm{trials}} \Big\{\operatorname*{mean}_{\mathrm{all} \; \mathrm{test} \: \mathrm{samples}} \frac{\|\delta_{\mathrm{GT}}(\bm{\alpha}_1)\|_1}{\|\bm{\alpha}_1\|_1}\Big\},
\end{align*}
where the nonzero entries of $\delta_{\mathrm{GT}}(\bm{\alpha}_1)$ are exactly those from $\bm{\alpha}_1$ that correspond to the ground truth class of the given test sample. 

\begin{figure}[!htb]%
\centering
\begin{subfigure}[b]{0.65\textwidth}
\centering
	\includegraphics[width=\linewidth]{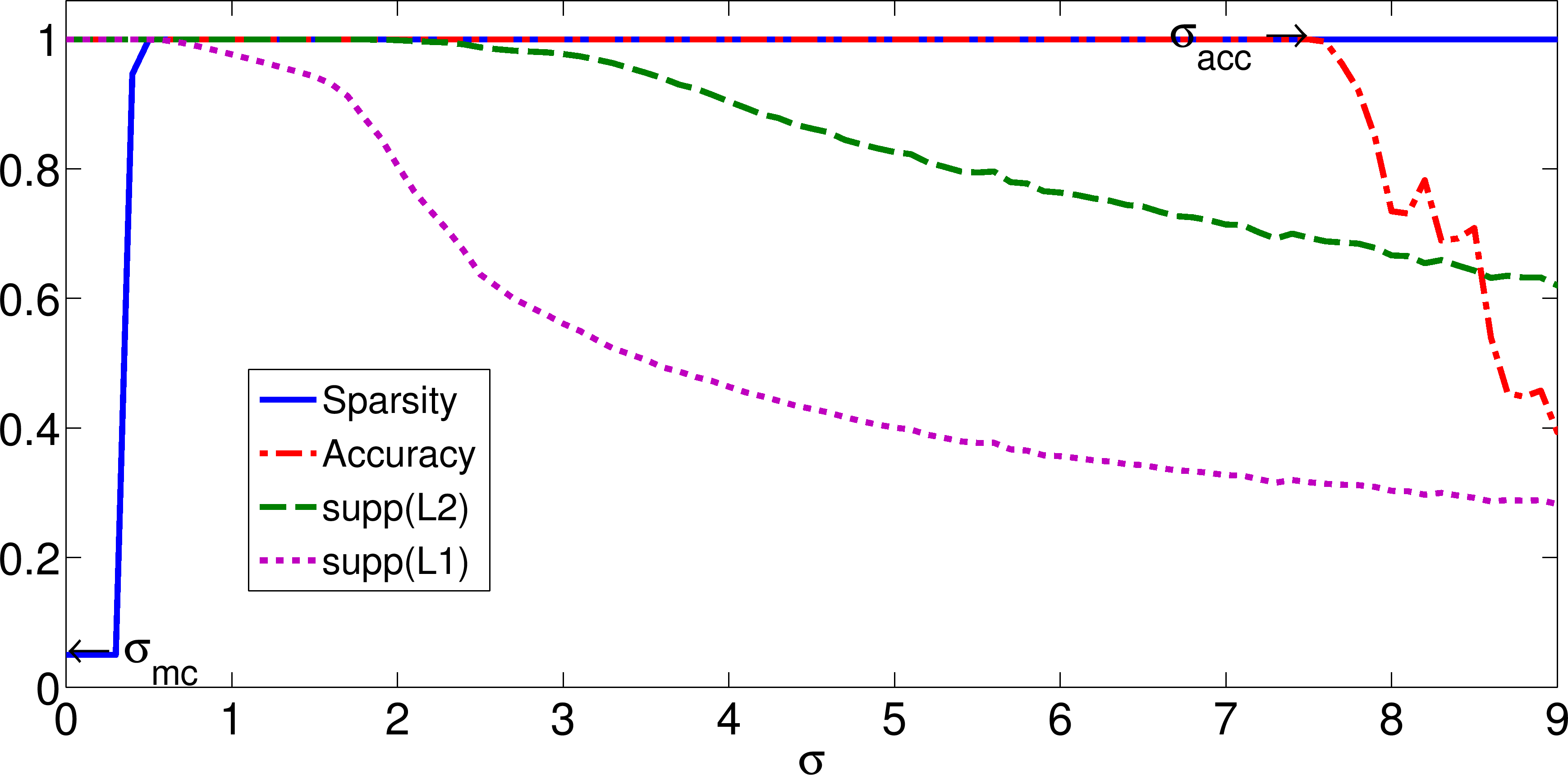}
		\caption{$\eta = 0.001$}
	\label{fig:Clean_10} \hfill%
\end{subfigure}
\begin{subfigure}[b]{0.65\textwidth} 
\centering
	\includegraphics[width=\linewidth]{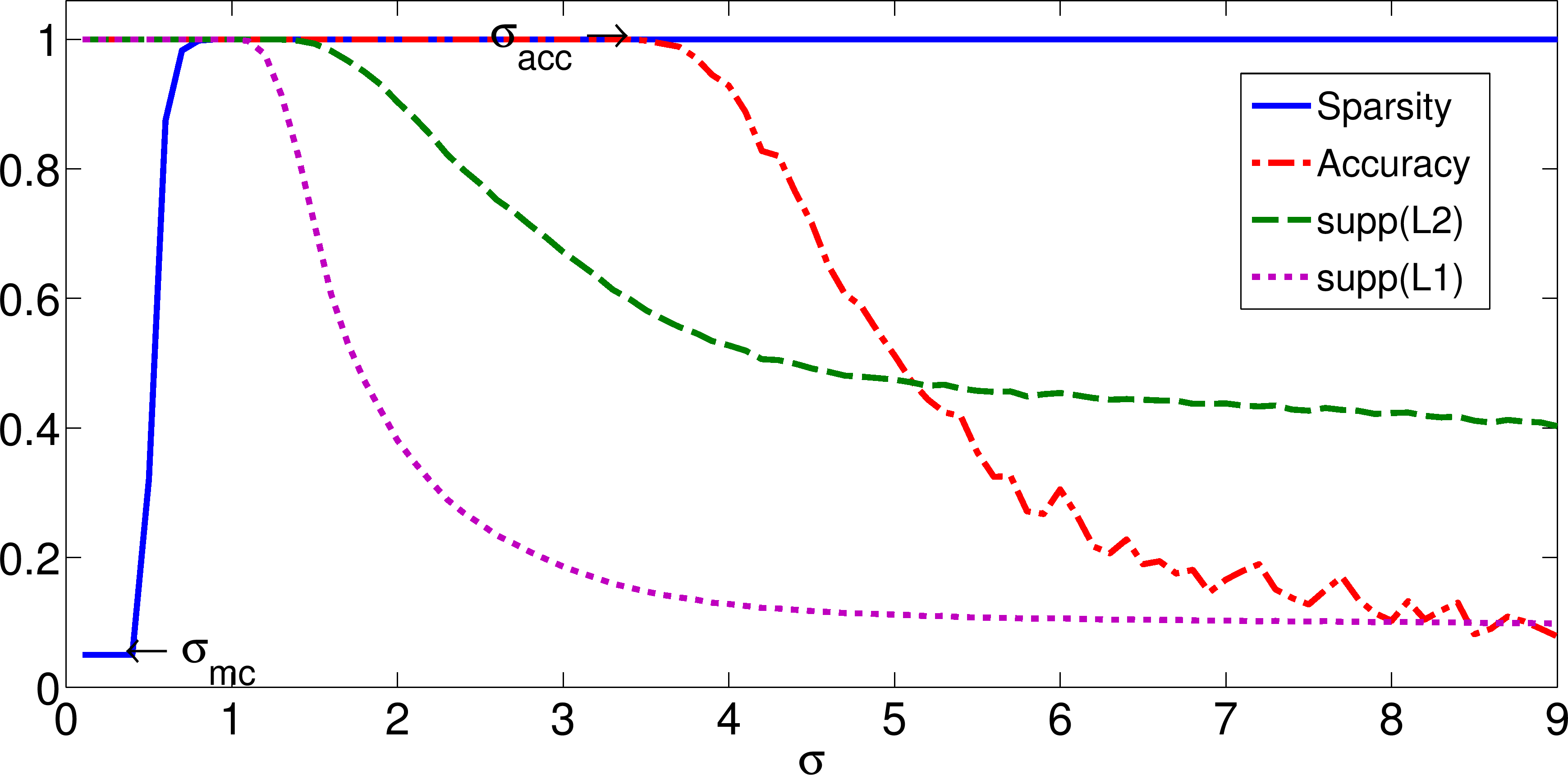}
		\caption{$\eta=0.1$}
	\label{fig:Clean_1000} \hfill%
\end{subfigure}
\begin{subfigure}[b]{0.65\textwidth} 
\centering
	\includegraphics[width=\linewidth]{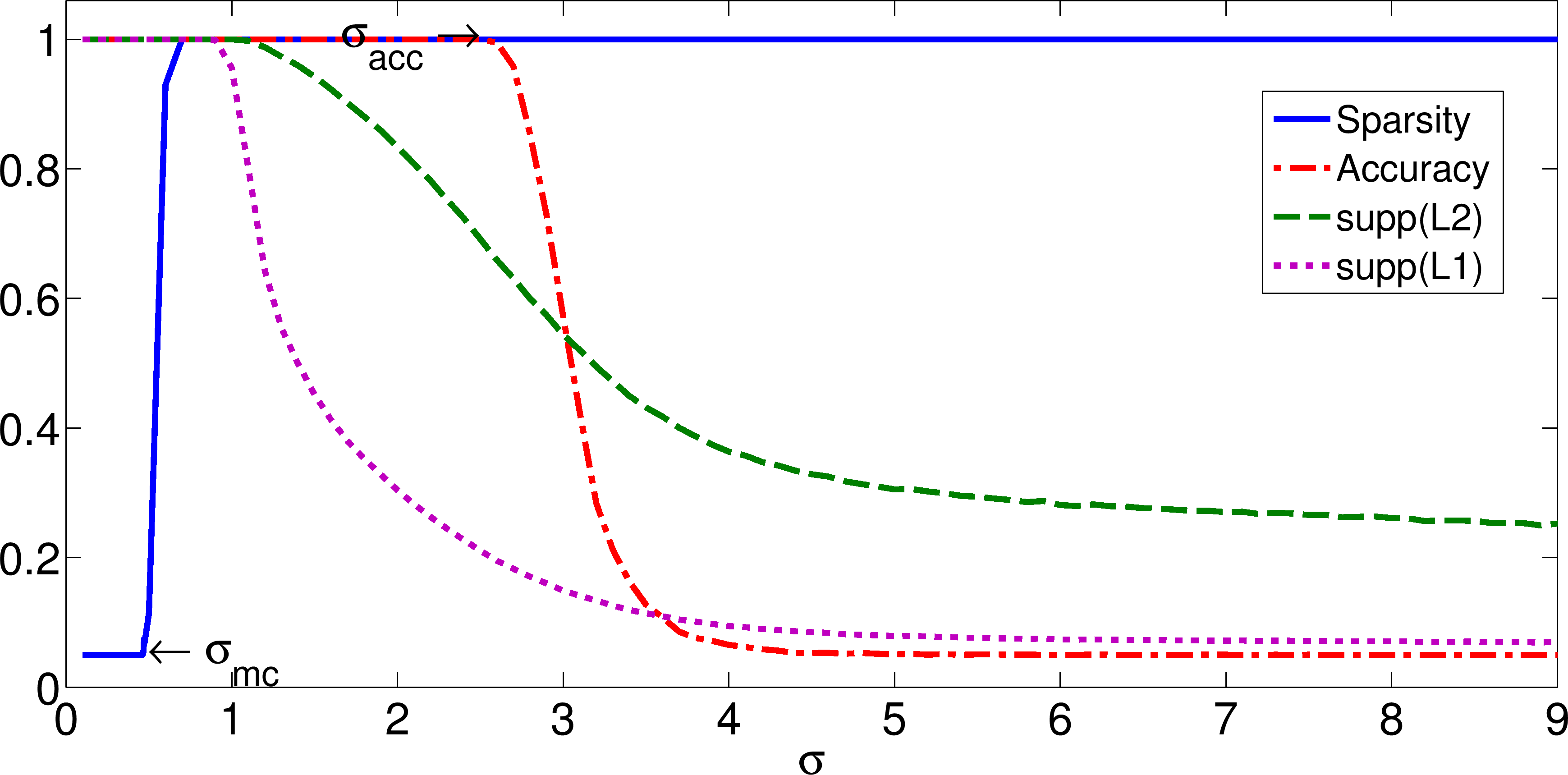}
		\caption{$\eta=0.5$}
	\label{fig:Clean_5000} \hfill%
\end{subfigure}

\caption[Average sparsity, accuracy, and support quantities as correlation increased in the kernel setup]{Average sparsity, accuracy, $\supp(\ell^2)$ and $\supp(\ell^1)$ (over 100 trials) as $\sigma$ increased in the kernel setup. The annotations ``$\sigma_\mathrm{mc}$'' and ``$\sigma_\mathrm{acc}$'' denote the maximum $\sigma$ for which Eq.~\eqref{eq:k_max} holds and for which maximum accuracy is obtained in Kernel SRC, respectively.}
\label{fig:Clean}
\end{figure}

From Figure \ref{fig:Clean}, we see that $\sigma_\mathrm{acc}$ was generally much larger than $\sigma_\mathrm{mc}$, and that the Kernel SRC method could tolerate substantial $\ell^1$ and $\ell^2$-support error before classification deteriorated. Further, perfect classification was achieved even for maximally dense $\bm{\alpha}_1$. This shows that a strictly-sparse solution vector is not always necessary to the success of SRC. 

As the level of noise $\eta$ increased, we see in Figure \ref{fig:Clean} that $\sigma_\mathrm{acc}$ decreased towards $\sigma_\mathrm{mc}$. However, 
\begin{align*}
\lim_{\eta \rightarrow \infty}\sigma_\mathrm{acc} \neq \sigma_\mathrm{mc}.
\end{align*} 
Once the class structure was lost due to noise in the original space,
increasing the noise level further had no effect on the quantities displayed in Figure \ref{fig:Clean}. In other words, Figure \ref{fig:Clean_5000} is representative of the results for larger values of $\eta$. 

We also observe that for $\eta = 0.001$, the sparsest solution was obtained by $\ell^1$-minimization for values of $\sigma$ slightly larger than $\sigma_\mathrm{mc}$ (note the position of the $\sigma_\mathrm{mc}$ arrow tip in Figure \ref{fig:Clean_10}). In fact, the mutual coherence of the dataset with $\eta = 0.001$ reached $\mu = 0.9994$ before $\ell^1$-minimization failed to retrieve the sparsest solution. This indicates that when the classes are well-separated (for small $\eta$ and sufficiently small $\sigma$, separability in the original space carries over to kernel space in this experiment), $\ell^1/\ell^0$-equivalence can still be achieved even when the mutual coherence is much larger than that allowed by Eq.~\eqref{eq:k_max}. This reinforces the findings from Section \ref{MCD_Project_2}, namely, that $\ell^1/\ell^0$-equivalence holds on highly-correlated data as long as the vectors corresponding to the support of the sparsest solution are sufficiently separated from the other dictionary elements. On the other hand, for larger values of $\eta$, i.e., when the classes were less well-separated, the bound in Eq.~\eqref{eq:k_max} appears to be approximately tight.

\subsubsection{Examining the Accuracy Threshold}

It is notable that $\sigma_\mathrm{acc}$ is substantially larger than $\sigma_\mathrm{mc}$ for all $\eta$, and that the accuracy in Kernel SRC has a steep drop-off as soon as $\sigma>\sigma_\mathrm{acc}$. The value $\sigma_\mathrm{acc}$ appears to be a threshold for which the linear relationship between $\phi_\kappa(\bm{y})$ and the training samples in its ground truth class cannot be identified by the classification mechanism in (Kernel) SRC. We want to know what triggers this threshold. 

We first look for an ``elbow'' or sharp change in the correlation between $\phi_\kappa(\bm{y})$ and training samples in its ground truth class, and that between $\phi_\kappa(\bm{y})$ and samples in other classes. In particular, we computed 
\begin{align*}
\corr_\mathrm{GT} := \operatorname*{mean}_{\mathrm{all}\; \mathrm{trials}} \Big\{\operatorname*{median}_{\bm{x}_j^{(l)}: \, \bm{y} \, \in \text{ class } l} \ip{\phi_\kappa(\bm{y}),\phi_\kappa(\bm{x}_j^{(l)})}\Big\}
\end{align*}
and
\begin{align*}
\corr_\mathrm{other} := \operatorname*{mean}_{\mathrm{all}\; \mathrm{trials}} \Big\{\operatorname*{median}_{l: \, \bm{y} \, \notin \text{ class } l} \Big\{\operatorname*{median}_{1\leq j \leq N_l} \ip{\phi_\kappa(\bm{y}),\phi_\kappa(\bm{x}_j^{(l)})}\Big\} \Big\}.
\end{align*}
Again, we compute the median quantities within each trial to make the correlation values more robust to sample outliers. 

The results for $\eta = 0.1$ are shown in Figure \ref{fig:corr}. The plots for the other values of $\eta$ are similar. As we can see, the accuracy threshold $\sigma_\mathrm{acc}$ occurred \emph{after} the sharp increase in the correlation quantities. In fact, we see that SRC was able to retrieve the correct classification assignment when $\corr_\mathrm{GT}$ was only moderately larger than $\corr_\mathrm{other}$. On the other hand, the sharp increase in the correlation quantities appears to correspond to the steep increase in sparsity level, which makes sense in the context of the mutual coherence recovery guarantee in Theorem \ref {thm:mc}. 

\begin{figure}[!htb]
\begin{center}
\includegraphics[width=0.9\linewidth]{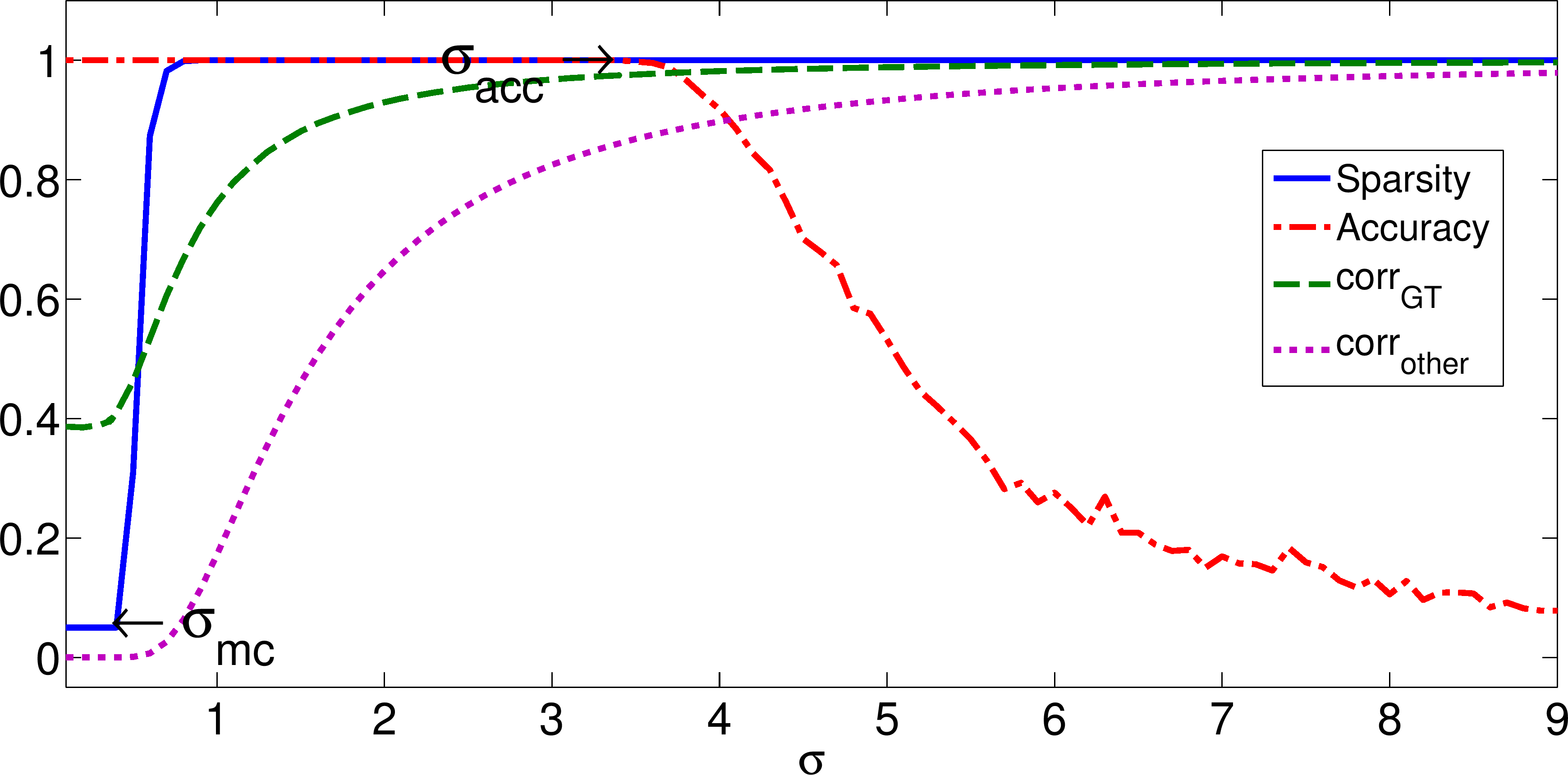}
\caption[Median correlation between the test sample and training samples in the same class, and between the test sample and training samples in different classes]{Median correlation (averaged over 100 trials) between the test sample $\phi_\kappa(\bm{y})$ and training samples in the same class ($\corr_\mathrm{GT}$) and training samples in different classes ($\corr_\mathrm{other}$) for the synthetic database with $\eta = 0.1$. Sparsity and accuracy are also displayed for comparison. Notice that the drop in accuracy occurs well after the jump in the correlation terms and sparsity. }
\label{fig:corr}
\end{center}
\end{figure}

As a more informative approach to understanding the accuracy threshold, in particular, what causes the sharp drop-off in accuracy at $\sigma_\mathrm{acc}$, we consider the distribution of the absolute values of the coefficients, i.e., the magnitude of the coordinates of $\bm{\alpha}_1$, with respect to the different classes. Without loss of generality, we do this by studying the coefficients for the class $l=20$ test samples. More specifically, for $\eta = 0.1$, we computed the mean vector $|\bm{\alpha}_1|$ over the $N_0=5$ class $l=20$ test samples, and then averaged the result over 100 trials:
\begin{align*}
\operatorname*{mean}_{\mathrm{all}\; \mathrm{trials}} \Big\{\operatorname*{mean}_{\bm{y} \, \in \text{ class } l=20}\big\{|\bm{\alpha}_1|\big\}\Big\}.
\end{align*}
Lastly, we normalized the resulting vector so that its entries summed to 1.

We plot the results in Figure \ref{fig:hist} for a handful of representative values of $\sigma$. The $x$-axis in the left-hand-side plots (Figures \ref{fig:MCD_3_Hist_sigma_mc_coeff}, \ref{fig:MCD_3_Hist_sigma_3_coeff}, \ref{fig:MCD_3_Hist_sigma_5_coeff}, and \ref{fig:MCD_3_Hist_sigma_9_coeff}) corresponds to the individual coordinates of the averaged vector $|\bm{\alpha}_1| \in \mathbb{R}^{N_\mathrm{tr}}$. The coordinates corresponding to training samples in each class are simply summed to produce the right-hand-side plots (Figures \ref{fig:MCD_3_Hist_sigma_mc_class}, \ref{fig:MCD_3_Hist_sigma_3_class}, \ref{fig:MCD_3_Hist_sigma_5_class}, and \ref{fig:MCD_3_Hist_sigma_9_class}), so that the contribution from each class in the representation of $\phi_\kappa(\bm{y})$ can be viewed easily. We also include the corresponding Kernel SRC classification accuracies for reference.

\begin{figure}[!htb]%
\centering
\begin{subfigure}[b]{0.33\textwidth}
\centering
	\includegraphics[width=\linewidth]{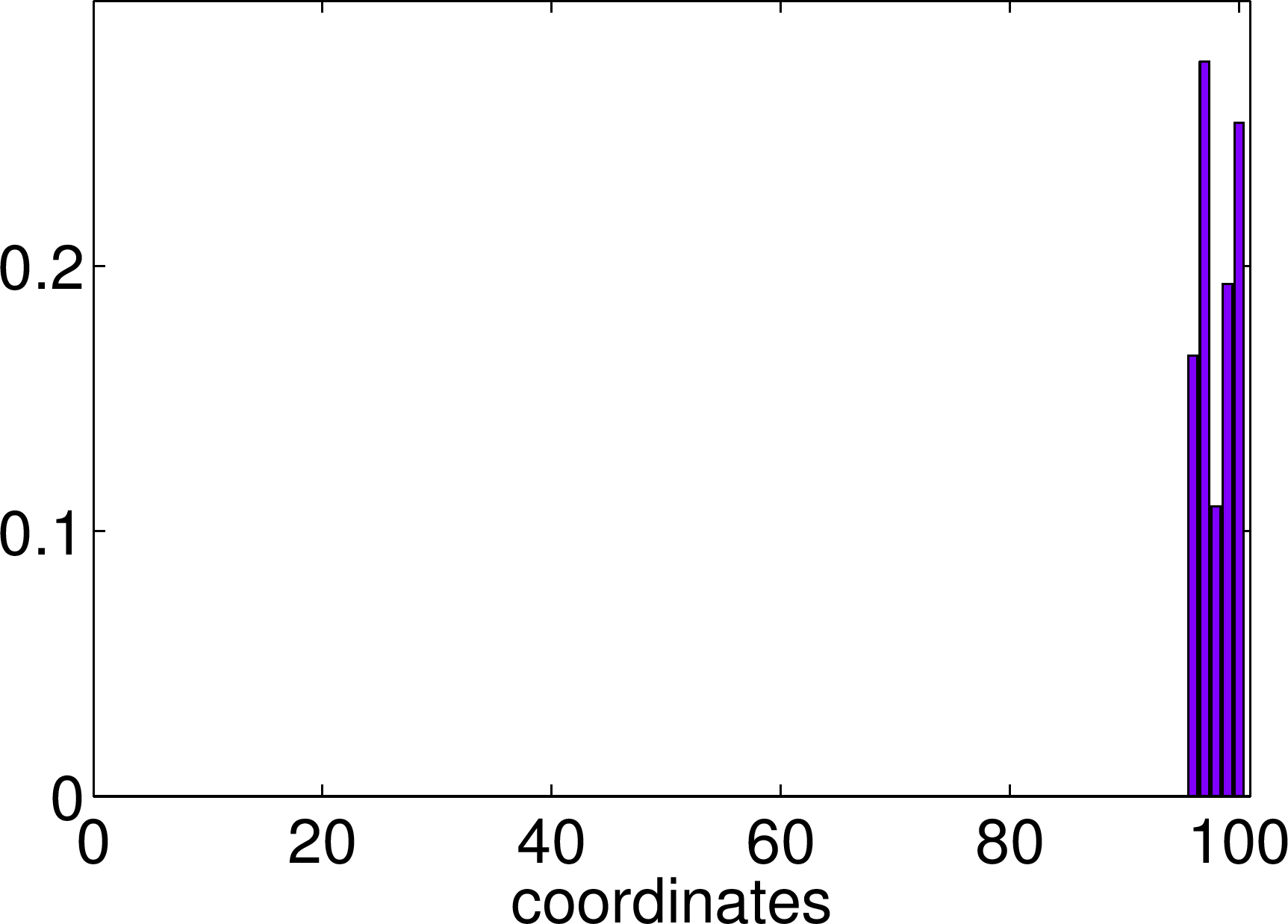}
		\caption{$\sigma = \sigma_\mathrm{mc}$, Accuracy $= 1$}
	\label{fig:MCD_3_Hist_sigma_mc_coeff} \hfill%
\end{subfigure}
\begin{subfigure}[b]{0.33\textwidth} 
\centering
	\includegraphics[width=\linewidth]{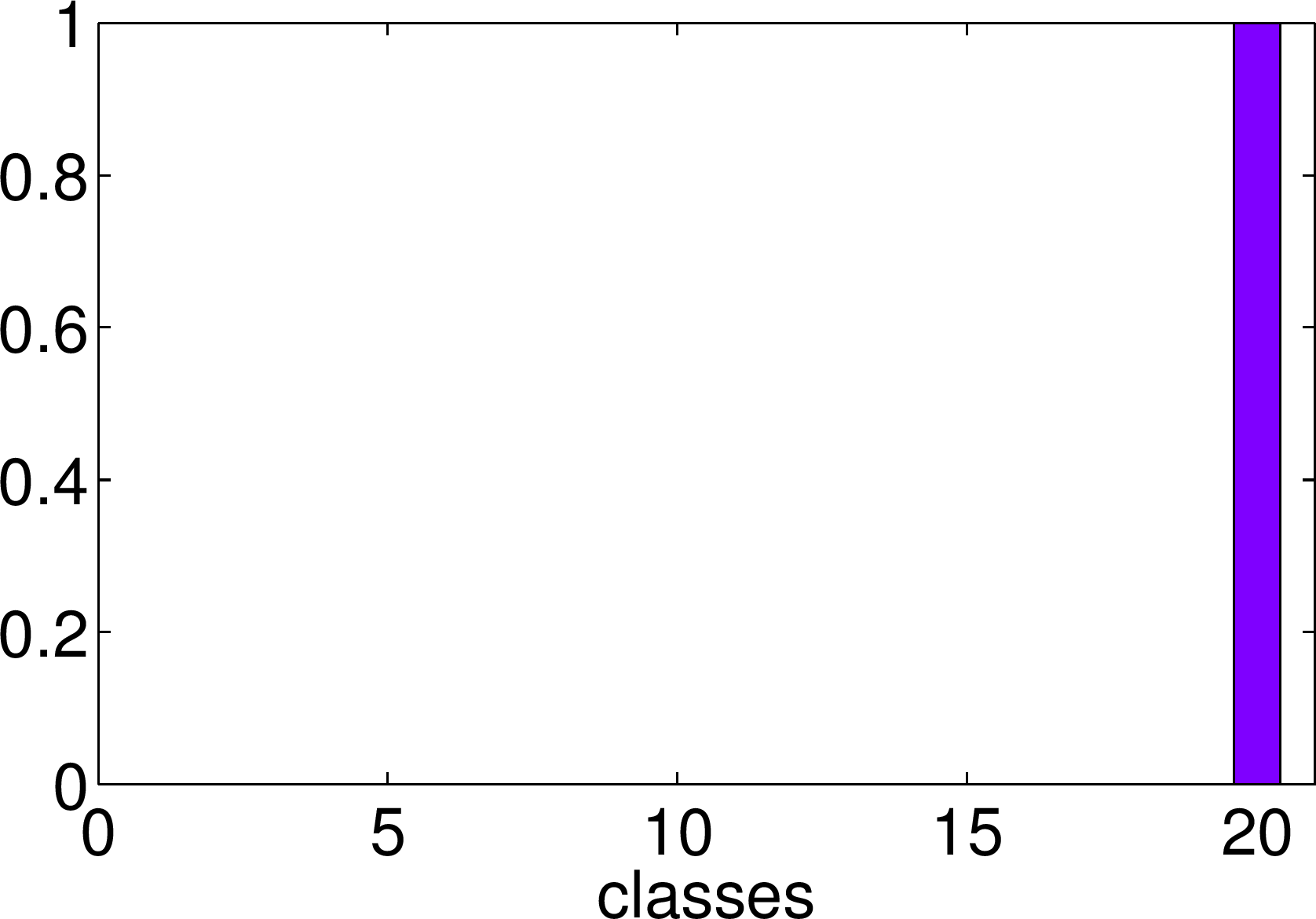}
		\caption{$\sigma = \sigma_\mathrm{mc}$, Accuracy $= 1$}
	\label{fig:MCD_3_Hist_sigma_mc_class} \hfill%
\end{subfigure}
\begin{subfigure}[b]{0.33\textwidth} 
\centering
	\includegraphics[width=\linewidth]{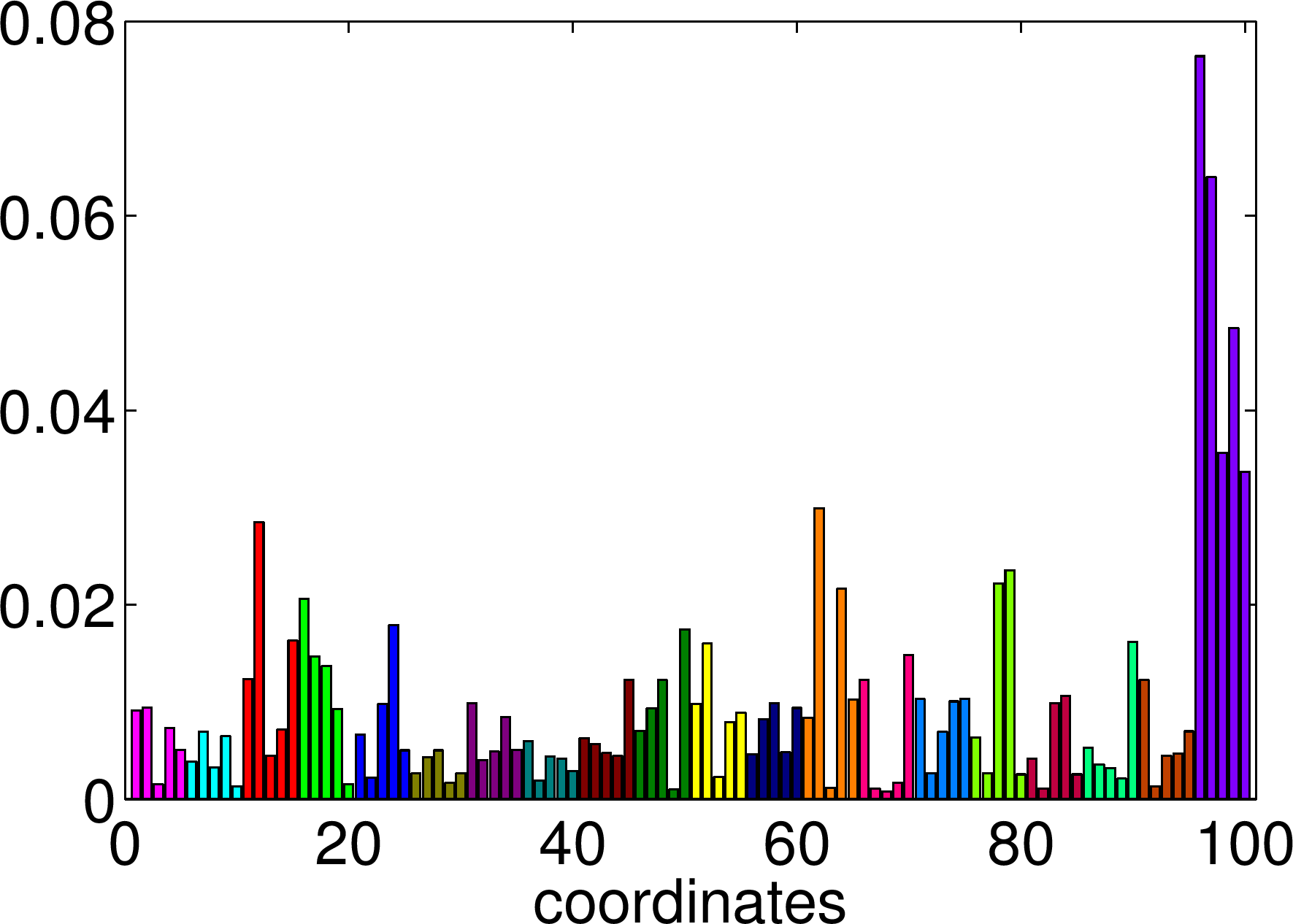}
		\caption{$\sigma = 3$, Accuracy $= 1$}
	\label{fig:MCD_3_Hist_sigma_3_coeff} \hfill%
\end{subfigure}
	\begin{subfigure}[b]{0.33\textwidth}
\centering
	\includegraphics[width=\linewidth]{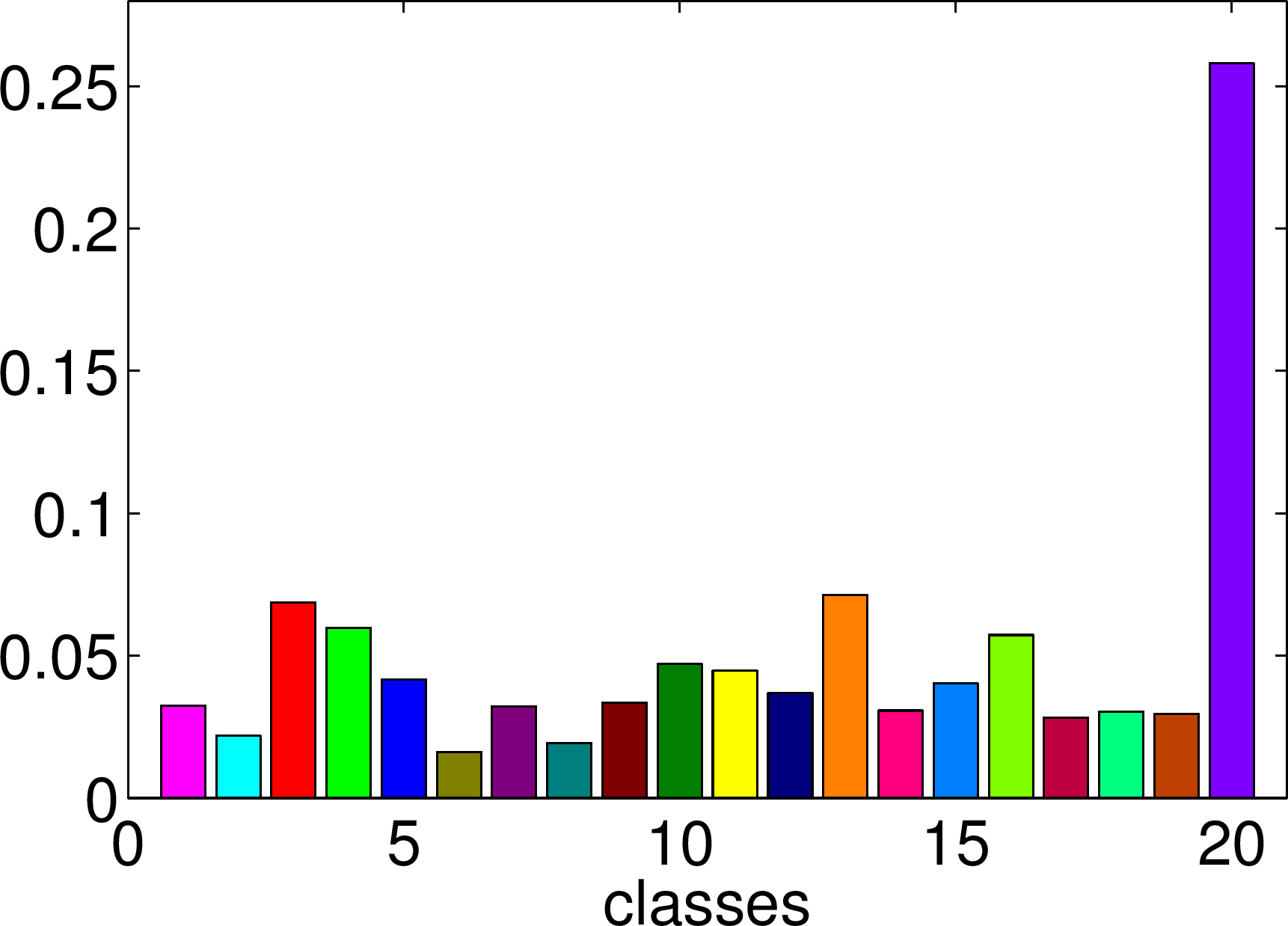}
		\caption{$\sigma = 3$, Accuracy $= 1$}
	\label{fig:MCD_3_Hist_sigma_3_class} \hfill%
\end{subfigure}
\begin{subfigure}[b]{0.33\textwidth} 
\centering
	\includegraphics[width=\linewidth]{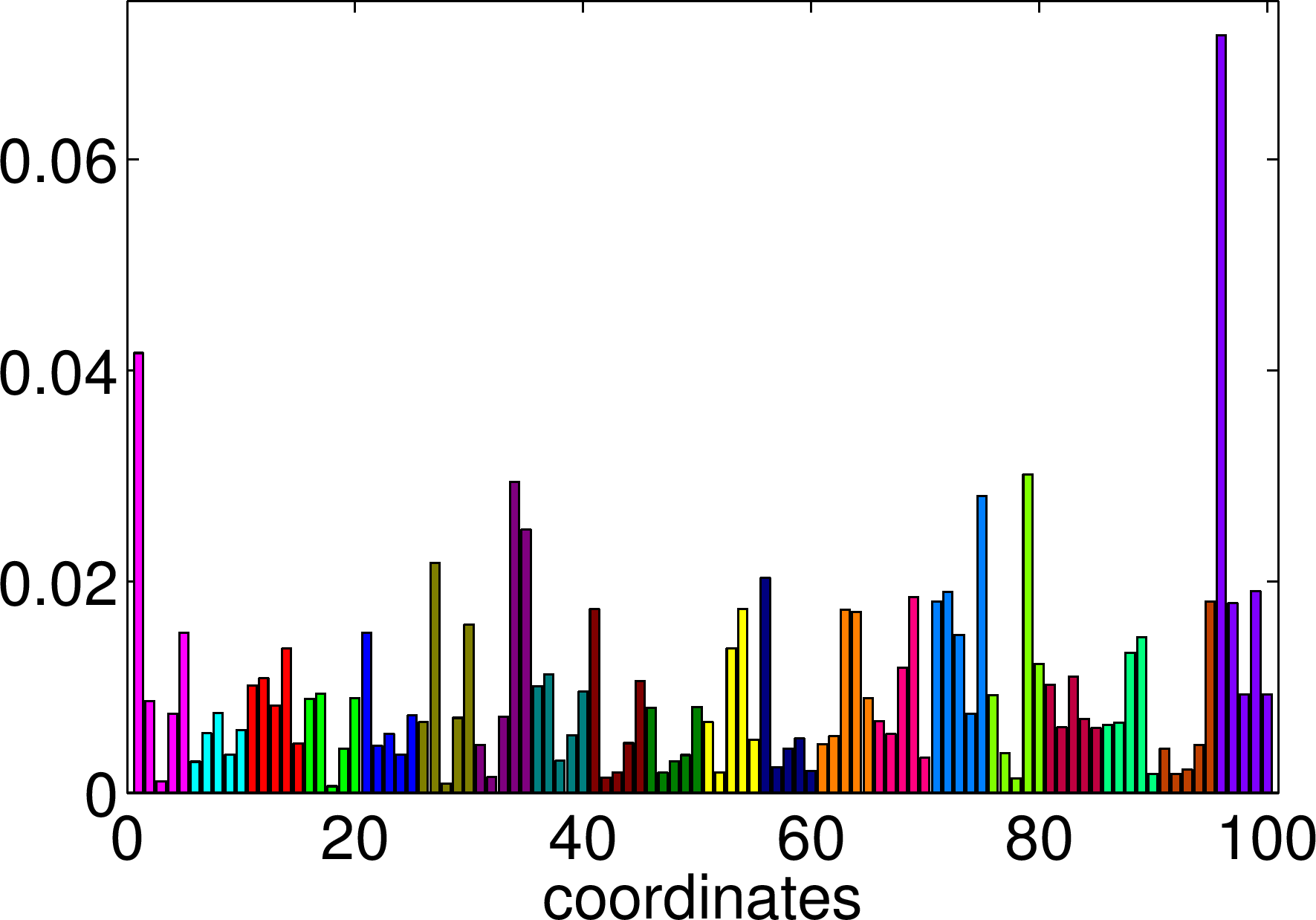}
		\caption{$\sigma = 5$, Accuracy $= 0.51$}
	\label{fig:MCD_3_Hist_sigma_5_coeff} \hfill%
\end{subfigure}
\begin{subfigure}[b]{0.33\textwidth} 
\centering
	\includegraphics[width=\linewidth]{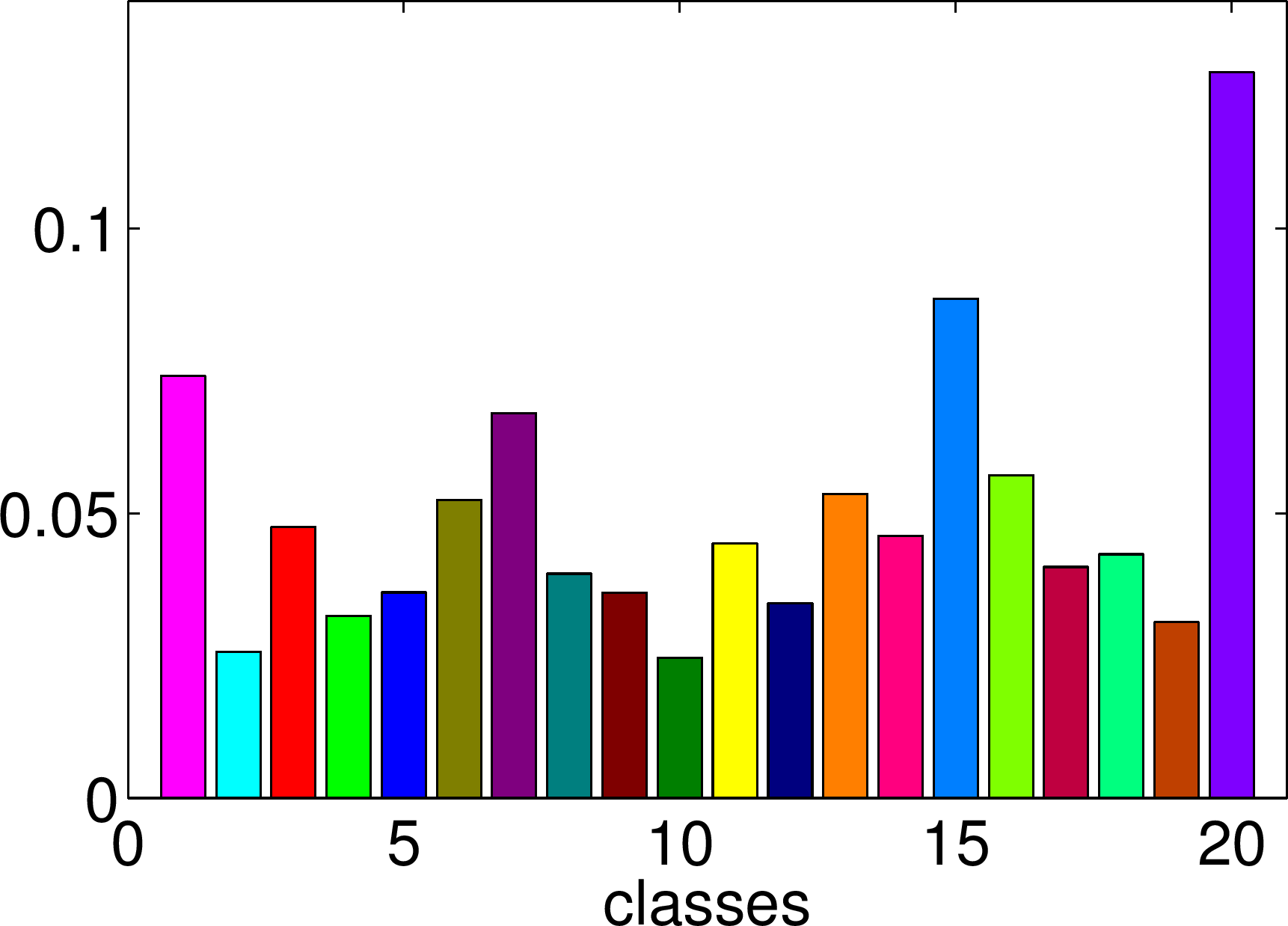}
		\caption{$\sigma = 5$, Accuracy $= 0.51$}
	\label{fig:MCD_3_Hist_sigma_5_class} \hfill%
\end{subfigure}
\begin{subfigure}[b]{0.33\textwidth} 
\centering
	\includegraphics[width=\linewidth]{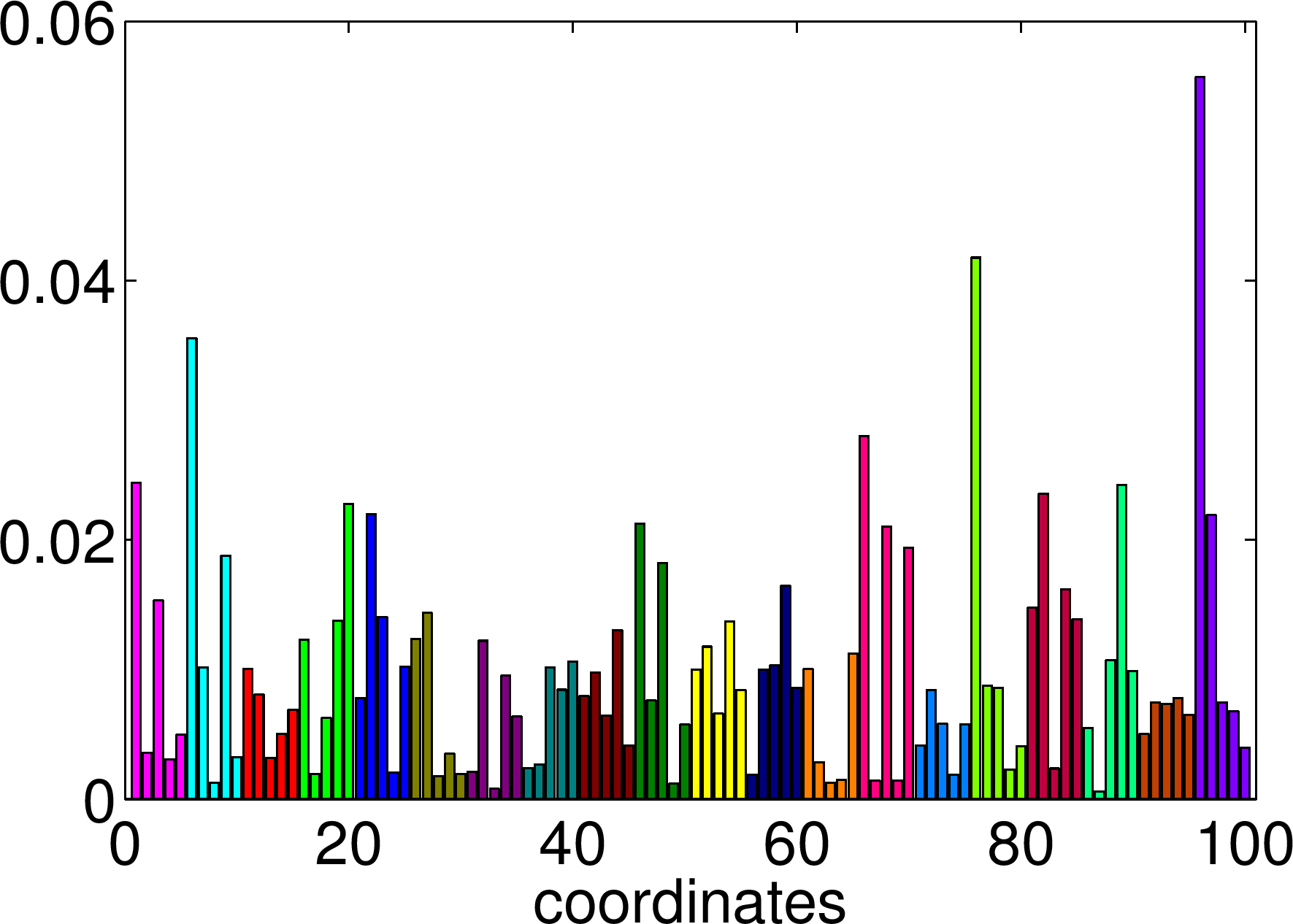}
		\caption{$\sigma = 9$, Accuracy $= 0.08$}
	\label{fig:MCD_3_Hist_sigma_9_coeff} \hfill%
\end{subfigure}
\begin{subfigure}[b]{0.33\textwidth} 
\centering
	\includegraphics[width=\linewidth]{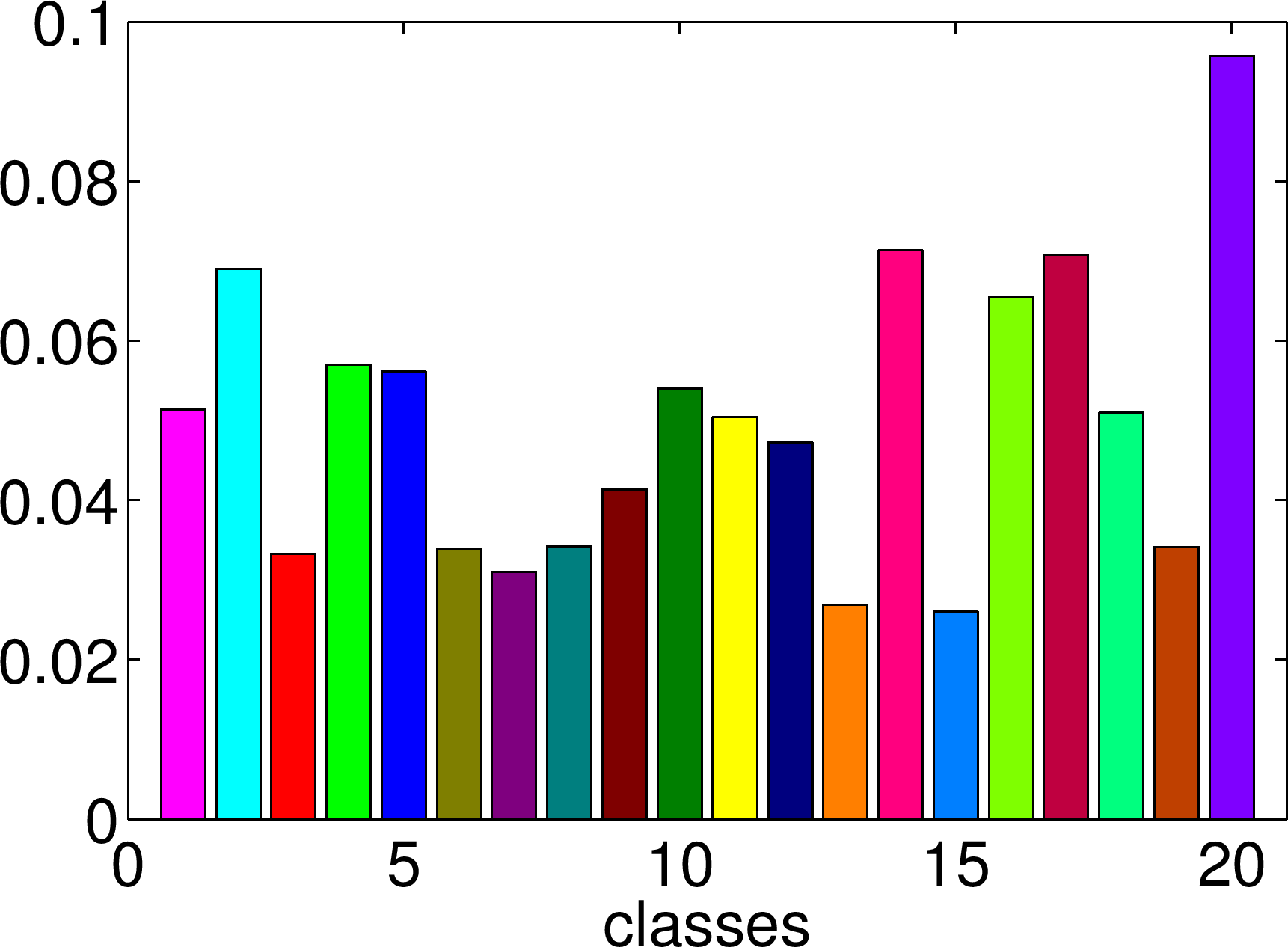}
		\caption{$\sigma = 9$, Accuracy $= 0.08$}
	\label{fig:MCD_3_Hist_sigma_9_class} \hfill
\end{subfigure}
\caption[Average coefficient magnitudes and class contributions of coefficient vectors corresponding to class $l=20$ test samples at various levels of coherence]{Average class contributions (over 100 trials) of coefficient vectors corresponding to class $l=20$ test samples. The colors denote the classes.}
\label{fig:hist}
\end{figure}

Given the dominance of coefficients corresponding to class $l = 20$ in Figures \ref{fig:MCD_3_Hist_sigma_mc_coeff}-\ref{fig:MCD_3_Hist_sigma_3_class}, it is not surprising that Kernel SRC obtains perfect accuracy in these cases. It is also quite clear from these figures that small coefficients in the wrong class do not negatively affect classification accuracy. Thus there is no reason to require a solution sparser than that with $\sigma = 3$. 

For $\sigma \in \{5,9\}$, the closeness in the coefficient magnitudes between those corresponding to class $l=20$ and those corresponding to other classes illustrates the decreased accuracy in Kernel SRC; recall that these plots contain averages. Additionally, we note that the distribution of the coefficients in class $l=20$ became fairly unbalanced among that class's training samples for these large values of $\sigma$. This is because as mutual coherence increased, the class $l=20$ samples became more and more parallel to each other. Thus most of $\phi_\kappa(\bm{y})$ could be represented using only the first training sample in that class. 

Figure \ref{fig:hist} helps to explain the sharp drop-off in accuracy at $\sigma_\mathrm{acc}$. Though the quantities $\corr_\mathrm{GT}$ and $\corr_\mathrm{other}$ are only slightly increasing at $\sigma_\mathrm{acc}$ (and the general behavior of the coefficients varying smoothly), the threshold occurs right at the point where the coefficients of other classes become competitive with those from the correct class (as we would expect). The sharp drop-off can be attributed to the nonlinearity of the \texttt{min} function in determining $\min_{1\leq l \leq L} \{\|\phi_\kappa(\bm{y})-\Phi_\kappa(X_\mathrm{tr})\delta_l(\bm{\alpha}_1)\|_2\}$ in the classification stage of (Kernel) SRC.

\subsection{Key Findings}

We summarize some important conclusions from this section:
\begin{itemize} [noitemsep,nolistsep]
\item Any procedure that spreads out the data in each class in a way that decreases mutual coherence yet aims to maintain class structure will necessarily come into conflict with maintaining a linear relationship between $\bm{y}$ and \emph{any} subset of training samples. More precisely, it is generally impossible to write $\bm{y}$ as a linear combination of the training samples in class $l$ while satisfying the bound 
\begin{align*}
\|\bm{\alpha}\|_0 < \frac{1}{2}\Big(1+\frac{1}{\mu(X_\mathrm{tr})}\Big) \approx \frac{1}{2}\Big(1+\frac{1}{\mu([X^{(l)},\bm{y}])}\Big),
\end{align*}
 i.e., when $\bm{y}$ is spread out in the same manner as the other samples in the database. Besides artificially generating $\bm{y}$ as a linear combination of the training samples \emph{after} they have been spread out, it is not clear to us how to overcome this conflict. 

\item Though generating $\bm{y}$ as a linear combination of its ground truth class training samples in kernel space prevented us, in some sense, from isolating the relationship between $\sigma_\mathrm{mc}$ and classification accuracy, we were still able to study the correspondence between $\sigma_\mathrm{mc}$ and sparsity level $\|\bm{\alpha}_1\|_0$. In particular, we confirmed our previous findings that perfect recovery can be achieved on highly-correlated data as long as the classes are sufficiently well-separated (in this experiment, this meant small $\eta$). 

\item We saw that there was a sharp drop-off in classification accuracy as soon as $\sigma > \sigma_\mathrm{acc}$, which was not directly correlated with a sharp change in either sparsity or the relationship between within-class and between-class correlation, or in the normalized $\ell^2$ and $\ell^1$-norms of $\delta_\mathrm{GT}(\bm{\alpha}_1)$. Though $\ell^1/\ell^0$-equivalence (whether provable by Theorem \ref {thm:mc} or not) was a way to ensure perfect classification accuracy in this experiment, it was not necessary. The classification mechanism in SRC can clearly tolerate even the maximal number of nonzero coefficients in the representation, as long as the magnitudes of coefficients corresponding to the wrong classes are small with respect to those from the correct class. In this sense, \emph{relative}---or \emph{approximate}---sparsity is the key to SRC. It might be possible to make this idea precise in terms of a coefficient thresholding procedure similar to the one used in Section \ref{MCD_Project_2}.

\end{itemize}

In future research, it would be interesting to consider the modification of the above experiment when noise is added to the test sample $\phi(\bm{y})$ after it is generated as a linear combination of its ground truth class training samples in kernel space. Of course, this will not have the same effect as adding noise to the original (and implicitly-defined) test sample $\bm{y}$, but it would allow us to investigate the relationship between classification accuracy in SRC and the mutual coherence bound in the case of noise as stated in Theorem \ref{thm:mc_noise}.

\section{Conclusion} \label{sec:conclusion}

In this paper, we investigated the applicability of $\ell^1/\ell^0$-equivalence guarantees on dictionaries containing training samples. We detailed the inherent conflict between tightly-clustered classes---desirable for good classification---and the sufficient incoherence required by recovery guarantees such as those based on mutual coherence. In particular, we proved that under the assumptions of SRC, i.e., that class manifolds are linear subspaces spanned by their respective training data, Donoho et al.'s mutual coherence guarantees can only hold in the case that we have \emph{exactly} enough training samples to span each lower-dimensional subspace. Considering that the performance of SRC should generally improve as the training class size increases, it is likely counter-productive for classification purposes to restrict the training set in this way. Further, despite existing methods to estimate the class manifold dimension, it is impractical to assume that such approaches will always work perfectly.

Despite not being able to prove $\ell^1/\ell^0$-equivalence on most class-structured data, we saw that it can indeed be achieved in some specific cases. Inspired by the random model of Wright and Ma to generate face image-like databases, we designed an experiment to test the ability of $\ell^1$-minimization to recover the sparsest solution on highly-correlated data. The results were mostly positive. We observed that in all cases, $\ell^1$-minimization recovered a solution closely approximating the sparsest solution (defined by generating the test sample as a linear combination of training samples in its ground truth class). Further, within-class correlation actually improved recovery relative to uniformly-random data, provided that the between-class correlation was sufficiently low, i.e., that the classes were sufficiently separated. In many cases, $\ell^1$-minimization exactly recovered the sparsest solution. Additionally, in the case that noise was added to the test sample, the correct support was found in nearly every case in which correlation was introduced.

We also considered the role of sparsity in the context of SRC and similar classification algorithms. One obstacle in determining this relationship is obtaining access to the sparsest solution for comparison without the aid of $\ell^1/\ell^0$-equivalence guarantees. Towards resolving this problem, we designed a nonlinear transform, based on kernel methods using the Gaussian kernel, to decrease the within-class mutual coherence while still maintaining class structure so that (hypothetically) provable equivalence and good classification could be simultaneously achieved. However, we found that the degree to which we had to decrease coherence in this setup meant that the test sample was no longer in the span of the training data, and so we were forced to limit our analysis to test samples artificially generated as linear combinations of their ground truth class training samples, as in Section \ref{MCD_Project_2}. Though this to some extent limited the applicability of our experiment, the results clearly indicate that strict sparsity is not necessary for good classification in SRC. Instead, its success lies in its ability to correctly differentiate the coefficient magnitudes of training samples in different classes, i.e., to find \emph{approximately} or \emph{relatively} sparse solutions, in the case that the linear subspace assumption is observed and the classes themselves are not too correlated, i.e., not close together. 

There is certainly much work to be done to quantify these findings. We mention two potential next steps: Eldar and Kuppinger's notion of \emph{block-coherence} \cite{eld:bs}, with blocks corresponding to classes of the training database, might serve to make precise the meaning of between-class correlation; note that this was observed to play a role in both $\ell^1/\ell^0$-equivalence on highly-correlated data and SRC's classification performance. Additionally, the accuracy threshold detected in Section \ref{MCD_Project_3} might be better understood in the context of Wang et al.'s interpretation of SRC as a maximum margin-based classifier \cite{wang:max_margin}. As an alternative to the thresholding route as suggested in Section \ref{MCD_Project_3}, their work could be very helpful in rigorously defining the concept of \emph{approximate sparsity} as it relates to the classification performance of SRC.

\section*{Acknowledgments}

C. Weaver's research on this project was conducted with government support under contract FA9550-11-C-0028 and awarded by DoD, Air Force Office of Scientific Research, National Defense Science and Engineering Graduate (NDSEG) Fellowship, 32 CFR 168a. She was also supported by National Science Foundation VIGRE DMS-0636297 and NSF DMS-1418779. N. Saito was partially supported by ONR grants N00014-12-1-0177 and N00014-16-1-2255, as well as NSF DMS-1418779.

\section*{References}

\bibliography{Bib_Master}

\end{document}